\documentclass[10pt]{article} % For LaTeX2e

\usepackage[accepted]{rlj} % Should be uncommented for the camera-ready
\usepackage{amssymb}            % Defines common symbols like \mathbb R
\usepackage{mathtools}          % Extends amsmath, providing common math tools
\usepackage{mathrsfs}           % Enables \mathscr, which can work in cases that \mathcal does not
\usepackage{graphicx}           % For including images
\usepackage{subcaption}         % Allows for the use of subfigures and subcaptions
\usepackage[space]{grffile}     % For spaces in image names
\usepackage{url}                % For displaying URLs
\usepackage{lipsum}             % For placeholder text

\RequirePackage{algorithm}
\usepackage{algorithmic}
\input{defs.tex}
\newenvironment{proof}{\par\noindent{\bf Proof:\ }}{\eproof} 
\title{Value Bonuses using Ensemble Errors\\ for Exploration in Reinforcement Learning}

\setrunningtitle{Value Bonuses using Ensemble Errors for Exploration in Reinforcement Learning}
\author{Abdul Wahab\textsuperscript{1}, Raksha Kumaraswamy\textsuperscript{2}, Martha White\textsuperscript{1,3}}

\emails{\{wahab1, whitem\}@ualberta.ca, \ raksha.kumaraswamy@sony.com}

\affiliations{
$^{1}$\textbf{Department of Computing Science, University of Alberta, Canada}\\
$^{2}$\textbf{Sony AI}\\
$^{3}$\textbf{Alberta Machine Intelligence Institute (Amii)}\\
}

\contribution{
    Proposes a new approach for estimating Value Bonuses using Ensemble Errors that allows for first-visit optimism and deep exploration.
    }
    {
    Many prior works utilize the idea of \textit{value bonuses} to estimate optimistic values for exploration \citep{osband2019deep}. Typically, many estimate an additional value function that propagates \textit{reward bonuses} in order to estimate the value bonus \citep{burda2018exploration}. This work proposes a variant of value bonuses that does not rely on propagating additional reward bonuses; this allows for desirable features like first-visit optimism. We show our framework allows for Optimistic Initial Values with high probability. Additionally, the value bonuses have a similar timescale of learning as the main value function, therefore, potentially allowing for deep exploration.
    }

\contribution{
    Provides insight into how our proposed value bonuses are similar to and different from some relevant widely-used approaches. 
    }
    {
    In specific, we contrast how the proposed bonuses capture MDP-specific properties like transition dynamics, unlike those of RND \citep{burda2018exploration}. Additionally, we highlight the similarity between quantities estimated by the proposed algorithm and BDQN \citep{osband2019deep}, despite the difference that BDQN utilizes a Thompson sampling approach to induce optimism. 
    }

\contribution{
    Empirically evaluate the utility of the proposed value bonuses for inducing exploration in classic-control problems, and demonstrate scalability through Atari.
    }
    {
    We demonstrate how existing methods lack first-visit optimism in a controlled setting designed to test state coverage. We show that the proposed algorithm can extend to more complex environments like Atari without design choices that alter the underlying algorithm.
    }

 % Your keywords

\summary{
Optimistic value estimation can be useful to direct exploration and improve sample efficiency in Reinforcement Learning. Despite many such methods in literature, simpler, undirected approaches like $\epsilon$-greedy still continue to be widely used. One potential reason for this is that many existing methods can be onerous to use as they may not be compatible with the base learning algorithm, or can be hard-to-use as many design choices need to be made to make them effective in practice. This paper proposes a simple approach to address these limitations. Building on ideas that utilize an ensemble for optimistic value estimation, this work proposes an algorithm called Value Bonuses using Ensemble Errors (VBE) that is easy to use and compatible with any base Reinforcement Learning algorithm, with a small additional computational foot-print. VBE's similarity and difference to existing approaches is discussed, and the algorithm is evaluated extensively. Our code is available at: \href{https://github.com/mirzaabdulwahab1612/VBE}{\texttt{https://github.com/mirzaabdulwahab1612/VBE}}
}

\newcommand{\esize}{k}
\newcommand{\targetq}[1]{f_{#1}}
\newcommand{\ensembleq}[1]{f_{w_{#1}}}
\newcommand{\ensembleqt}[1]{f_{\tilde{w}_{#1}}}
\newcommand{\ensembler}[1]{r_{#1}}

\newcommand{\vb}{b}
\newcommand{\bscale}{c}

\begin{document}

\makeCover  % Create the cover page
\maketitle  % Make the title section

\begin{abstract}
Optimistic value estimates provide one mechanism for directed exploration in reinforcement learning (RL). The agent acts greedily with respect to an estimate of the value plus what can be seen as a \emph{value bonus}. The value bonus can be learned by estimating a value function on \emph{reward bonuses}, propagating local uncertainties around rewards. However, this approach only increases the value bonus for an action retroactively, after seeing a higher reward bonus from that state and action. Such an approach does not encourage the agent to visit a state and action for the first time. In this work, we introduce an algorithm for exploration called Value Bonuses with Ensemble errors (VBE), that maintains an ensemble of random action-value functions (RQFs). VBE uses the errors in the estimation of these RQFs to design value bonuses that provide first-visit optimism and deep exploration. The key idea is to design the rewards for these RQFs in such a way that the value bonus can decrease to zero. We show that VBE outperforms Bootstrap DQN and two reward bonus approaches (RND and ACB) on several classic environments used to test exploration and provide demonstrative experiments that it can scale easily to more complex environments like Atari. We provide code for VBE \href{https://github.com/mirzaabdulwahab1612/VBE}{here}.
\end{abstract}

\section{Introduction}
\label{sec:introduction}
A typical approach to incorporate exploration into a value-based reinforcement learning (RL) agent is to obtain optimistic value estimates. The agent takes greedy actions according to this optimistic value estimate, leading it to take actions that look good either because they have high uncertainty or because the action is actually high value. This approach has been well-developed for the contextual bandit setting, with a variety of algorithms and theoretical results on optimality \citep{li2010acontextual, abbasi2011improved}. Understanding is growing about how to soundly extend these ideas to reinforcement learning, though the theoretical results on estimating and using optimistic values are limited to the linear function approximation setting \citep{grande2014sample,osband2016generalization,abbasi2019exploration,wang2019optimism}. 

Though the theory is difficult to extend, there has been a concerted effort to develop and empirically evaluate such optimistic value estimation approaches for the deep RL setting. Bootstrap DQN with priors, for example, maintains an ensemble of action-values, which reflect uncertainty in the value estimates \citep{osband2018randomized,osband2019deep}. It takes a Thompson sampling approach---which can be seen as optimistic---by sampling one action-value in the ensemble and following it for an entire episode. 
%Another common approach to obtain optimistic value estimates are \emph{reward bonuses} \citep{Bellemare2016unifying,ostrovski2017count,burda2018exploration,ash2022anticoncentrated}, where an uncertainty for the transition is added to the reward, increasing the value for states and action where the outcome was uncertain 
Another common approach to obtain optimistic value estimates employs the usage of \emph{reward bonuses} \citep{Bellemare2016unifying,ostrovski2017count,burda2018exploration,ash2022anticoncentrated}. A reward bonus, reflecting uncertainty with respect to the transition, is added to the reward, increasing the estimated value proportionally for the corresponding states and action.

Most works, however, eschew these directed exploration approaches in favor of simpler, undirected exploration approaches like $\epsilon$-greedy. One potential reason for this is that reward bonus approaches 
%inherently have a key limitation that they 
do not encourage \emph{first-visit optimism}. They encourage revisiting a state, if the reward bonus was high in that state; namely, they retroactively reason about uncertainty of states they have seen. The reward bonus cannot encourage visiting a state for the first time.
% Do we even end up taling about epistemic uncertainty?
%meaning an important part of epistemic uncertainty is not accounted for. 
Bootstrap DQN with priors (BDQN), on the other hand, does not have this issue, using fixed additive priors to provide first-visit optimism. Unlike reward bonuses, though, BDQN is more onerous to use. It requires completely changing the algorithm to one that maintains and updates an ensemble, and making key choices like how often to follow one of the value functions in the ensemble before switching. 
% MARTHA: I am having a hard time believe this as I write it
%Additionally, it couples the return estimates and uncertainty estimates into one value function, making it difficult to assess the current level of uncertainty .  making choices a shift from a simpler base algorithm like DQN  One disadvantage of this approach is that couples return estimates and uncertainty estimates into one value function, rather than learning uncertainty separately. 
% MARTHA: Do we know it makes learning harder somehow? I would think it doesnt actually
%This approach is stringent as the exploration part is coupled with value learning, i.e., the value function has to adjust around the fixed priors. 
%BDQN also relies on a large ensemble size to ensure that the randomly initialized fixed prior provides optimism. It is also not clear when to switch to a different value function, especially in the non-episodic setting. 
%There is mixed empirical support for Bootstrap DQN in the literature \citep{find} \raksha{maybe remove this part? can't find good cites except bsuite paper for bdqn-p}, and 
Recent work suggests it is key to have a large ensemble for BDQN \citep{janz2019successor,osband2023approximate}. Epinets \citep{osband2023approximate} can match the performance of BDQN with much less compute, but are arguably even more onerous to implement than BDQN. Our goal is to develop an easy-to-use exploration approach for deep RL, that can easily be added to an existing algorithm, making it less onerous to displace the default $\epsilon$-greedy approach. 

To do so, we explore how to directly estimate a \emph{value bonus}. The agent acts greedily according to the value estimate plus this separate value bonus $\vb$, namely $\argmax_a q(s,a) + \vb(s,a)$. The value bonus should ideally represent the uncertainty for that state and action. Though this may be the first time this term is used,\footnote{Usually, $\vb$ would be called a confidence interval, with $q(s,a) + \vb(s,a)$ an upper confidence bound. However, we do not use that term here, because for the heuristics we use, it is not clear we get a valid upper confidence bound. Instead, it is a bonus added to the value when deciding which action looks promising.} there are some works that estimate value bonuses. One simple approach is to separate out the reward bonuses and learn them with a second value function, as was proposed for RND \citep{burda2018exploration} and later adopted by ACB \citep{ash2022anticoncentrated}. This approach, however, still suffers from the fact that reward bonuses are only retroactive, and the resulting $\vb$ is unlikely to be high for unvisited states and actions. For the contextual bandit setting, the ACB algorithm actually directly estimates the value bonus using the maximum over an ensemble of functions, which is high for unvisited states and actions; but the extension to deep RL with reward bonuses loses this first-visit optimism. UCLS \citep{kumaraswamy2018context} and UBE \citep{odonoghue2018uncertainty,janz2019successor} both directly estimate value bonuses, but are limited to linear function approximation. Dora \citep{choshen2018dora} uses value bonuses that are inversely proportional to visitation counts, which is again difficult to extend to the general function approximation setting.

In this work, we introduce a new approach to obtain value bonuses for reinforcement learning, with an algorithm we call Value Bonuses with Ensemble errors (VBE). Similarly to ACB, we use a maximum over an ensemble, but directly use that maximum as the value bonus, rather than indirectly through reward bonuses. 
%The key to directly using the ensemble for the value bonus is to treat the functions in the ensemble as value functions. 
%We introduce a novel update for the value functions in the ensemble, that better mimics the learning dynamics under bootstrapping. 
The idea is to sample a random action-value function (RQF)---such as a random neural network---and extract the implicit random reward function underlying this RQF target. The RQF predictor in the ensemble is updated using temporal difference learning on this random reward. 
Because the RQF target is sampled from the same function class as the RQF predictor, the error can eventually reduce to zero, allowing the value bonus to shrink to zero. These value bonuses are learned separately from the main action-values, and so can be layered on top of many algorithms. In our experiments, for example, we simply use Double DQN \citep{van2016deep}, and modify the step where the agent selects an action from $\epsilon$-greedy to instead taking the greedy action in the value estimate plus the value bonus. We show that this simple approach is an effective, and scalable method for exploration that improves sample efficiency of learning in a range of domains: from hard exploration gridworlds, to image-based Atari domains.

\section{Background}
\label{sec:background}
We focus on the problem of an agent learning optimal behaviour in an environment, whose interaction process is modelled as a Markov Decision Process (MDP). A MDP consists of $(\States, \Actions, P, \Rfcn,\gamma)$
where
$\States$ is the set of states;
$\Actions$ is the set of actions;
$P: \States \times \Actions \times \States \rightarrow [0,\infty)$ provides the transition probabilities;
$\Rfcn: \States \times \Actions \times \States \rightarrow \RR$ is the reward function;
and $\gamma:  \States \times \Actions \times \States \rightarrow [0,1]$ is the transition-based discount function which enables either continuing or episodic problems to be specified \citep{white2017unifying}.
At each step, the agent selects action $A_t$ in state $S_t$,  and transitions to $S_{t+1}$, according to $P$, receiving reward
$R_{t+1} \defeq \Rfcn(S_t, A_t, S_{t+1})$ and discount $\gamma_{t+1} \defeq \gamma(S_t, A_t, S_{t+1})$.

For a policy $\pi: \States \times \Actions \rightarrow [0,\infty]$,
the value for taking action $a$ in state $s$
is the expected discounted sum of future rewards, with actions selected according to $\pi$ in the future,
\begin{equation*}
q^\pi(s, a) = \E_\pi\left[R_{t+1} + \gamma_{t+1} q^\pi(S_{t+1}, A_{t+1}) \Big| S_t = s, A_t = a \right]
\end{equation*}
where $\E_\pi$ means that actions are selected according to $\pi$ in the expectation. The policy $\pi$ can be progressively improved by making it greedy in $q^\pi(s, a)$, then updating the action-values for the new policy, then repeating until convergence. 

In practice, these steps are approximated. The action-values $q^\pi$ are approximated using $q_w$ parameterized by $w \in \wspace \subset \RR^{\xdim}$. One algorithm to estimate $q_w$ is Double DQN (DDQN). DDQN is an off-policy algorithm, meaning that it uses a different behavior policy $\pi_b$ to select actions from the policy it evaluates, which is greedy in $q_w$. This algorithm uses a target network $q_{\tilde{w}}$ for bootstrapping, giving the following update for one transition $(s,a,r,s',\gamma)$:
\begin{equation}
w \gets w + \eta \delta \nabla q_w(s,a) \quad \text{ for } \delta \defeq r + \gamma q_{\tilde{w}}(s', \argmax_{a'} q_w(s',a')) - q_w(s,a) \label{eq_ddqn}
\end{equation}

The behavior policy is typically defined to be $\epsilon$-greedy in $q_w$, but can be any policy that promotes exploration. 
In this work, we consider an alternative choice for the behavior policy: one that uses a value bonus $\vb$, $\pi_b(s) = \argmax_a q_w(s,a) + \vb(s,a)$. The value bonus should reflect uncertainty in the action-value estimate, encouraging the behavior policy to take an action in a state if it has high uncertainty. It might have high uncertainty if $(s,a)$ is quite different from what it has seen before---meaning it has never been visited---or because the agent has not yet visited it sufficiently often to be certain about its value. The focus of this work is a new approach for obtaining $\vb$ for the deep RL setting.
% MARTHA: Too much, implued by insufficiently often
%The amount that the agent has to revisit $(s,a)$ will depend on the stochasticity in the rewards and transitions, as well as the uncertainty in the states that $(s,a)$ leads to.   

% How to effectively estimate $\vb(s,a)$ in the general reinforcement learning setting with nonlinear function approximation for $q_w$ remains an open question. %There are some suggestions, as outlined in the introduction, but more work needs to be done here. 
% We introduce a new approach to estimate $\vb$ in the next section.

\section{Value Bonuses with Ensemble Errors}\label{sec_rqfs}

In this section, we first motivate why we use an ensemble of value functions, rather than simply using supervised learning for the ensemble. We then discuss how to appropriately define the rewards for the ensemble value functions, and %finally provide the VBE algorithm that uses this ensemble. 
contrast the unique property of the bonuses produced by these ensemble of value functions.

The most straightforward approach to get an error from an ensemble is to use a random target, as is done in RND. 
For an ensemble of size $\esize$, we can generate random neural networks $f_1, \ldots, f_\esize$ and update the learned functions $\hat{f}_1, \ldots, \hat{f}_\esize$ in the ensemble using a squared error: for each $(s,a)$, update each $\hat{f}_i$ using loss $(f_i(s,a) - \hat{f}_i(s,a))^2$. The value bonus for any $(s,a)$ can be set to 
% \begin{equation}
%     \vb(s,a) \doteq \max_{i \in [\esize]} (f_i(s,a) - \hat{f}_i(s,a))^2. 
% \end{equation}
\begin{equation}
\label{eq:bonus_form}
    \vb(s,a) \doteq \max_{i \in [\esize]} |\hat{f}_i(s,a) - f_i(s,a)| 
\end{equation}
\citet{ciosek2020conservative} show that fitting random prior functions serve as a computationally tractable approach towards estimating uncertainty in the supervised learning setting.
Unfortunately, in the reinforcement learning setting, this is likely to concentrate too quickly, and will not do what has been called deep exploration \citep{osband2019deep}. We want the agent to reason not just about uncertainty for this state and action, but also about the uncertainty of the state that it leads into.\footnote{Note that RND do not use these errors directly for exploration. Instead, they used them as reward bonuses, which can retroactively promote deep exploration, with the issue that they do not promote first-visit optimism.}

Instead, we want an ensemble of value functions that are more likely to promote deep exploration. More specifically, we want to generate random rewards $r_i$ for each $\ensembleq{i}$, where the $\ensembleq{i}$ are updated using standard temporal difference learning bootstrapping approaches. We want the learning dynamics for these value functions to resemble the primary value function, so that they learn at a similar timescale and are more likely converge to zero once the primary value function has also converged. 

We need to define rewards and target functions that are consistent with each other and that allow us to easily measure the errors. Consider if we again do the simplest thing: generate a random neural network $r_i$ for each $\ensembleq{i}$. Let us assume for now that we have a fixed policy, $\pi$. First, it is not clear how we would actually measure the error since we do not know the true value function $\targetq{i}$, namely the expected return using $r_i$ under policy $\pi$. Further, this true value function may not be representable by $\ensembleq{i}$.

Instead, our proposed approach is to generate a random action-value function (RQF) $\targetq{i}$, and then define rewards consistent with that $\targetq{i}$. Define the stochastic ensemble reward from $(S_t, A_t)$ to be
\begin{equation}
\label{eq:rvf_reward}
R_{i,t+1} \defeq \targetq{i}(S_t,A_t) - \gamma_{t+1} \targetq{i}(S_{t+1},A_{t+1}),
\end{equation}
where $A_{t+1} \sim \pi(\cdot |S_{t+1})$ and $\gamma_{t+1} \defeq \gamma(S_t,A_t,S_{t+1})$ is defined by the environment. Further, by definition, the action-values of the random prediction function is:
\begin{equation}
\label{eq:random_value_target}
q^\pi_i(s,a) \defeq \E_\pi\left[ R_{i,t+1} + \gamma_{t+1} q^\pi_i(S_{t+1},A_{t+1}) \big| S_t=s, A_t=a\right].
\end{equation}
We show in the following proposition that $q^\pi_i = \targetq{i}$.
\begin{proposition}\label{prop_main}
For all $i \in [\esize]$, we have $q^\pi_i = \targetq{i}$.
\end{proposition}
\vspace{-0.3cm}
\begin{proof}
\begin{align*}
q^\pi_i (s,a) 
&= \E_\pi\left[ R_{i,t+1} + \gamma_{t+1} q^\pi_i(S_{t+1},A_{t+1}) \big| S_t=s, A_t=a\right] \\
&= \E_\pi\left[ R_{i,t+1} + \gamma_{t+1} R_{i,t+2} +\gamma_{t+1}\gamma_{t+2} q^\pi_i(S_{t+2},A_{t+2}) \big| S_t=s, A_t=a\right] \\
&= \E_\pi\Big[ [\targetq{i}(s,a) - \gamma_{t+1} \targetq{i}(S_{t+1},A_{t+1})] + \gamma_{t+1} [\targetq{i}(S_{t+1},A_{t+1}) - \gamma_{t+2} \targetq{i}(S_{t+2},A_{t+2})]  \\
&\quad\quad\quad\quad+\gamma_{t+1}\gamma_{t+2} q^\pi_i(S_{t+2},A_{t+2}) \big| S_t=s, A_t=a\Big] \\
&= \E_\pi\Big[ [\targetq{i}(s,a) \underbrace{-\gamma_{t+1} \targetq{i}(S_{t+1},A_{t+1})]}_{\text{cancels}} \underbrace{+\gamma_{t+1} \targetq{i}(S_{t+1},A_{t+1})}_{\text{cancels}}\\ 
& \quad\quad- \gamma_{t+1} \gamma_{t+2} \targetq{i}(S_{t+2},A_{t+2})]  +\gamma_{t+1}\gamma_{t+2} q^\pi_i(S_{t+2},A_{t+2}) \big| S_t=s, A_t=a\Big]. 
\end{align*}
We can keep unrolling this, and these terms will continue to telescope, leaving only the first term $\targetq{i}(s,a)$, completing the proof. 
\end{proof} 

Therefore, updating $\ensembleq{i}$ with rewards $\ensembler{i}$ should converge to $q^\pi_i$---and so to $\targetq{i}$---because $\targetq{i}$ is in the function class of $\ensembleq{i}$. This convergence ensures the value bonuses go to zero, which is desired if we want the agent to stop exploring and converge to the greedy policy.  
Even with a fixed policy, however, this convergence will only occur under certain conditions. Primarily, the failure would be that $\ensembleq{i}$ gets stuck in a local minima or even that it diverges, due to know issues with temporal difference (TD) learning algorithms combined with neural networks and with off-policy update. 

There is fortunately a large (and growing) literature understanding the convergence behavior of TD algorithms. Under linear function approximation, we know least-squares TD converges at a rate of $1/\sqrt{T}$ to the global solution, even under off-policy sampling \citep{tagorti2015rate}. 
%We can specialize this result to our realizable setting, removing any errors due to the true value function not being in the function class. 
With the advent of theory for overparameterized networks, TD with a particular neural network function class has been shown to converge to the global solution, under on-policy sampling \citep{cai2019neural}. In general, we know that a class of modified TD algorithms, called gradient TD methods, converge even under off-policy sampling and nonlinear function approximation \citep{dai2017learning,patterson2022generalized}. 
% no need, the journal papers covers this ghiassian2020gradient
Convergence under off-policy sampling is key in our setting, because the behavior policy is optimistic but the target policy may be greedy. We expect that under certain conditions on the neural network it might be possible to say that these gradient TD methods converge to global solutions, though to the best of our knowledge, no such work yet exists. 
% In fact, let's just omit these since we don't talk about them
%Such convergence results also exist for other off-policy TD methods \citep{yu2015convergence,liu2020finite}. \raksha{just moved this sentence to be after discussion of gradient-TD, from before}. 
We provide a more complete discussion in Appendix of how this existing theory on convergence of TD applies to our setting.  

\subsection{Bonuses that reflect MDP-specific properties}
While the form of the bonus proposed above in Equation~\ref{eq:bonus_form} looks similar in principle to RND, the main difference as mentioned previously is that we consider the ensemble to be composed of random value functions, in contrast to methods like RND which consider the ensemble to be composed of random functions. This simple change provides an interesting property to the bonuses derived from this ensemble: they can reflect the stochasticity in the transition dynamics of the MDP.
\begin{proposition}
Let $\hat{f}_i(s,a) \defeq \E[R_{i, t+1} + \gamma_{t+1} \hat{f}_i(S_{t+1},A_{t+1}) | S_t=s, A_t=a] + \epsilon_i(s,a)$, where $\epsilon_i(s,a)$ denotes the Bellman error for $(s,a)$ in $\hat{f}_i$. Then
\begin{equation*}
b(s,a) = \max_{i\in[k]}\big|\E[\gamma_{t+1}(\hat{f}_i(S_{t+1}, A_{t+1}) - f_i(S_{t+1}, A_{t+1}))| S_t=s, A_t=a]  + \epsilon_i(s,a) \big|.
\end{equation*}
\end{proposition}
\begin{proof}
By the assumption that $f_i$ is a value-function, we know that
\begin{equation*}
f_i(s,a) \defeq \E[R_{i, t+1} + \gamma_{t+1} f_i(S_{t+1},A_{t+1}) | S_t=s, A_t=a].
\end{equation*}
Plugging the definitions into Equation~\ref{eq:bonus_form}, we get
\begin{align*}
    b(s,a)  &= \max_{i\in[k]} |\hat{f}_i(s,a) - f_i(s,a)| \\
    &= \max_{i\in[k]} \big|\E[R_{i,t+1} +\gamma_{t+1}\hat{f}_i(S_{t+1},A_{t+1}) | S_t=s, A_t=a] + \epsilon_i(s,a) \\
    &\qquad- \E[R_{i,t+1} +\gamma_{t+1}f_i(S_{t+1},A_{t+1}) | S_t=s, A_t=a]\big| \\
    &= \max_{i\in[k]}\big|\E[\gamma_{t+1}(\hat{f}_i(S_{t+1}, A_{t+1}) - f_i(S_{t+1}, A_{t+1}))| S_t=s, A_t=a]  + \epsilon_i(s,a) \big|.
\end{align*}
\end{proof}
\section{Using the Ensemble of Value Functions}
We provide pseudocode in Algorithm \ref{alg_vbe}, for the case where the base algorithm is Double DQN, but it is possible to swap in many different off-policy value-based algorithms. Even actor-critic, which explicitly maintains a critic $q_w$, could easily incorporate the value bonuses by using an optimistic critic. For the purposes of this paper, however, we restrict our focus to Double DQN.  

The ensemble value functions are updated on the same target policy as Double DQN, namely the greedy policy in $q_w$. This choice comes from the fact that we want to understand uncertainty in the values for the target policy. We further investigate the impact of this choice in Appendix \ref{sp_tp_exp}. The update is similar to Double DQN, except the actions are sampled according to $q_w$ rather than $\ensembleq{i}$, and we use the ensemble reward $r_i$ defined above in Equation \eqref{eq:rvf_reward}:
\begin{equation}
w_i \gets w_i + \eta \delta_i \nabla \ensembleq{i}(s,a) \quad \text{ for } \delta \defeq r_{i} + \gamma \ensembleqt{i}(s', \argmax_{a'} q_w(s',a')) - \ensembleq{i}(s,a) \label{eq_evf}
\end{equation}
On each step, we only update one RQF predictor. Updating the entire ensemble is expensive, and arguably unnecessary. There are multiple ways to control the magnitude of the value bonus, and how quickly it decays. One way is the size of the ensemble, where the larger the ensemble, the more slowly this bonus should decay. Updating each RQF predictor less frequently, however, will also cause the bonus to decay more slowly. It both allows us to make the ensemble smaller, and ensure that regardless of the ensemble size, the computation per-step is simply double that of Double DQN: one update to the main value function and one update to an RQF predictor. 

%Finally, there is also a scale parameter on the value bonus. Previous work has proposed to keep running sums 

\begin{algorithm}[t]
	\caption{Value Bonuses with Ensemble Errors (VBE) }
	\begin{algorithmic}[1]
		\STATE \textbf{Parameters:} ensemble size $\esize$, bonus scale $\bscale$, target net update frequency $\tau$, batch size $m$
		\STATE Initialize empty buffer: \(B \leftarrow \emptyset\), action-value function: $q_w$, target RQFs: $\targetq{i}, \ldots \targetq{\esize}$, predictor RQFs: $\ensembleq{1}, \ldots, \ensembleq{\esize}$, and target networks: $q_{\tilde{w}}$, $\ensembleqt{1}, \ldots, \ensembleqt{\esize}$
            \STATE \textbf{Optimistic behavior policy:}\\ $\pi_b(s)  \leftarrow \argmax_{a \in \mathcal{A}} q_w(s, a) + \bscale \ \vb(s,a)$\\ where $\vb(s,a) \leftarrow \max_{i \in [\esize]} |\ensembleq{i}(s,a) - \targetq{i}(s,a)|$ 
		\STATE Get the initial state $s_{0}$
		\FOR{environment interactions $t = 0, 1, \ldots$}
		\STATE Take action $a \leftarrow \pi(s_t)$ and observe $r_{t+1}, s_{t+1}, \gamma_{t+1}$
		      \STATE Add $(s_t, a_t, r_{t+1}, s_{t+1}, \gamma_{t+1})$ to the buffer $B$
		    \STATE // Update action-values using DDQN update
		    \STATE Sample mini-batch from $B$, update $q_w$ using Eq. \eqref{eq_ddqn}
		    \STATE // Update one randomly selected RQF
		      \STATE Sample $i$ from $[\esize]$ uniform randomly
		    \STATE Sample mini-batch from $B$, update $\ensembleq{i}$ using Eq. \eqref{eq_evf} where\\
		    for each $(s,a,r,s',\gamma)$ replace $r$ with\\ $r_i \defeq \targetq{i}(s,a) - \gamma \targetq{i}(s',\argmax_{a' \in \Actions} q_w(s',a'))$
                	\IF{\(t+1 \mod \tau == 0\)}
	\STATE $\tilde{q}_w \gets q_w$ and for all $i$, $\ensembleqt{i} \gets \ensembleq{i}$
			\ENDIF
		\ENDFOR
	\end{algorithmic}
    \label{alg_vbe}
\end{algorithm}

\subsection{Guaranteeing Optimistic Initial Values}

In this section we show the ensemble size and bonus scale $\bscale$ can be set to obtain optimistic initial values with high probability. The result motivates that value bonuses with ensemble errors provide sufficient first-visit optimism. 
%We additionally want to guarantee that these bonuses shrink to zero eventually; we do not provide such a result in this work, but briefly discuss it at the end of this section. 

We assume that we initialize our action-values $q$ and target and predictor RQFs with a standard random neural network initialization. Given a fixed initialization, we can jointly reason about the deep setting and the linear setting. In the linear setting, for a given state, we have $n$-dimensional features $\phi(s), \phi_1, \ldots, \phi_\esize(s) \in \mathbb{R}^n$ for the $q$ and RQFs respectively. In the deep setting, the last layer of the fixed neural network specifies the features. We additionally assume that the features from this last layer are all normalized, $\| \phi_i(s) \|_2 = 1$ for all $s$, to simplify the proof; but this is not strictly necessary and a similar result can be obtained without normalization.  
\begin{assumption}
The feature functions $\phi(s), \phi_1, \ldots, \phi_\esize(s) \in \mathbb{R}^n$ are all unit length---have an $\ell_2$ norm of 1 for every state. All weights are iid sampled from $\mathcal{N}(0, 1/n)$, to get initializations $q(s,a) = \phi(s)^\top w_a$,  $\targetq{i}(s,a) = \phi_i(s)^\top w^*_{a,i}$ and $\ensembleq{i}(s,a) = \phi_i(s)^\top w_{a,i}$
\end{assumption}

\begin{proposition}
Let $z(\delta)$ be the z-value for a Gaussian $X \sim \mathcal{N}(0,1)$ where $\text{Pr}(X > z(\delta)) = 1-\delta$ for $\delta \in (0,1)$. Take any $q_{\text{max}} \in \mathbb{R}$ and any $\delta \in (0,1)$. Then for any $s,a$, if 
\begin{equation*}
    \bscale \ge \sqrt{\tfrac{n}{\pi}} \left( q_{\text{max}} - z(\delta/2)/\sqrt{n}\right) / \left(\log(\esize/2) - \log \log(2/\delta)  \right) 
\end{equation*}
then 
\begin{equation*}
    \quad \quad q(s,a) + \bscale \max_{i \in {1, \ldots, \esize}} |\targetq{i}(s,a) - \ensembleq{i}(s,a)| > q_{\text{max}}
\end{equation*}
with probability $1-\delta$.
\end{proposition}
\begin{proof}
Pick any $s \in \States, a \in \Actions$. Notice that because $w^*_{a,i,j} \sim \mathcal{N}(0,1/n)$ for all $j \in \{1, \ldots, n\}$, that the linear combination of these Gaussians $\targetq{i}(s,a) = \phi_i(s)^\top w^*_{a,i}$ is distributed as $\mathcal{N}(0,\| \phi_i(s)\|_2^2/n) = \mathcal{N}(0,1/n)$ because $\| \phi(s)\|_2^2 = 1$. Define $\delta_i \doteq \targetq{i}(s,a) - \ensembleq{i}(s,a) = \phi_i(s)^\top w^*_{a,i} - \phi_i(s)^\top w_{a,i}$, which is the difference of two $\mathcal{N}(0, 1/n)$ random variables and itself is again Gaussian with double the variance, $\mathcal{N}(0,\tfrac{2}{n})$. We can rewrite $\delta_i = \sqrt{\tfrac{2}{n}} Y_i$ where $Y_i \sim \mathcal{N}(0,1)$. 
Then we can leverage the result from \citep[Lemma 1]{ash2022anticoncentrated} that gives a lower bound on the maximum of the absolute value of $\esize$ standard normal Gaussians to get that with probability $1-\delta/2$
\begin{align*}
&\max_{i \in \{1,\ldots,\esize\}} |\delta_i|  = 
\sqrt{\tfrac{2}{n}} \max_{i \in \{1,\ldots,\esize\}} |Y_i|\\
&\ge \sqrt{\tfrac{2}{n}}  \sqrt{\tfrac{\pi}{2}} \sqrt{\log(\esize/2) - \log \log(2/\delta)}\\
&= \sqrt{\tfrac{\pi}{n}} \sqrt{\log(\esize/2) - \log \log(2/\delta)}
\end{align*}
Similarly to the RQFs, because $w_{a,j} \sim \mathcal{N}(0,1/n)$, we also have that $q(s,a) = \phi(s)^\top w_a$ is Gaussian with distribution $\mathcal{N}(0,\| \phi(s)\|_2^2/n) = \mathcal{N}(0,1/n)$. 
We can therefore say with probability $1-\delta/2$ that $q(s,a) > z(\delta/2)/\sqrt{n}$. Taking the union over these two high probability events, we get that with probability $1-\delta$, 
\begin{align*}
    &q(s,a) + \bscale \max_{i \in {1, \ldots, \esize}} |\targetq{i}(s,a) - \ensembleq{i}(s,a)| >\\ 
    & z(\delta/2)/\sqrt{n} + \bscale \sqrt{\tfrac{\pi}{n}} \sqrt{\log(\esize/2) - \log \log(2/\delta)}
\end{align*}
where plugging in the above $\bscale$ makes this second term equal $q_{\text{max}}$, giving us the desired result. 
\end{proof}

We can also ask what happens to the bonuses in VBE after initialization. Ideally, they eventually converge to zero, with the action-values converging and the behavior and target policies both converging to a greedy policy. This scenario goes beyond the convergence conditions discussed above in Section \ref{sec_rqfs} for fixed policies. In VBE, both our behavior policy and target policy are changing with time. Unfortunately, theory around TD does not address this scenario. There are some results for a fixed behavior policy for double Q-learning under linear function approximation \citep{zhao2021faster}, or for a variant of DQN with a fixed dataset \citep{wang2022convergent}. The issue with a changing behavior policy is that it changes the relative importance of states in the objective, and so the best value function may change as it changes how it trades off errors across states. In our realizable setting, this changing importance may be less important, because our RQF predictor can perfectly represent the target. In our own experiments, we found the value bonuses did always converge to zero. Nonetheless, we know of no theory that would allow us to guarantee this. 

\textbf{Connection to BDQN:} Though not obvious at first glance, there is a connection between RQFs and random prior functions in BDQN. In BDQN, the value function is $q_\theta = f_\theta + \bscale p$ for a random prior function $p$ that is not updated, prior scale $\bscale$, and a learned function $f_\theta$. Random priors were developed for stationary state distributions---though then applied to control---so let us consider the update for a fixed policy $\pi$. The update uses $a' \sim \pi(\cdot|s')$, giving 
% $r + \gamma (f_\theta(s', a') + \bscale p(s',a')) - (f_\theta(s, a) + \bscale p(s,a))
% = r - \bscale (p(s,a) - \gamma p(s',a')) + \gamma f_\theta(s', a')  - f_\theta(s, a)$.
\begin{align*}
% q_\theta(s, a) = f_\theta(s, a) + \bscale p(s,a) &= r + \gamma (f_\theta(s', a') + \bscale p(s',a'))\\
r - \bscale \underbrace{(p(s,a) - \gamma p(s',a'))}_{\text{RQF's reward}} + \gamma f_\theta(s', a') - f_\theta(s, a)
\end{align*}
This is a standard update with reward bonus $\bscale (p(s,a) - \gamma p(s',a'))$, and this bonus is a scaled negation of our reward in Equation \eqref{eq:rvf_reward}. With a fixed policy, we can separate the value function learning into $q^\pi$ that estimates the values for the rewards and $b^\pi$ that estimates the values for the reward bonuses. Namely, $f_\theta$ consists of $q^\pi+b^\pi$. As these functions converge, $b^\pi(s,a)$ approaches $-\bscale p(s,a)$ using the exact same argument to the one in our Proposition \ref{prop_main}, just negating the function $p$. Consequently, $f_\theta(s,a) + \bscale p(s,a) = q^\pi(s,a) + b^\pi(s,a) + \bscale p(s,a) = q^\pi(s,a) + (b^\pi(s,a) + \bscale p(s,a))$ goes to $q^\pi$ since $b^\pi(s,a) + \bscale p(s,a)$ eventually cancels. 

This argument is not how randomized priors are presented, but provides another intuitive interpretation. Further, it highlights a key difference between BDQN and VBE: BDQN takes a Thompson sampling approach to induce optimism, whereas VBE acts greedily with respect to optimistic value estimates. Another key point is that BDQN's prior-based bonus is scaled by the $\bscale$ parameter. The prior-based bonus can be seen as adding a fixed noise to the targets in the updates. Scaling the bonus term with $\bscale$ would increase the variance in the targets by a factor of $\bscale ^2$. This can make the optimization problem harder and may cause BDQN to be sensitive towards higher learning rates. Unlike BDQN, VBE does not incorporate the $\bscale$ parameter when estimating the RQFs, and only uses the bonus scale in the behavior policy.
%this interpretation highlights that separating out estimation of RQFs and taking a maximum over errors is a sensible way to incorporate random functions than is done inside BDQN. The above argument for the role of priors only holds for a fixed policy. But, if we switch to estimating values and bonuses separately, then it is more appropriate to change the value update to a Q-learning update with a maximum. The bonus continues to simply play the role of deep exploration, and focuses the ensemble on uncertaThere is a jump from the posterior interpration, to including these random functions inside the bootstrap update with a maximum over actions. The reason for this is that 

%\section{Random Value Functions and Optimism}
%\input{optimism.tex}

%\section{Control with Ensemble Errors}
%\input{algoname.tex}

% \section{OOPI}
% \label{sec:oopi}
% \input{oopi.tex}

\section{Experiments}
\label{sec:experiments}
We evaluate our proposed algorithm on four classic exploration environments and six Atari environments, particularly in comparison to BDQN and the reward bonus approaches ACB and RND. We first investigate the algorithms in a pure exploration setting, on Deepsea, where we evaluate state coverage. 
% Then we compare performance on the classic environments, and investigate the impact of the bonus scale and number of RQFs in the ensemble. We conclude with experiments in Atari, particularly highlighting that VBE scales successfully to this setting.
Then we compare performance on the classic environments, and conclude with experiments in Atari, highlighting that VBE scales successfully to this setting. We provide more analysis of VBE's sensitivity to its two parameters -- ensemble size, and bonus scale in Appendix~\ref{es_x_bs}.

% Changes since the environments section did not talk about atari
\subsection{Experimental Settings}\label{sec_algs_and_experiments}

The four classic exploration environments are Sparse Mountain Car, Puddle World, River Swim and Deepsea. 
%These four environments have varying requirements for exploration: Deepsea and River Swim are considered hard exploration environments, whereas Puddle World and Mountain Car require less exploration. 
Full details are in Appendix \ref{app_details}, but we list a few key details here. 
Mountain Car has two-dimensional continuous inputs with sparse rewards: the agent only receives a reward of 1 at the goal and 0 otherwise. Puddle World also has two-dimensional continuous inputs, noisy actions and highly negative rewards in puddles along the way to the goal. 
%River Swim and Deepsea were both designed as hard exploration problems, requiring persistent behavior with likely failure under dithering, $\epsilon$-greedy exploration. 
River Swim resembles a problem where a fish tries to swim upriver, with high reward (+1) upstream which is difficult to reach and, a lower but still positive reward (+0.005), which is easily reachable downstream. This environment has a single continuous state dimension in $[0,1]$, with stochastic displacement when taking actions left or right.\footnote{One seemingly innocuous but important point to highlight is that we flipped the observation such that the high reward is at observation $0$ and the lower reward is at observation $1$. We did so because the standard random initialization and ReLU activation often results in a higher value for a higher input, thus favouring the correct action in the standard variant and making algorithms like BDQN look artificially good. Our modified variant removes this inadvertent bias without changing the problem structure or difficulty in any way.} 
% We found deep RL agents trivially solving this original variant, but not for the reasons we would like! 
%
Deepsea is similar to River Swim, but is a two-dimensional grid world. Reaching the high-reward state requires the agent to take the action to go right every time. There is a penalty of \(\frac{0.01}{N} \) for taking the action right, except in the bottom right corner where there is a reward of 1 for taking the  action right. A policy that explores uniform randomly has the probability of \(2^{-N}\) of reaching the goal state in each episode. 

We use slightly different evaluation metrics for the various environments. River Swim is continuing, so we report accumulated reward over learning. For both Deepsea and Puddle World, we report the undiscounted episodic return. For Mountain Car, we report the discounted return, because for every successful episode, the undiscounted return is 1 and so not meaningful in this sparse variant. For all episodic environments, we report steps on the x-axis and the corresponding episodic return on y-axis. All results in the classic environments use 50000 steps and 30 runs, except Deepsea which uses 10000 episodes and 5 runs.
\begin{figure}[ht]
\vspace{-0.3cm}
\captionsetup[subfigure]%
     {justification=centering}
     \centering
    \begin{subfigure}{1.0\textwidth}
        \includegraphics[width=1\hsize]{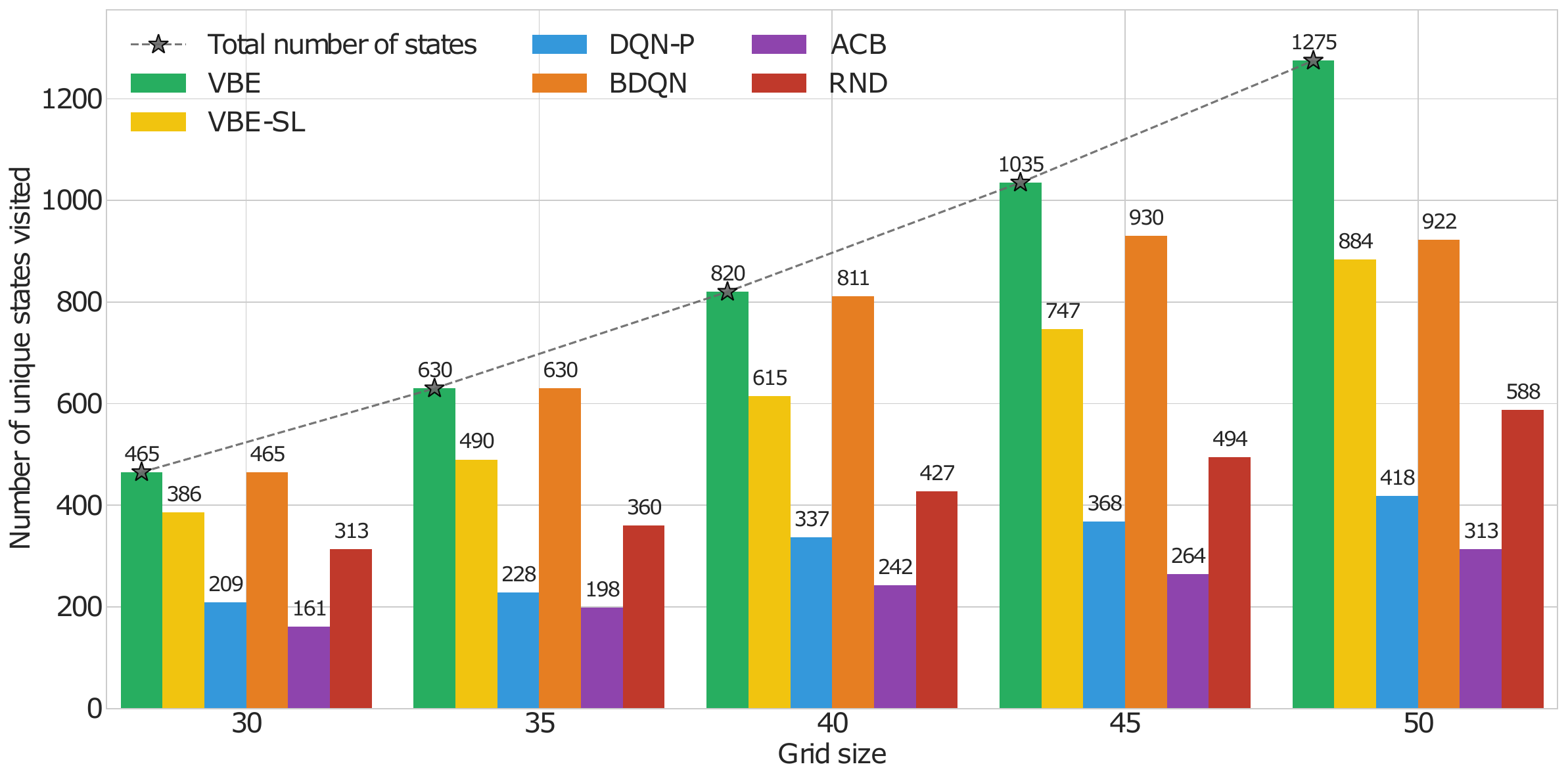}
        \subcaption{State coverage in Deepsea of different grid sizes}
        \label{fig:ds_pe_l_best}
    \end{subfigure}
    \hfill%
    \begin{minipage}{0.5\textwidth}
        \includegraphics[width=1\linewidth]{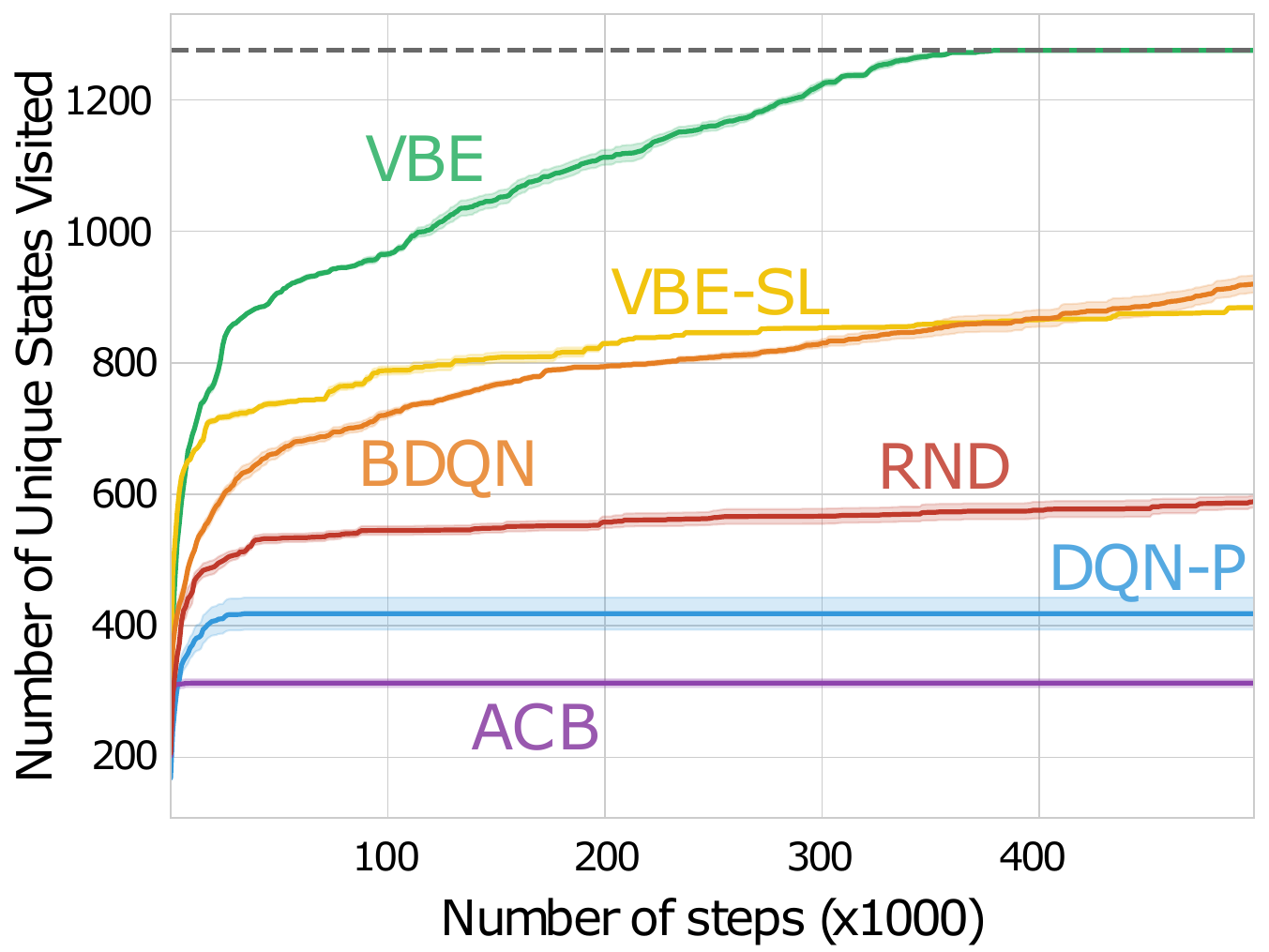}
        \subcaption{Progression of unique states visited (grid size 50)}
        \label{fig:ds_pe_l_50}
    \end{minipage}
    \hfill%
    \begin{minipage}{0.49\textwidth}
        \caption{Contrasting the state coverage abilities of exploration algorithms in DeepSea. In (a) each bar corresponds to the total number of unique states visited by an agent after completing 10,000 episodes. The black stars indicate the total number of unique states for each grid size. Notably, VBE covers the entire state space, even for the larger grid sizes. (b) displays the progression of unique states visited by agents over the course of learning for Deepsea with grid size 50. The dotted line represents the total number of unique states (1275) in this environment. It provides evidence that VBE consistently explores new states at a significantly higher rate.}
    \end{minipage}
\end{figure}
% \begin{figure*}
% \centering
% \vspace{-0.3cm}
% \includegraphics[width=1.0\hsize]{Figures/Pure_exploration/DeepSea Linear best performance grids barplot.pdf}
% \vspace{-0.3cm}
% \caption{Contrasting the state coverage abilities of exploration algorithms in Deepsea. Each bar corresponds to the total number of unique states visited by an agent after completing 10,000 episodes. The black stars indicate the total number of unique states for each grid size. Notably, VBE covers the entire state space, even for the larger grid sizes.}
% \label{fig:ds_pe_l_best}
% \end{figure*}

Across problems we compare VBE with DDQN-based variants of ACB and RND, DQN with additive priors (DQN-P) and BDQN. ACB, RND and VBE only differ in their value bonuses; we use the reward bonuses underlying ACB and RND to learn their respective value bonuses. As originally proposed, we make the reward-bonus value function non-episodic for ACB and RND. We also compared VBE with the released variants of ACB and RND that use PPO in Appendix \ref{app_ppo}; in general, we find this PPO version to be less sample efficient than the DDQN versions. DQN-P simply adds an additive prior to DQN, like BDQN; it can be seen as BDQN with one value function in the ensemble. We evaluate the algorithms using 1, 2, 8 and 20 value functions in the ensembles and bonus scales of 1, 3 and 10. To match their original implementation RND uses two deep neural networks with multiple (64) nodes in the final layer as the target and predictor network for the reward bonus. All methods use the same neural network architectures, detailed in Appendix \ref{app_details_classic_control}.

We also include VBE-SL, that uses supervised learning instead of a TD update for the RQFs, to ablate this component of VBE. We discussed in Section \ref{sec_rqfs} that the errors for VBE-SL likely reduce too quickly, resulting in insufficient exploration; we test that hypothesis here. Note that both VBE and VBE-SL only update their ensemble, with errors defining their value bonus, whereas ACB and RND both have to update their ensembles to get reward bonuses and learn a second value function to get the value bonus. 

\subsection{Pure Exploration}
\label{sec:pure-exploration}
% \begin{figure}
%     \begin{minipage}{0.5\textwidth}
%         \includegraphics[width=1\linewidth]{Figures/explore/deepsea_50.pdf}
%     \end{minipage}
%     \hfill%
%     \begin{minipage}{0.49\textwidth}
%         \caption{The progression of unique states visited by agents over the course of learning in Deepsea with grid size 50 -- dotted line represents the total number of unique states (1275). It provides evidence that VBE consistently explores new states at a significantly higher rate.}
%         \label{fig:ds_pe_l_50}
%     \end{minipage}
% \end{figure}
We first test how effectively the agents cover the state space in Tabular Deepsea with increasing grid sizes. In this setting, the agents observe no reward from the environment; thus, there are no harder to reach states, as in the original DeepSea environment. This allows us to evaluate agent's ability to do both directed deep exploration and employ first-visit optimism. For this tabular setting, the agents are otherwise the same as the other experiments, except the function approximation is linear in a one-hot encoding. 

Figure~\ref{fig:ds_pe_l_best} shows that VBE covers the entire state space for all the grid sizes. BDQN is able to cover the state space for a grid size of 30 and 35, but fails on bigger grids. Both ACB and RND fail to cover the state space, with ACB covering even less than DQN-P. This outcome is not surprising, given that neither approach ensures first-visit optimism. VBE-SL visits more unique states compared to ACB, RND and DQN-P. This is because VBE-SL provides first-visit optimism, encouraging the agent to take an action in a state if it has not done so before. But, as expected, it does not explore as much as VBE, likely as its value bonuses decay too quickly. 

These suboptimal behaviors are emphasized in Figure~\ref{fig:ds_pe_l_50} for a grid size of 50. All methods initially explore at a similar rate, easily reaching around 300 unique states. ACB, RND and DQN-P largely stop visiting new states very early in learning, though RND is slowly increasing the number of states it visits. BDQN and VBE-SL are similar across in their behavior, with VBE-SL exploring more early, possibly due to better first-visit optimism. Over time, however, BDQN starts to catch up and then surpasses VBE-SL. VBE is the only algorithm that maintains a consistent increase until it has seen all states. It is interesting to note that VBE is able to cover the state-space with just 1 RQF in the ensemble, whereas the rest of the algorithms require much bigger ensembles and still fail to cover the state-space (Table \ref{tab:pe_best_params}). This highlights the ability of VBE to provide first-visit optimism and do deep exploration.

\subsection{Comparison in Classic Environments}
\label{sec:classic_control}
\newcommand{\figthreewidth}{0.24}
\begin{figure}[H]
\centering
    \begin{subfigure}{\figthreewidth\textwidth}
        \includegraphics[width=\hsize]{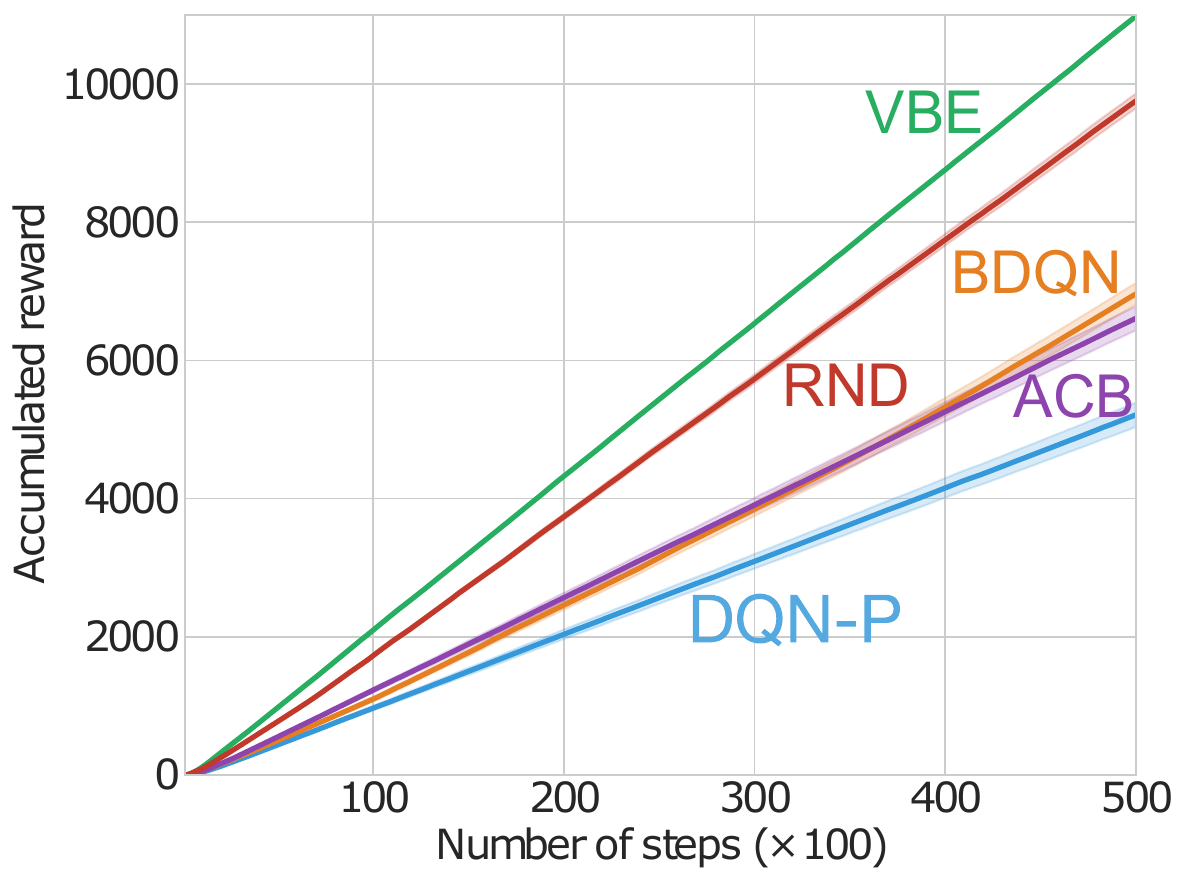}
        \caption{River Swim}
        \label{fig:rs_nn_best}
    \end{subfigure}
    \begin{subfigure}{\figthreewidth\textwidth}
        \includegraphics[width=\hsize]{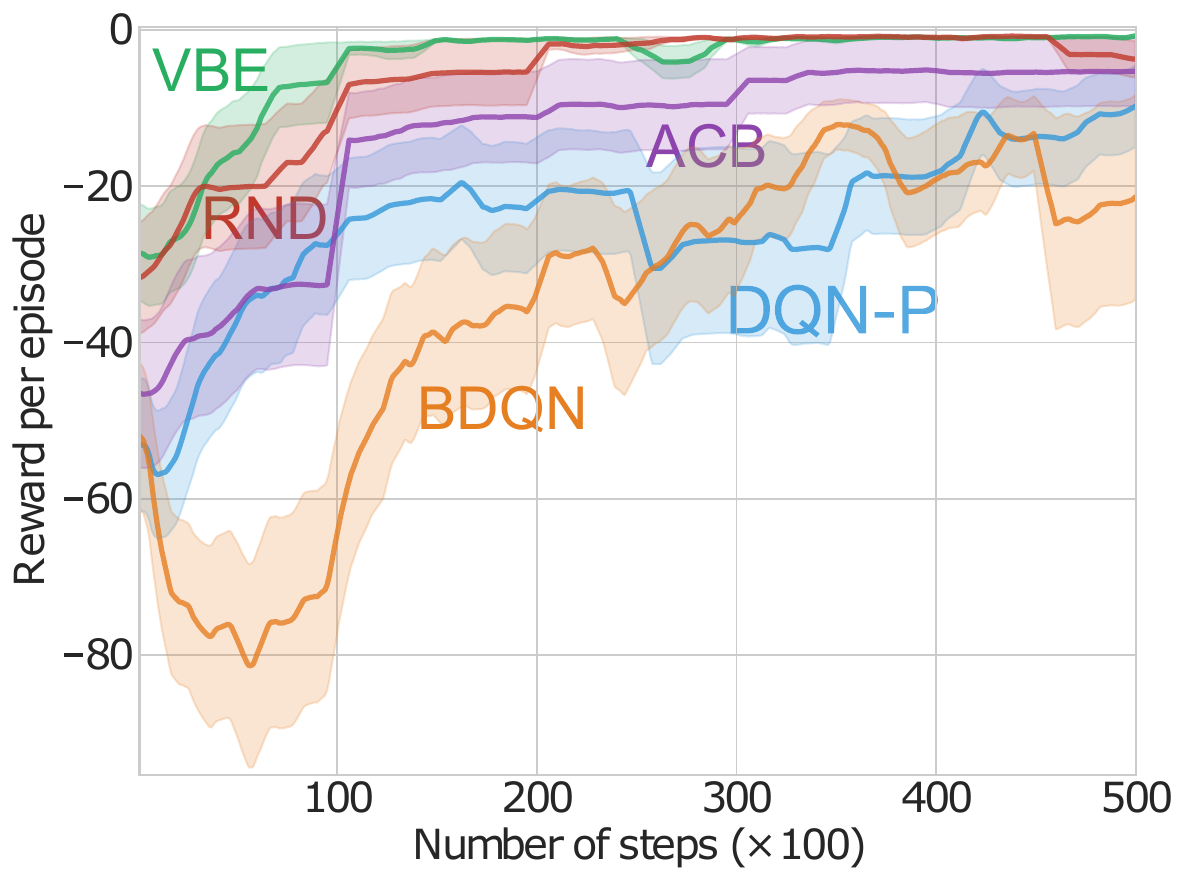}
        \caption{Puddle World}
        \label{fig:pw_nn_best}
    \end{subfigure}
    \begin{subfigure}{\figthreewidth\textwidth}
        \includegraphics[width=\hsize]{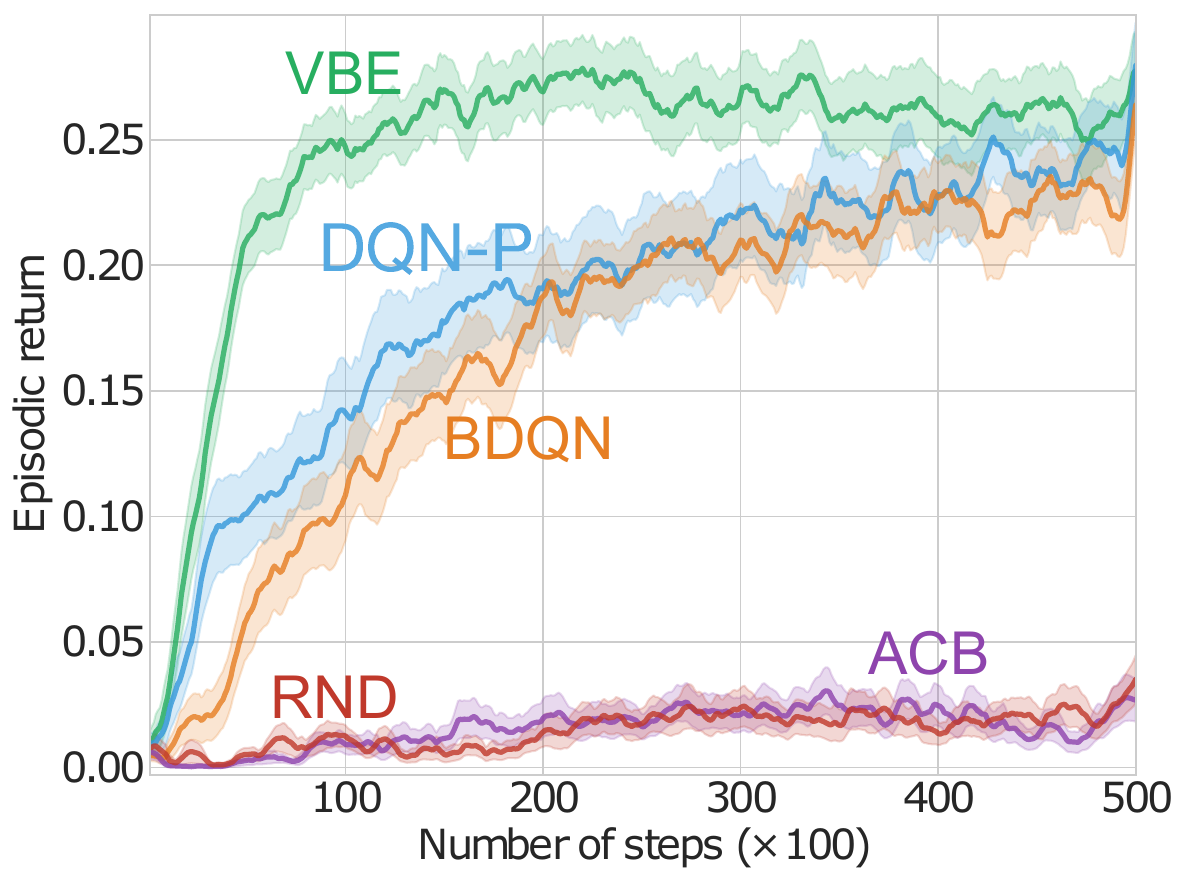}
        \caption{Mountain Car}
        \label{fig:mc_nn_best}
    \end{subfigure}
    \begin{subfigure}{\figthreewidth\textwidth}
        \includegraphics[width=\hsize]{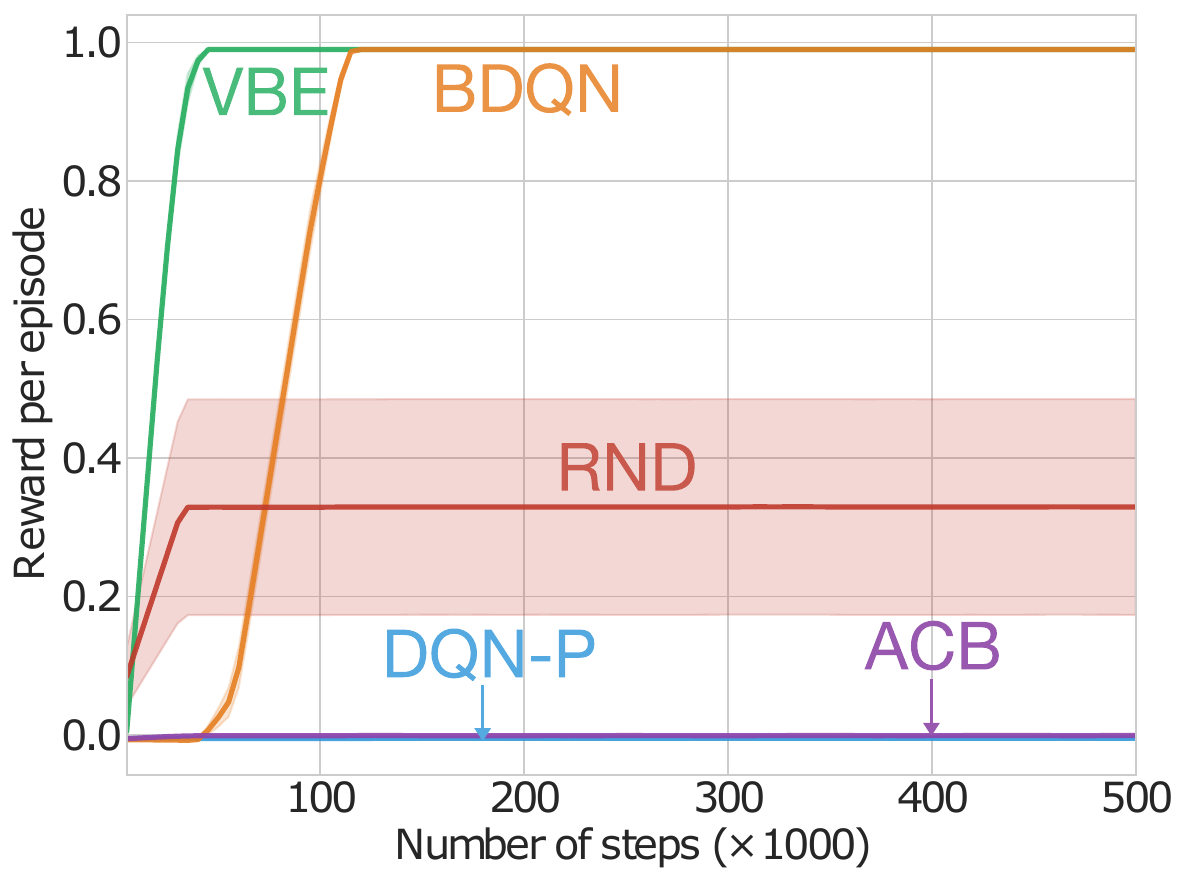}
        \caption{Deepsea}
        \label{fig:ds_nn_best}
    \end{subfigure}
    \vspace{-0.3cm}
    \caption{Online performance in River Swim, Puddle World, Mountain Car, and Deepsea. Higher on the y-axis is better. The x-axis denotes the number of interaction steps with the environment. The shaded region corresponds to standard errors.
    }
    \label{fig:nns_best}
\end{figure}
In this section we compare VBE with DQN-P, BDQN, ACB, and RND on the four classic control environments. In Figure~\ref{fig:nns_best}, we see that VBE learns faster and reaches the best final performance in all four environments. Surprisingly, DQN-P is competitive with BDQN in three out of the four environments. 
% MARTHA: Maybe we dont need to say more about DQN-P
In Deepsea, where persistent optimism is essential to reach the state with high reward, DQN-P fails. 
ACB and RND both fail to learn in the sparse reward domain Mountain Car, whereas, in Puddle World which has a denser reward structure, they perform better. 
RND is competitive in River Swim and Puddle World, however, it fails to learn the optimal policy in the majority of runs in Deepsea. ACB is competitive in Puddle World but fails in Deepsea. We compare VBE with PPO-based variants of ACB and RND in classic control environments in Appendix \ref{app_ppo_control}, and their linear counterparts in Appendix \ref{sec:apdx_lfa_control}.
\subsection{Atari}
\label{sec:atari}
\begin{figure*}[t]
    \centering
    \vspace{-0.3cm}
    \begin{subfigure}{0.32\textwidth}
        \includegraphics[width=\hsize]{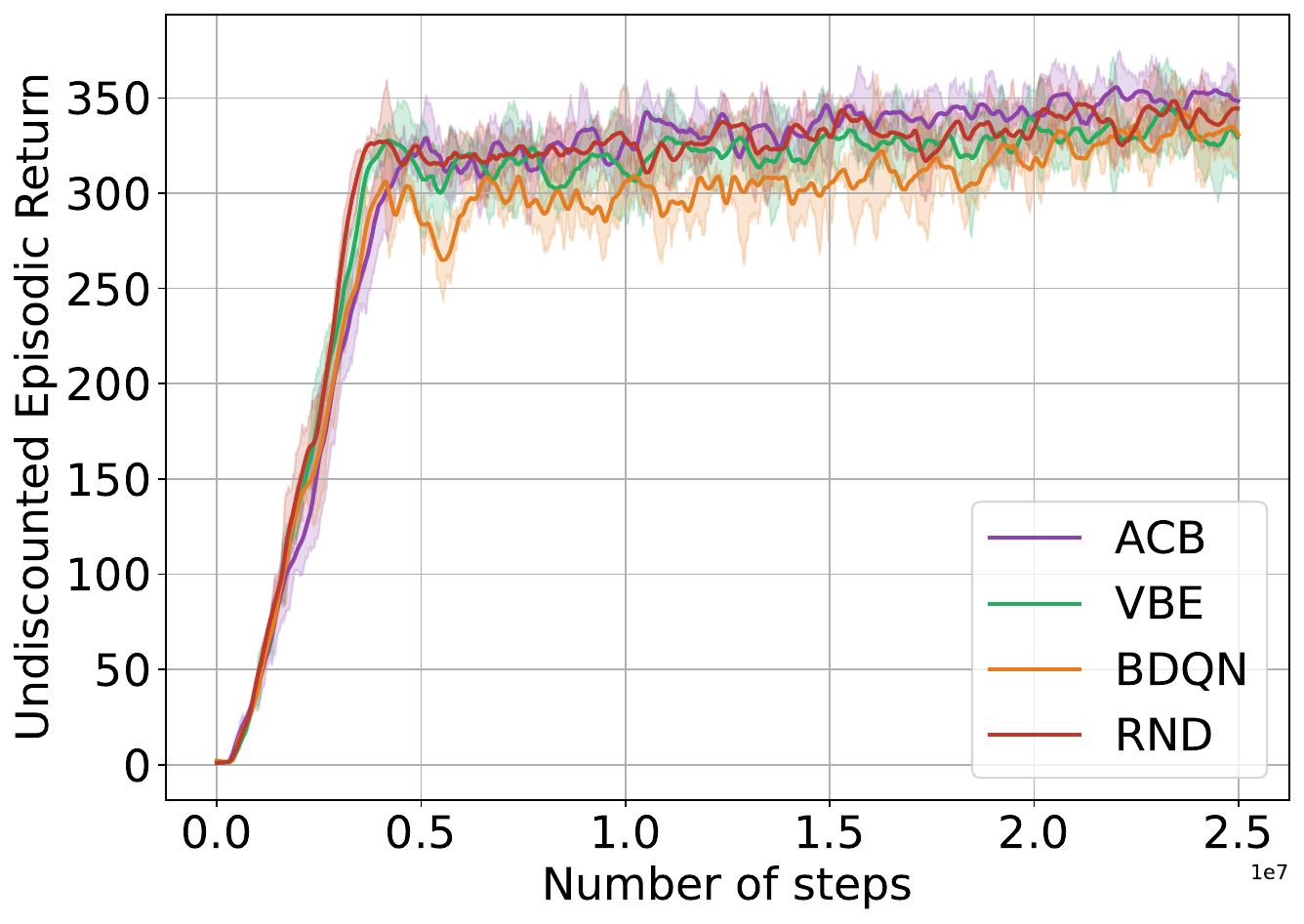}
        \caption{Breakout}
        \label{fig:vbe_breakout}
    \end{subfigure}
    \begin{subfigure}{0.32\textwidth}
        \includegraphics[width=\hsize]{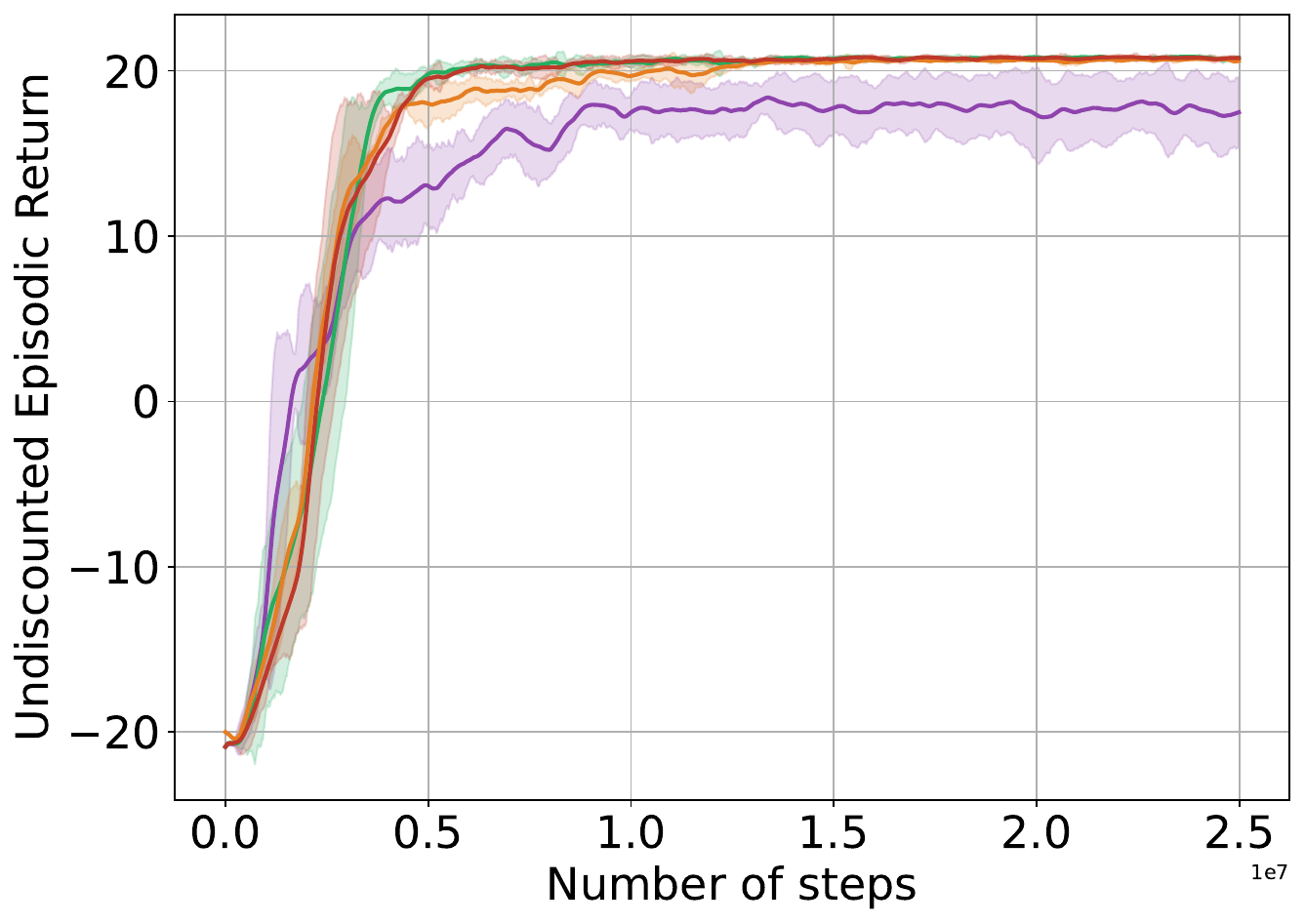}
        \caption{Pong}
        \label{fig:vbe_pong}
    \end{subfigure}
    \begin{subfigure}{0.32\textwidth}
        \includegraphics[width=\hsize]{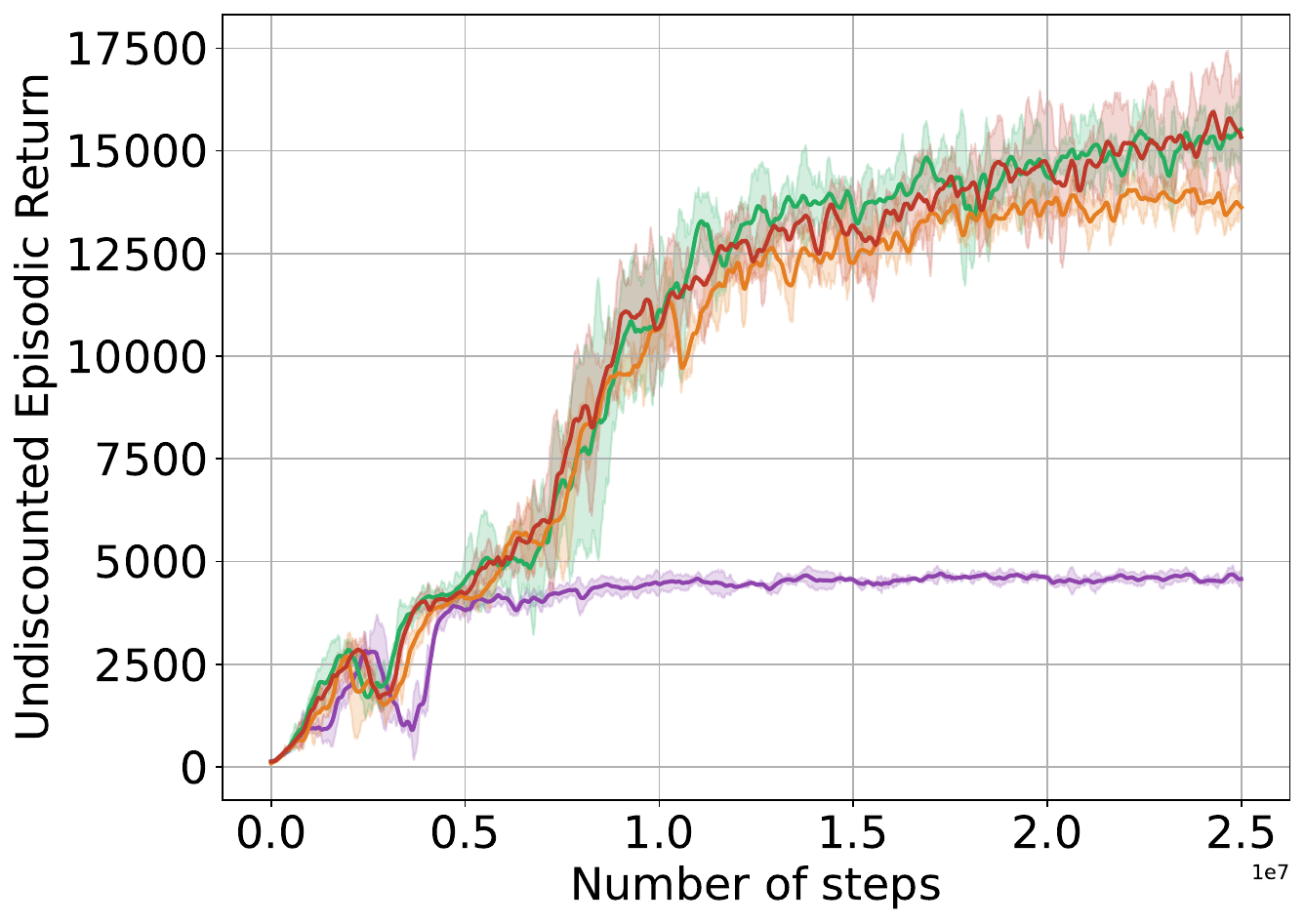}
        \caption{Q$^*$bert}
        \label{fig:vbe_qbert}
    \end{subfigure}
    \begin{subfigure}{0.32\textwidth}
        \includegraphics[width=\hsize]{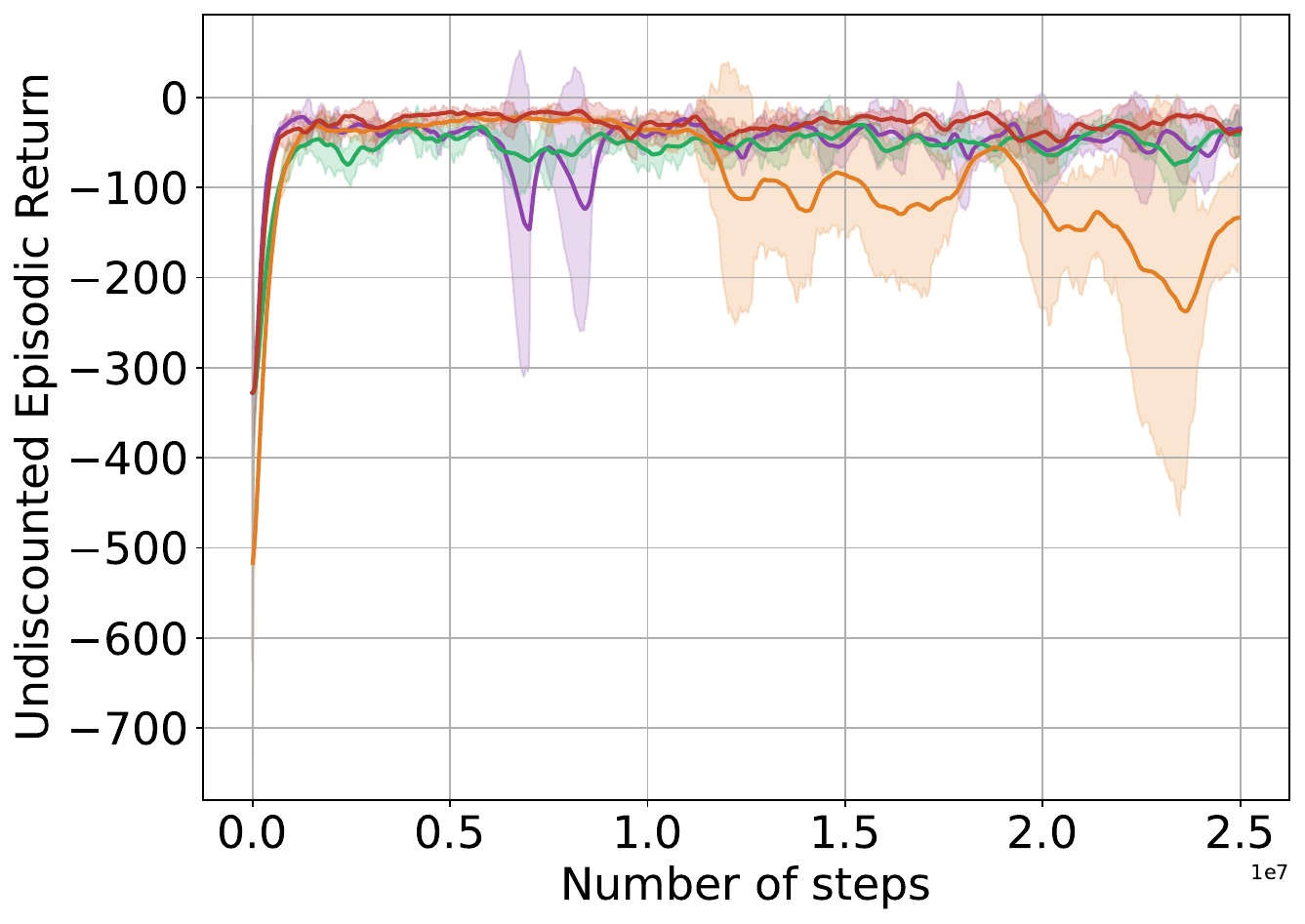}
        \caption{Pitfall}
        \label{fig:vbe_pitfall}
    \end{subfigure}
    \begin{subfigure}{0.32\textwidth}
        \includegraphics[width=\hsize]{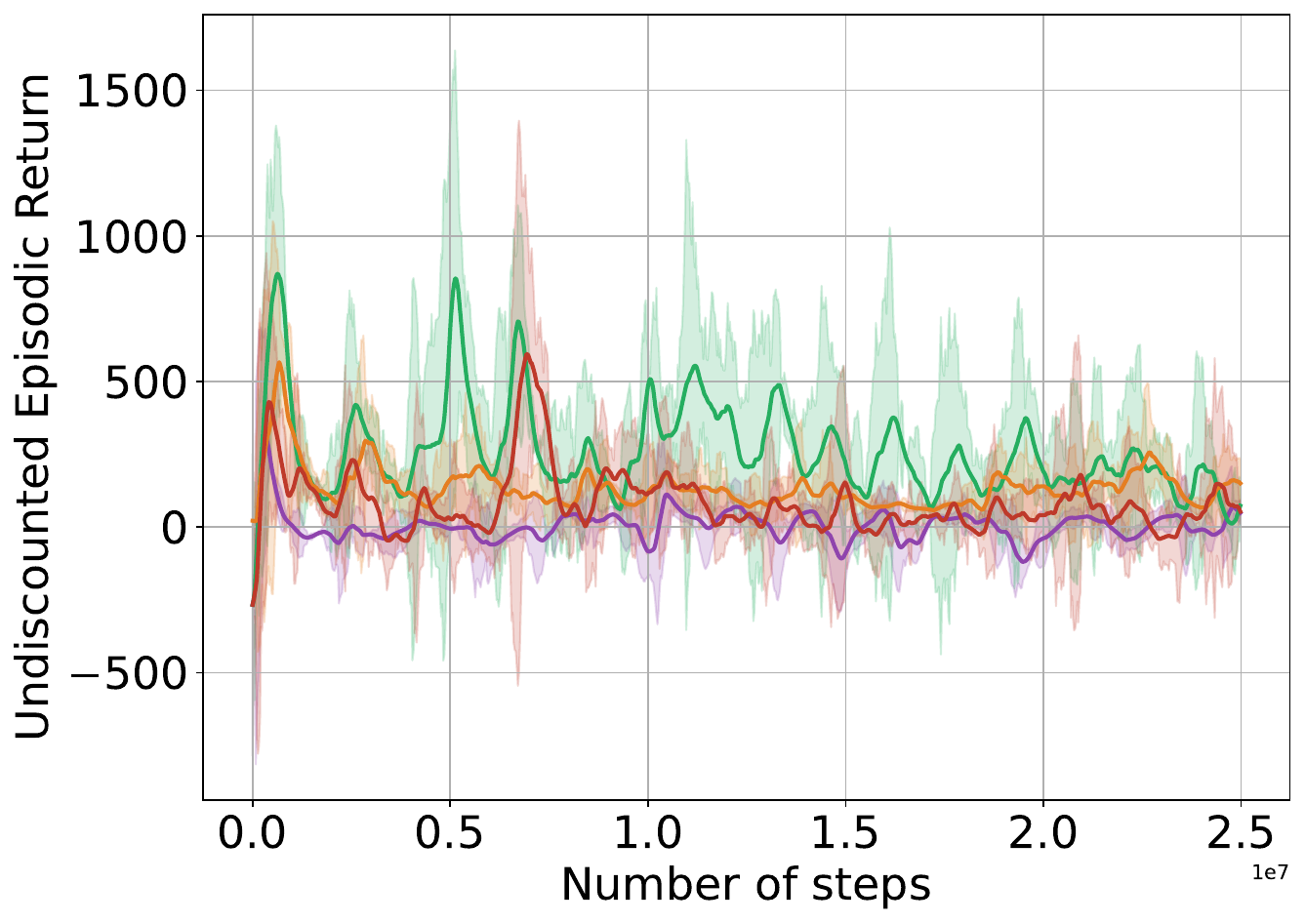}
        \caption{Private-Eye}
        \label{fig:vbe_privateeye}
    \end{subfigure}
    \begin{subfigure}{0.32\textwidth}
        \includegraphics[width=\hsize]{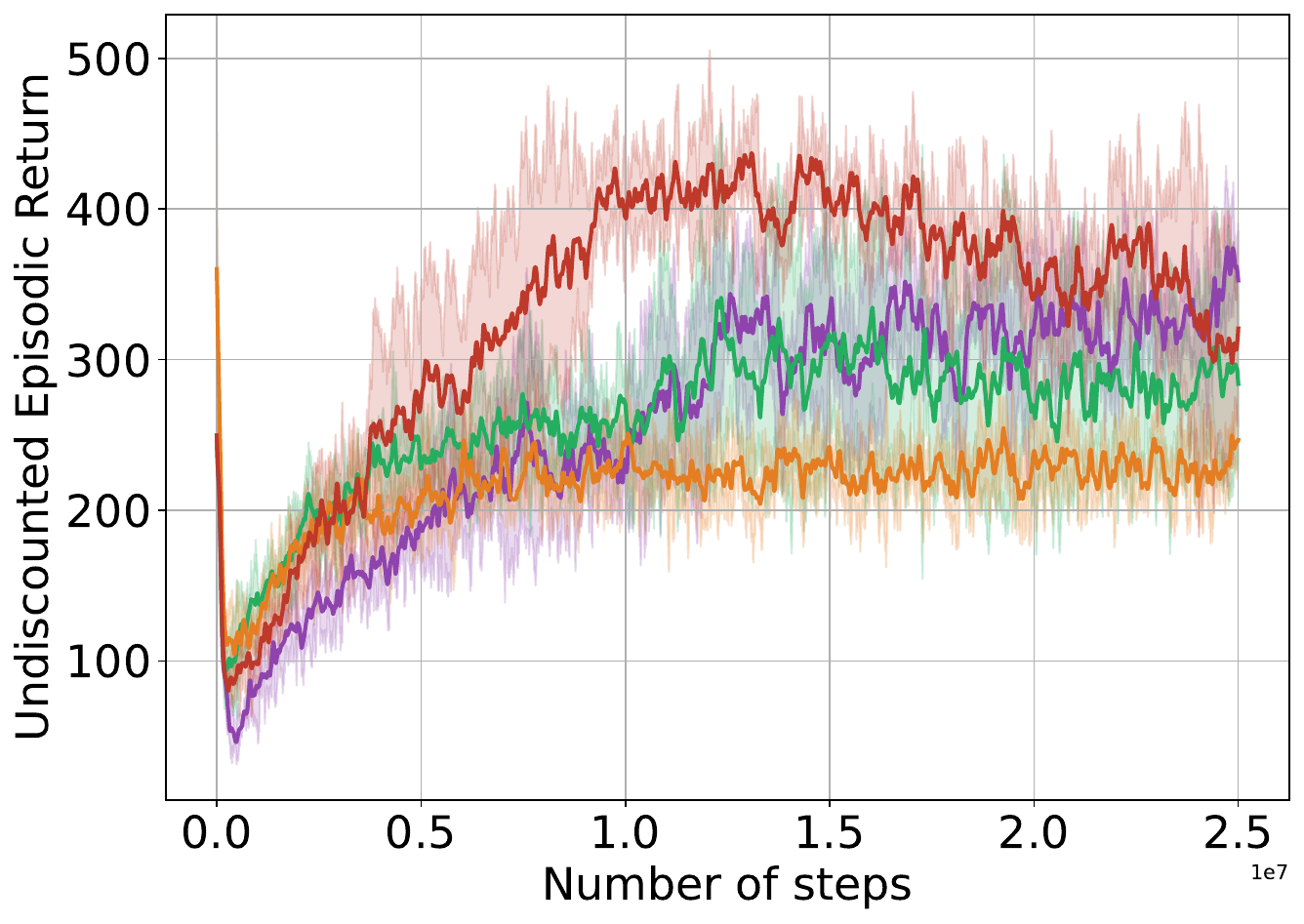}
        \caption{Gravitar}
        \label{fig:vbe_gravitar}
    \end{subfigure}
    \vspace{-0.3cm}
    \caption{Online performance in six Atari games, with shaded regions corresponding to standard errors. The x-axis is the number of environment interaction steps in millions, and the y-axis is the online Undiscounted Episodic Return, where higher is better. The environments in the second row are considered to be more challenging in terms of exploration.}
    \label{fig:vbe_best}
    \vspace{-0.5cm}
\end{figure*}
In this section we test VBE on several hard exploration Atari games, namely Private-Eye, Pitfall, Gravitar \citep{burda2018exploration}, and also on Breakout, Pong and Qbert. We chose this set to ensure a good mix of both hard and easy exploration environments \cite{taiga2021bonus}. As is standard, we combine four consecutive frames to make the observation ($4\times84\times84$), and update agents ever four steps. We clip the rewards between $[-1, 1]$, and do 3 runs for all agents for 25 million steps. 
%We use standard choices, similar experimental setup as used in previous work BDQN, ACB, and RND, i.e., combining four consecutive frames to make the observation ($4\times84\times84$), and skip four frames after the agent takes an action. 
% We run the agents for a total of 4 million steps. 
% %, and show the online performance for each environment in Figure~\ref{fig:vbe_best}. 

% In VBE the target and predictor RQFs have 3 CNN-layers, followed by 2 non-linear layers (representation network) and a final linear output layer. We only update the final layer of the predictor RQFs, and initialize the predictors to have the same representation network as the target RQFs. In practice we found that even $1$ RQF works well for VBE, so that is what we use in our experiments. 
% We chose this configuration because it was faster to run and was more stable than updating the whole network. We include the comparison to the variant of VBE where we update the whole network for the RQFs in Appendix ~\ref{sec:apdx_atari_vbe_vbe_cnn}.

For BDQN, we follow the choices from \citep{osband2016deep}, and use an ensemble size of $10$, without bootstrapping.\footnote{The original BDQN implementation for Atari does not use bootstrapping, that is, all members of the ensemble see the same data.} As in the original BDQN implementation, each value function in the ensemble uses a shared representation network. We use a DDQN update for BDQN to be fair and consistent with the rest of the agents. For VBE, we also use an ensemble size of $10$ and use a shared representation for the RQFs. To further improve the computational complexity we only update the RQF heads. ACB uses an ensemble size of $k=128$ for computing the reward bonus, and RND uses a CNN-based target and predictor with $512$ nodes in the final layer. We use $\bscale$ = 10 for all the agents in all the environments. Additional implementation details for Atari are mentioned in Appendix \ref{app_details_atari}

In Figure~\ref{fig:vbe_best} we see that VBE performs consistently better than BDQN in all Atari environments. In Breakout and Pong, it takes longer for BDQN to achieve the same level of performance as the other agents. This could be attributed to random sampling of the value function for the behavior policy, thus requiring more time for all value functions to adjust. In Qbert and Gravitar, BDQN plateaus before the other agents, suggesting less exploration in these environments compared to the other agents. In Pitfall, BDQN does well in the beginning, but then fails to maintain a good policy in the later stages of training. ACB shows early learning in Pong, but then fails to converge to an optimal policy. Similarly in Qbert, ACB converges to a suboptimal policy, suggesting that ACB's value bonus is not suitable in these environments. In Breakout and Pitfall, ACB performs similarly to VBE, with ACB being slightly better in Breakout towards the end. In Gravitar, ACB is slow in the early stages of the training but then surpasses VBE in the later stages. RND performs very similar to VBE in Breakout, Pong, and Qbert, and is slightly better in Pitfall. However, in Gravitar, RND seems to do much better than all the other agents initially, but then fails to maintain that performance and drops down matching VBE. In Private-Eye none of the agents do well; however, VBE seems to be collecting a higher reward more often than the other agents. 

The exact reason for an agent's performance in complex environments like Atari is hard to discern, and is something that requires additional understanding. However, after showcasing the efficacy of VBE in more controlled settings, the purpose of this experiment is to demonstrate that VBE can easily be extended to more complex environments and can even surpass or be at par with other exploratory baselines. In Appendix \ref{sec:apdx_atari_vbe_vbe_cnn} we demonstrate that variants of VBE can further improve on these results.

\section{Conclusion}
In this work we introduced a new approach to do directed exploration in deep RL, called Value Bonuses with Ensemble errors (VBE). The utility of value bonuses is that it is simple to layer on top of an existing algorithm: the value bonuses are separately estimated and only impact the behavior policy. Improving how we estimate value bonuses, therefore, provides a promising path to replacing simple, but undirected exploration strategies like $\epsilon$-greedy. To date, the primary way to estimate value bonuses has been to estimate a separate value function on reward bonuses, as was done for ACB and RND. However, this approach does not encourage first-visit optimism; it only encourages revisiting an action once a reward bonus is observed. We show that, in general, ACB and RND fail to provide effective exploration in classic control environments. In contrast, on Atari, VBE scales well to more complex environments and performs on par with, if not better than, existing baselines.
% \clearpage
% acknowledgments, including those to funding agencies, go at the end of the paper

\clearpage
\appendix
\section*{Appendix: A Discussion on Convergence Criteria for Value Bonuses}\label{app_theory}

First let us discuss how the theory for LSTD applies to our setting. The result from \citep[Corollary 1]{tagorti2015rate} bounds the error of the value function learned under LSTD to the true value function, assuming features are linearly independent (Assumption 1) and a mixing assumption for the environment and behavior policy (Assumption 2). This bound includes an error to the best linear solution, for infinite data, and the error between the best linear solution and the true value function. Because we are in the realizable case and the objective is convex for linear function approximation, the best linear solution is the true value function and in the limit of data the LSTD solution will reach this best linear solution. We can write this as a corollary of their result. Note their result is written by value functions, but automatically extends to action-value function by considering state-action features and stationary distribution $\mu_b(s,a) = \mu(s) \pi_b(a|s)$.

\begin{corollary}[Corollary following from [Theorem 1]\citep{tagorti2015rate}] Assume we are given behavior policy $\pi_b$ with stationary distribution $\mu$ and target policy $\pi$ and the rewards are defined using a randomly sampled $\targetq{i}$ from the set of linear functions on features $\phi(s,a)$ and the formula in Equation \eqref{eq:rvf_reward}. Under Assumption 1 and 2 from \citep{tagorti2015rate}, for a large enough number of samples $T$ given by \citep[Eq 6]{tagorti2015rate} (called $n$ in their result), then $\ensembleq{i}$ returned by LSTD satisfies
\begin{equation*}
\mathbb{E}_{s\sim\mu a\sim\pi_b(\cdot|s)}[(\ensembleq{i}(s,a) - \targetq{i}(s,a))^2] \le O(1/\sqrt{T})
\end{equation*}
\end{corollary}

Now let us discuss how the work on neural TD applies to our setting \citep{cai2019neural}.
%If we satisfy the required conditions for $\ensembleq{i}$ to be able to converge to $\targetq{i}$, then the value bonuses will similarly converge to zero. There are recent results showing that temporal difference (TD) learning with an overparameterized neural network in the on-policy setting converges to the global solution \citep{cai2019neural}. 
The result is proved for neural networks with a single hidden layer using a ReLU activation for the hidden layer, with the additional condition that the stationary distribution for the policy has a bounded density over states and the stepsizes decrease at a rate of $1/\sqrt{t}$. This result immediately implies that our $\ensembleq{i}$ should converge to $\targetq{i}$, because the global solution for this problem is $\targetq{i}$ because it is in the value function class. We state this as a corollary of their result here, to be clear about how it applies.

\begin{corollary}[Corollary following from [Theorem 4.6]\citep{cai2019neural}] Assume that 1) the policy $\pi$ is fixed with stationary distribution $\mu$, where $\mu(s) \pi(a|s)$ has bounded density across the space $x = (s, a)$ 2) the function class $\mathcal{F} = \{ \frac{1}{\sqrt{m}} \sum_{j=1}^m b_j \max(x^\top w_j, 0) | W = (b_1, \ldots, b_m, w_1, \ldots, w_m), ||W - W(0) ||_2 \le B\}$ for $x = (s, a)$, $W(0)$ a point at which the weights are initialized in the algorithm and $B$ some constant, 3)  $\|x\|_2 = 1$ for all $x$ and the rewards are defined using a randomly sampled $\targetq{i}$ from $\mathcal{F}$ and the formula in Equation \eqref{eq:rvf_reward}, and 4) the Neural TD algorithm (Algorithm 1 in \citep{cai2019neural}) is run for $T$ steps with stepsize $\eta = \min((1-\gamma)/8,1/ \sqrt{T})$. Then the algorithm returns $\ensembleq{i}$ that satisfies
\begin{equation*}
\mathbb{E}_{W \sim , \mu\pi}[(\ensembleq{i}(s,a) - \targetq{i}(s,a))^2] \le \frac{O(B^2)}{(1-\gamma)^2 \sqrt{T}} + O(B^2 m^{-1/2} + B^{5/2} m^{-1/4})
\end{equation*}
\end{corollary}
\begin{proof}
The result also requires that the reward magnitudes are all bounded, which they are by construction. Theorem 4.6 states that the outputted action-value function is bounded as above to the global optimum in the function class. Because $\targetq{i}(s,a)$ is in the function class, we know it is the global optimum. 
\end{proof}

\subsubsection*{Acknowledgments}
\label{sec:ack}
We would like to thank NSERC, CIFAR, and Amii for research funding and the Digital Research
Alliance of Canada for computational resources.

\clearpage
\bibliography{main}

@inproceedings{janz2019successor,
 author = {Janz, David and Hron, Jiri and Mazur, Przemys\l aw and Hofmann, Katja and Hern\'{a}ndez-Lobato, Jos\'{e} Miguel and Tschiatschek, Sebastian},
 booktitle = {Advances in Neural Information Processing Systems},
 title = {Successor Uncertainties: Exploration and Uncertainty in Temporal Difference Learning},
 year = {2019}
}

@inproceedings{osband2023approximate,
  title = {Approximate {{Thompson Sampling}} via {{Epistemic Neural Networks}}},
  booktitle = {Uncertainty in {{Artificial Intelligence}}},
  author = {Osband, Ian and Wen, Zheng and Asghari, Seyed Mohammad and Dwaracherla, Vikranth and Ibrahimi, Morteza and Lu, Xiuyuan and Van Roy, Benjamin},
  year = {2023}
}

@inproceedings{cai2019neural,
  title = {Neural {{Temporal-Difference Learning Converges}} to {{Global Optima}}},
  booktitle = {Advances in {{Neural Information Processing Systems}}},
  author = {Cai, Qi and Yang, Zhuoran and Lee, Jason D and Wang, Zhaoran},
  year = {2019}
}

@inproceedings{tagorti2015rate,
  title = {On the {{Rate}} of {{Convergence}} and {{Error Bounds}} for {{LSTD}}({$\lambda$})},
  booktitle = {Proceedings of the 32nd {{International Conference}} on {{Machine Learning}}},
  author = {Tagorti, Manel and Scherrer, Bruno},
  year = {2015},
  publisher = {{PMLR}},
}

@article{patterson2022generalized,
  title={A generalized projected bellman error for off-policy value estimation in reinforcement learning},
  author={Patterson, Andrew and White, Adam and White, Martha},
  journal={The Journal of Machine Learning Research},
  year={2022},
}

@inproceedings{dai2017learning,
  title={Learning from conditional distributions via dual embeddings},
  author={Dai, Bo and He, Niao and Pan, Yunpeng and Boots, Byron and Song, Le},
  booktitle={Artificial Intelligence and Statistics},
  year={2017},
  organization={PMLR}
}

@inproceedings{van2016deep,
  title={Deep reinforcement learning with double q-learning},
  author={Van Hasselt, Hado and Guez, Arthur and Silver, David},
  booktitle={Proceedings of the AAAI conference on artificial intelligence},
  year={2016}
}

@inproceedings{ostrovski2017count,
author = {Ostrovski, Georg and Bellemare, Marc G and van den Oord, A{\"a}ron and Munos, R{\'e}mi},
title = {Count-Based Exploration with Neural Density Models},
booktitle = {International Conference on Machine Learning},
year = {2017}
}

@article{Bellemare2016unifying,
author = {Bellemare, Marc G and Srinivasan, Sriram and Ostrovski, Georg and Schaul, Tom and Saxton, David and Munos, Remi},
title = {Unifying Count-Based Exploration and Intrinsic Motivation},
journal = {Advances in Neural Information Processing Systems},
year = {2016}
}

@inproceedings{grande2014sample,
author = {Grande, R and Walsh, T and How, J},
title = {Sample Efficient Reinforcement Learning with Gaussian Processes},
booktitle = {International Conference on Machine Learning},
year = {2014}
}

@inproceedings{osband2016generalization,
author = {Osband, I and Van Roy, B and Wen, Z},
title = {Generalization and Exploration via Randomized Value Functions},
booktitle = {International Conference on Machine Learning},
year = {2016}
}

@inproceedings{osband2016deep,
author = {Osband, Ian and Blundell, Charles and Pritzel, Alexander and Van Roy, Benjamin},
title = {Deep Exploration via Bootstrapped DQN},
booktitle = {Advances in Neural Information Processing Systems},
year = {2016}
}

@inproceedings{odonoghue2018uncertainty,
  title = {The {{Uncertainty Bellman Equation}} and {{Exploration}}},
  booktitle = {International {{Conference}} on {{Machine Learning}}},
  author = {O'Donoghue, Brendan and Osband, Ian and Munos, Remi and Mnih, Vlad},
  year = {2018}
}

@inproceedings{zhao2021faster,
  title={Faster non-asymptotic convergence for double Q-learning},
  author={Zhao, Lin and Xiong, Huaqing and Liang, Yingbin},
  booktitle={Advances in Neural Information Processing Systems},
  year={2021}
}

@inproceedings{wang2022convergent,
  title={Convergent and efficient deep Q network algorithm},
  author={Wang, Zhikang T and Ueda, Masahito},
  booktitle = {International {{Conference}} on {{Learning Representations}}},
  year={2022}
}

@inproceedings{white2017unifying,
author = {White, Martha},
title = {Unifying task specification in reinforcement learning},
booktitle = {International Conference on Machine Learning},
year = {2017}
}

@inproceedings{li2010acontextual,
author = {Li, Lihong and Chu, Wei and Langford, John and Schapire, Robert E},
title = {A contextual-bandit approach to personalized news article recommendation},
booktitle = {World Wide Web Conference},
year = {2010}
}

@article{abbasi2011improved,
  title={Improved algorithms for linear stochastic bandits},
  author={Abbasi-Yadkori, Yasin and P{\'a}l, D{\'a}vid and Szepesv{\'a}ri, Csaba},
  journal={Advances in neural information processing systems},
  volume={24},
  year={2011}
}

@article{abbasi2019exploration,
  title={Exploration-enhanced politex},
  author={Abbasi-Yadkori, Yasin and Lazic, Nevena and Szepesvari, Csaba and Weisz, Gellert},
  journal={arXiv preprint arXiv:1908.10479},
  year={2019}
}

@article{wang2019optimism,
  title={Optimism in reinforcement learning with generalized linear function approximation},
  author={Wang, Yining and Wang, Ruosong and Du, Simon S and Krishnamurthy, Akshay},
  journal={arXiv preprint arXiv:1912.04136},
  year={2019}
}

@article{kumaraswamy2018context,
  title={Context-dependent upper-confidence bounds for directed exploration},
  author={Kumaraswamy, Raksha and Schlegel, Matthew and White, Adam and White, Martha},
  journal={arXiv preprint arXiv:1811.06629},
  year={2018}
}

@inproceedings{
burda2018exploration,
title={Exploration by random network distillation},
author={Yuri Burda and Harrison Edwards and Amos Storkey and Oleg Klimov},
booktitle={International Conference on Learning Representations},
year={2019},
url={https://openreview.net/forum?id=H1lJJnR5Ym},
}

@article{osband2018randomized,
author = {Osband, Ian and Aslanides, John and Cassirer, Albin},
title = {{Randomized Prior Functions for Deep Reinforcement Learning.}},
journal = {NeurIPS},
year = {2018}
}

@article{choshen2018dora,
  author    = {Leshem Choshen and
               Lior Fox and
               Yonatan Loewenstein},
  title     = {{DORA} The Explorer: Directed Outreaching Reinforcement Action-Selection},
  journal   = {CoRR},
  year      = {2018}
}

@inproceedings{
ash2022anticoncentrated,
title={Anti-Concentrated Confidence Bonuses For Scalable Exploration},
author={Jordan T. Ash and Cyril Zhang and Surbhi Goel and Akshay Krishnamurthy and Sham M. Kakade},
booktitle={International Conference on Learning Representations},
year={2022},
url={https://openreview.net/forum?id=RXQ-FPbQYVn}
}

@article{osband2019deep,
  title={Deep Exploration via Randomized Value Functions.},
  author={Osband, Ian and Van Roy, Benjamin and Russo, Daniel J and Wen, Zheng and others},
  journal={J. Mach. Learn. Res.},
  volume={20},
  number={124},
  pages={1--62},
  year={2019}
}

@inproceedings{
ciosek2020conservative,
title={Conservative Uncertainty Estimation By Fitting  Prior Networks},
author={Kamil Ciosek and Vincent Fortuin and Ryota Tomioka and Katja Hofmann and Richard Turner},
booktitle={International Conference on Learning Representations},
year={2020},
url={https://openreview.net/forum?id=BJlahxHYDS}
}

@article{taiga2021bonus,
  title={On bonus-based exploration methods in the arcade learning environment},
  author={Taiga, Adrien Ali and Fedus, William and Machado, Marlos C and Courville, Aaron and Bellemare, Marc G},
  journal={arXiv preprint arXiv:2109.11052},
  year={2021}
}
\bibliographystyle{rlj}

\beginSupplementaryMaterials
\section{Experiment Details}\label{app_details}

\subsection{Environment Details}
\textbf{Mountain Car} is classic control problem of driving an underpowered car up a mountain. The original problem is set up as cost-to-goal, and here to frame it as a challenging exploration problem we offset the reward by 1, making it a sparse reward problem. The start state is sampled from the range \([-0.6, -0.4]\), which is the valley between two mountains, and the car starts with velocity zero.

\textbf{Puddle World} is a continuous state 2-dimensional world with \((x, y) \in [0, 1]^2\) with 2 intersecting puddles: (1) \([0.45, 0.4]\) to \([0.45, 0.8]\), and (2) \([0.1, 0.75]\)  to \([0.45, 0.75]\). The puddles have a radius of 0.1 and the goal is the region \((x, y) \in [0.95, 1.0], [0.95, 1.0]\). The problem is cost-to-goal with additional penalty for when the agent is either puddle. The penalty for being in a puddle is inversely proportional to the distance of the agent from the center of the puddle, i.e., higher negative reward for being closer to the center. The agent chooses a direction of movement, resulting in displacement equal to \(0.005 + \zeta, \zeta \sim N(\mu=0, \sigma=0.1)\) in the chosen direction. The starting positions for episodes is uniformly sampled from \((x, y) \in [0.1, 0.3], [0.45, 0.65]\). 
High variance transitions coupled with high magnitude penalties make this a challenging exploration problem.
%This domain highlights a common difficulty for traditional exploration methods: high magnitude negative rewards, which often cause the agent to erroneously decrease its value estimates too quickly. 

\textbf{River Swim} is a standard continuing exploration benchmark inspired by a fish trying to swim upriver, with high reward (+1) upstream which is difficult to reach and, a lower but still positive reward (+0.005), which is easily reachable downstream. The state space is continuous in \([0, 1]\), and the stochastic displacement is equal to \(0.1 + \zeta, \zeta \sim N(\mu=0, \sigma=0.01)\) in the direction of the chosen action up or down. As swimming upstream is difficult, action up is stochastically switched to down. We also flip the observation such that the high reward is at observation $0$ and the lower reward is at observation $1$. We do this because we noticed that using random initialization with ReLU activations would mostly result in a higher value for a higher input thus favouring the correct action in this case. The starting position is sampled uniformly in \([0.9, 1.0]\).

\textbf{Deepsea} is a finite-horizon episodic grid world environment, which poses a hard exploration challenge. In each state the agent can take two actions, left or right. Every action moves the agent down one row with column change being controlled by the chosen action. Collisions to the grid edges are handled by the agent staying in the same column but moving down one row. Given the transition structure, the agent can never access the states in the top-right triangle of the grid. Therefore, the total number of states are \(\frac{N \times (N+1)}{2}\). The most rewarding state is the state on the bottom-right corner of the grid. To reach this state successfully in an episode, the agent needs to take the action that moves it towards right at every step. However, there is a penalty of \(\frac{0.01}{N} \) for taking this action in every state, except for in the high rewarding state where the agent gets a reward of 1 for taking the right action. This transition and reward structure make it a very challenging environment. A policy that explores uniform randomly has a probability of \(2^{-N}\) of reaching the highly rewarding state in any episode.

\subsection{Algorithm Details for Classics Control}
\label{app_details_classic_control}
In the classic control experiments (Section \ref{sec:classic_control}), every agent uses the same neural architecture, containing 2 hidden non-linear layers with 50 nodes each and ReLU activation, followed by a linear output-layer. DQN-P, BDQN, VBE, ACB, and RND all use target networks which are updated periodically after every $\tau$ steps. For DQN-P and BDQN $\tau=4$ works best for all four classic environments, as used by the BDQN paper. VBE, ACB and RND use $\tau=4$ for Mountain Car, Puddle World and River Swim, and $\tau=64$ for Deepsea. We use a learning rate of $\alpha=0.001$ and a discount factor of $\gamma=0.99$. DQN-P, BDQN and VBE variants use an experience replay buffer that stores the most recent 50K transitions. The agent's parameters are updated after every step using a randomly sampled mini-batch of $128$. We sweep the agents on bonus scales $c=[1.0, 3.0, 10.0]$, and ensemble sizes $k=[1, 2, 8, 20]$. The PPO version of ACB uses an ensemble size of $k=128$, and RND uses a multi-layer neural network instead of an ensemble. Tables~\ref{tab:pe_best_params}, and ~\ref{tab:nn_best_params} show the best performing sets of ensemble size $k$ and bonus scale $c$ for results in Sections~\ref{sec:pure-exploration}, and ~\ref{sec:classic_control}, correspondingly.
\begin{table}[ht]
    \centering
    \begin{tabular}{ |c|c| }
        \hline
         & Deepsea\\
        \hline
        VBE & $k=1, c=1.0$ \\
        VBE-SL & $k=20, c=1.0$ \\
        DQN-P & $k=1, c=1.0$ \\
        BDQN & $k=20, c=1.0$ \\
        ACB & $k=20, c=1.0$ \\
        RND & $c=1.0$ \\
        \hline
    \end{tabular}
    \caption{Ensemble size $k$ and bonus scale $c$ for agents in Figure~\ref{fig:ds_pe_l_best}.}
    \label{tab:pe_best_params}
\end{table}
\begin{table}[ht]
    \centering
    \begin{tabular}{ |c|c|c|c|c| }
        \hline
         & River Swim & Puddle World & Mountain Car & Deepsea\\
        \hline
        VBE   & $k=20, c=1.0$ & $k=1, c=10.0$ & $k=2, c=1.0$ &$k=20, c=1.0$\\
        DQN-P & $k=1, c=10.0$ & $k=1, c=3.0$ & $k=1, c=1.0$ &$k=1, c=10.0$\\
        BDQN  & $k=8, c=10.0$ & $k=2, c=1.0$ & $k=20, c=1.0$ &$k=20, c=10.0$\\
        ACB & $k=20, c=10.0$ & $k=1, c=1.0$ & $k=8, c=1.0$ &$k=20, c=1.0$\\
        RND & $c=10.0$ & $c=1.0$ & $c=1.0$ &$c=1.0$\\
        \hline
    \end{tabular}
    \caption{Ensemble size $k$ and bonus scale $c$ for agents in Figure~\ref{fig:nns_best}.}
    \label{tab:nn_best_params}
\end{table}

\subsection{Algorithm Details for Atari}
\label{app_details_atari}
All the agents used in Atari experiments use the same architecture: A representation network with 3 CNN layers followed by a value function head containing one hidden layer with Relu activation and a final linear layer. We use a DDQN update for all agents and update every 4 steps. We update the target networks every 10000 steps. We use a replay buffer that stores 1 million of the most recent transitions and discount factor of 0.99. Finally, we use Adam optimizer with a learning rate of 0.0000625.

\section{Comparing VBE with PPO based variants of ACB and RND}\label{app_ppo}
In this section we compare VBE with PPO-based variants of ACB and RND as they were originally implemented and evaluated -- we call these algorithm ACB-PPO and RND-PPO, respectively. We use the opensource implementation of PPO-based ACB and RND as provided by \citet{ash2022anticoncentrated}\footnote{See https://github.com/JordanAsh/acb/tree/main}. To ensure that each algorithm uses the same amount of data, we use a single agent interacting with the environment for ACB-PPO and RND-PPO, instead of parallel agents interacting with parallel instances of the environment. Using a softmax policy and parallel agents interacting with the environments can also help with exploration, so to be fair we use only one instance of agent-environment interaction, as is done for VBE. 

\subsection{Pure Exploration}
In this section we first compare VBE with PPO-based variants of ACB and RND in the pure exploration setting. Similar to in Section \ref{sec:pure-exploration}, we use one-hot encoding and linear function approximation for all the agents. The agents observe no environment reward and are behaving only with respect to the exploration strategy employed by each of these algorithms. In Figure \ref{fig:pe_ppo_linear_best} we see that the PPO-based variants of ACB and RND cover much more of the state-space compared to their Value-based (VB) counterparts in Figure \ref{fig:ds_pe_l_best}. This may be due to the fact that PPO-based variants of ACB and RND use a softmax behaviour policy, which can cause the agent to take the non-greedy actions at random. In Figure \ref{fig:ds_pe_ppo_l_50} we see that ACB-PPO almost flatlines and stops visiting new states after sometime. However, RND-PPO keeps on visiting new states. Although this is a significant improvement, however, both ACB-PPO and RND-PPO still fail to cover the state-space.

\begin{figure}[ht]
\vspace{-0.3cm}
\captionsetup[subfigure]%
     {justification=centering}
     \centering
    \begin{subfigure}{0.9\textwidth}
        \includegraphics[width=1\hsize]{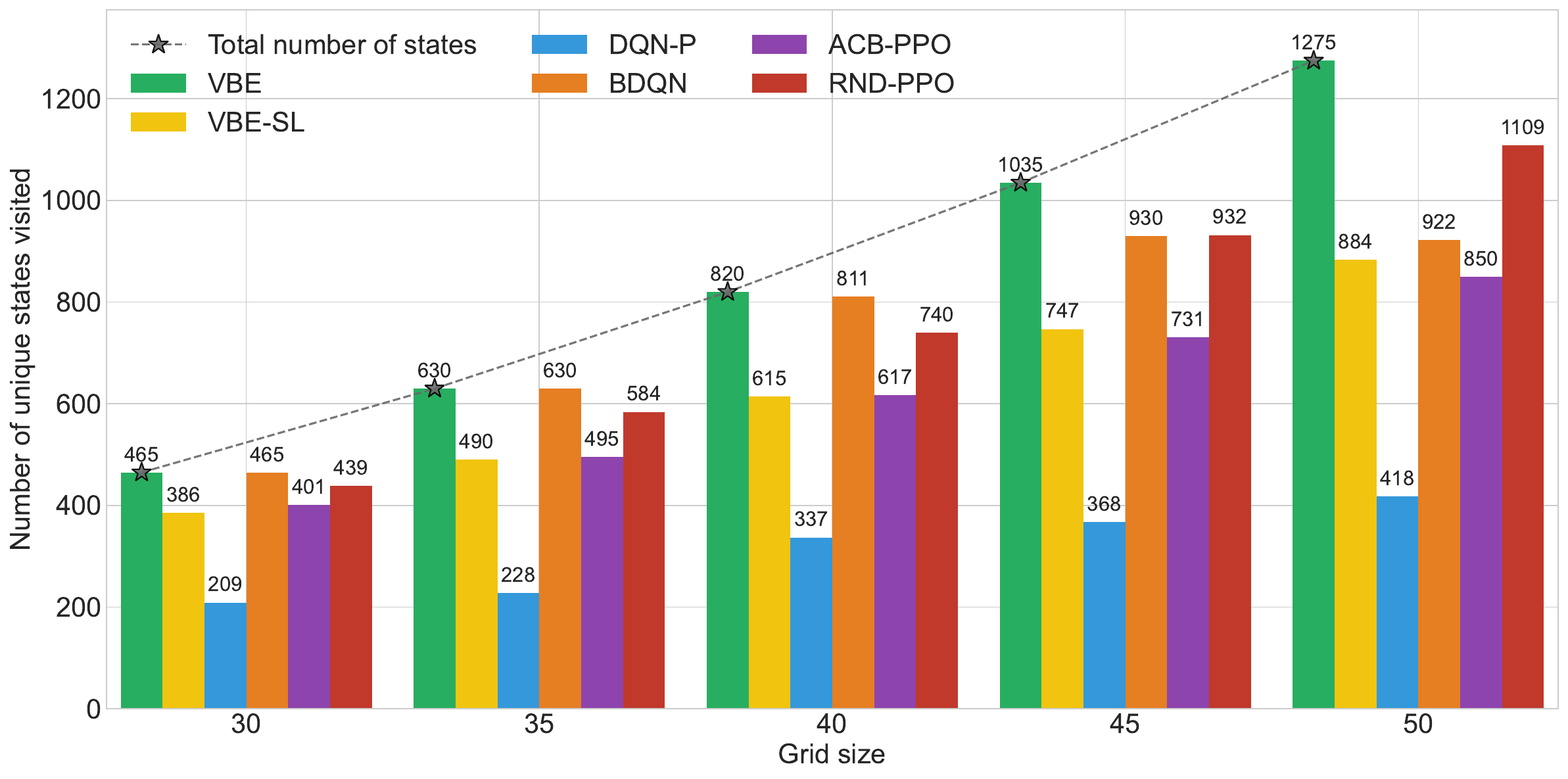}
        \subcaption{State coverage in Deepsea of different grid sizes}
        \label{fig:ds_pe_ppo_l_best}
    \end{subfigure}
    \hfill%
    \begin{minipage}{0.5\textwidth}
        \includegraphics[width=1\linewidth]{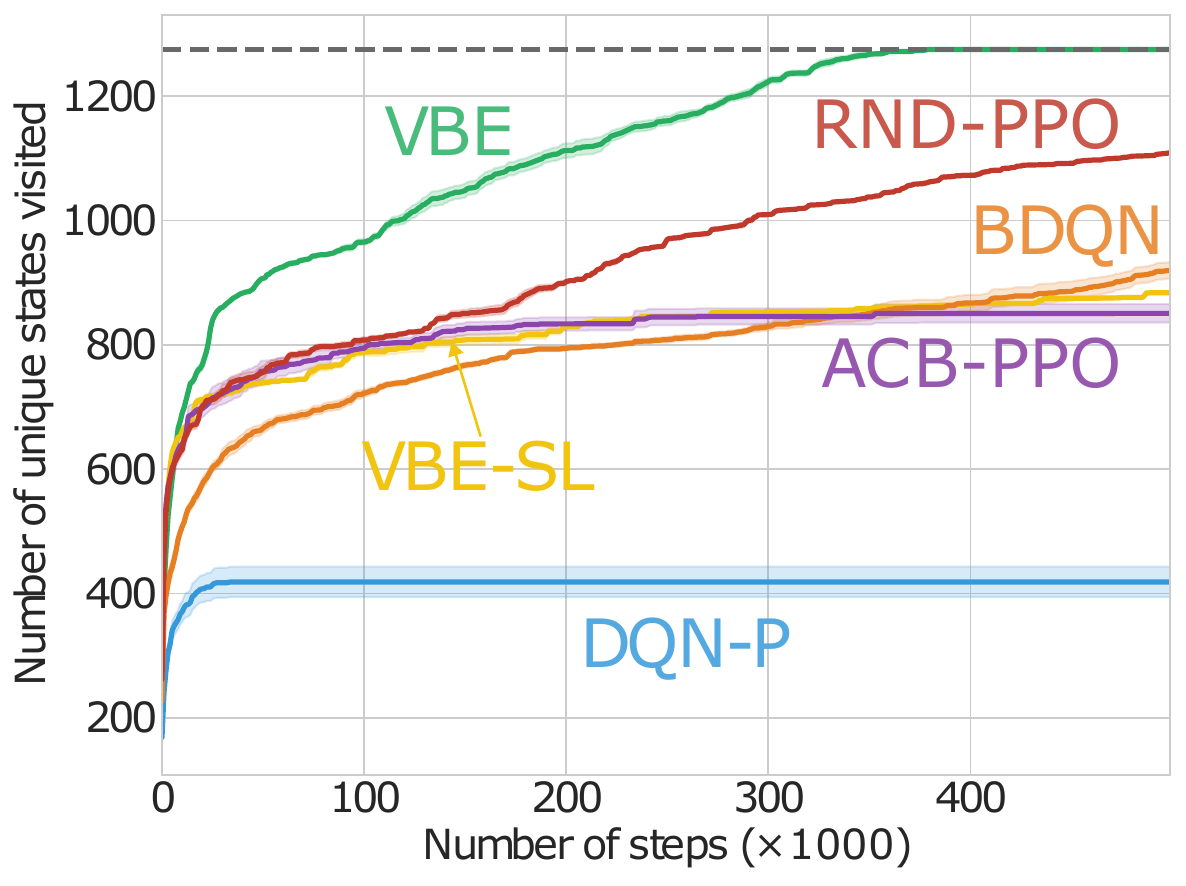}
        \subcaption{Progression of unique states visited (grid size 50)}
        \label{fig:ds_pe_ppo_l_50}
    \end{minipage}
    \hfill%
    \begin{minipage}{0.49\textwidth}
        \caption{Contrasting the state coverage abilities of exploration algorithms in DeepSea. In (a) each bar corresponds to the total number of unique states visited by an agent after completing 10,000 episodes. The black stars indicate the total number of unique states for each grid size. Notably, VBE covers the entire state space, even for the larger grid sizes. (b) displays the progression of unique states visited by agents over the course of learning for Deepsea with grid size 50. The dotted line represents the total number of unique states (1275) in this environment. It provides evidence that VBE consistently explores new states at a significantly higher rate.}
        \label{fig:pe_ppo_linear_best}
    \end{minipage}
\end{figure}

\subsection{Classic Control}
\label{app_ppo_control}
The chosen hyperparmaters of ACB-PPO and RND-PPO as compared to VBE for classic control domains are shown in Table~\ref{tab:rb_nn_best_params}. Figure~\ref{fig:rb_nns_best} shows the result comparing these algorithms. In general, across the domains the PPO-variants' perform similar to, or poorer than, their VB variants in Figure~\ref{fig:nns_best}. VBE, and BDQN, outperform them across domains. Even DQN-P is competitive with them, or better in all domains except Deepsea. Compared to their VB variants, the performance of the PPO variants drops in River Swim, and Puddle World, and continues to be the same in Mountain Car. In Deepsea RND-PPO improves upon its VB variant RND, while ACB-PPO continues to perform similarly to ACB.

\begin{figure*}[ht]
    \vspace{-0.1cm}
    \centering
    \begin{subfigure}{0.24\textwidth}
        \includegraphics[width=\hsize]{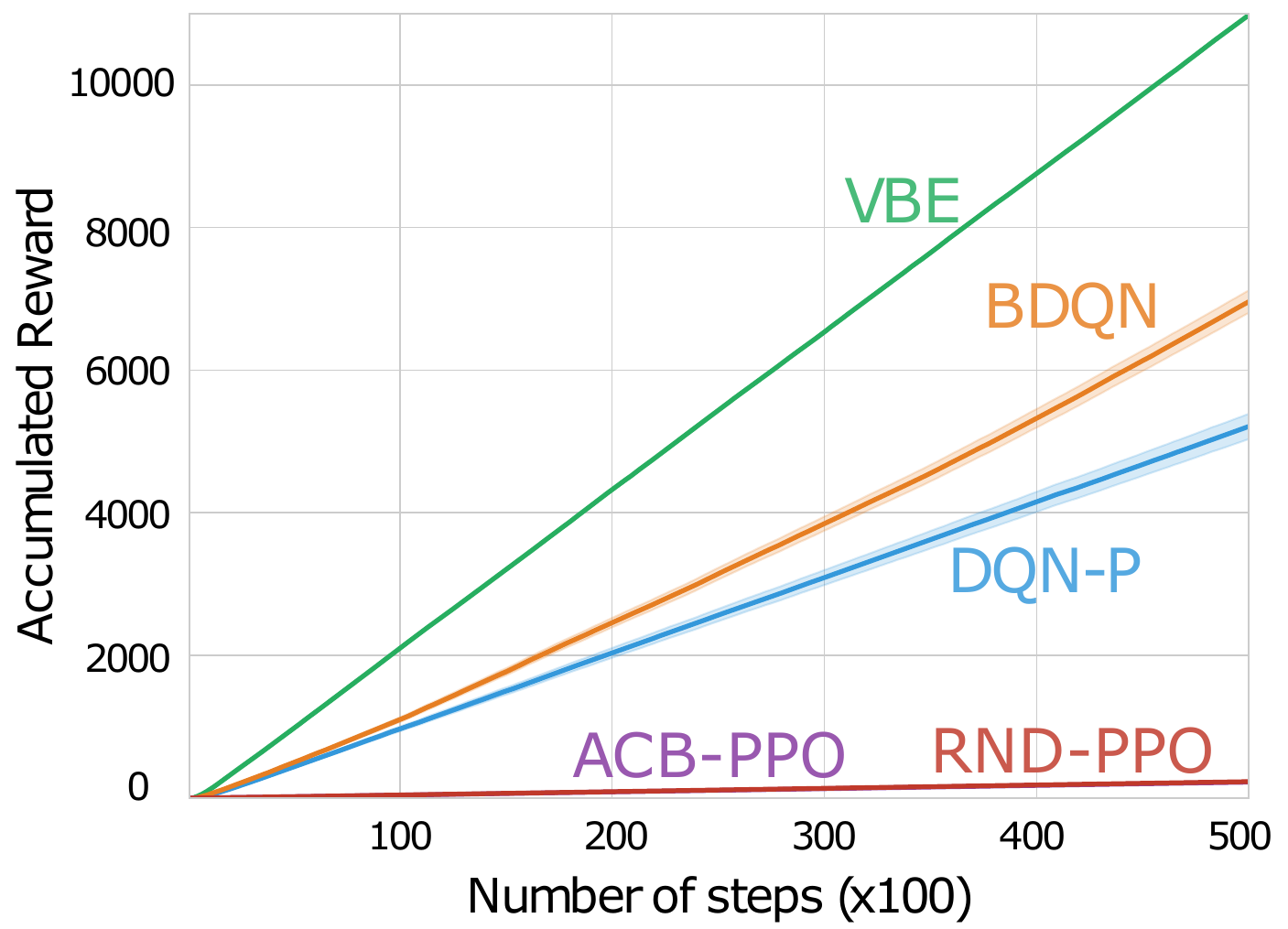}
        \caption{River Swim}
        \label{fig:rs_rb_nn_best}
    \end{subfigure}
    \hfill%\hfil
    \centering
    \begin{subfigure}{0.24\textwidth}
        \includegraphics[width=\hsize]{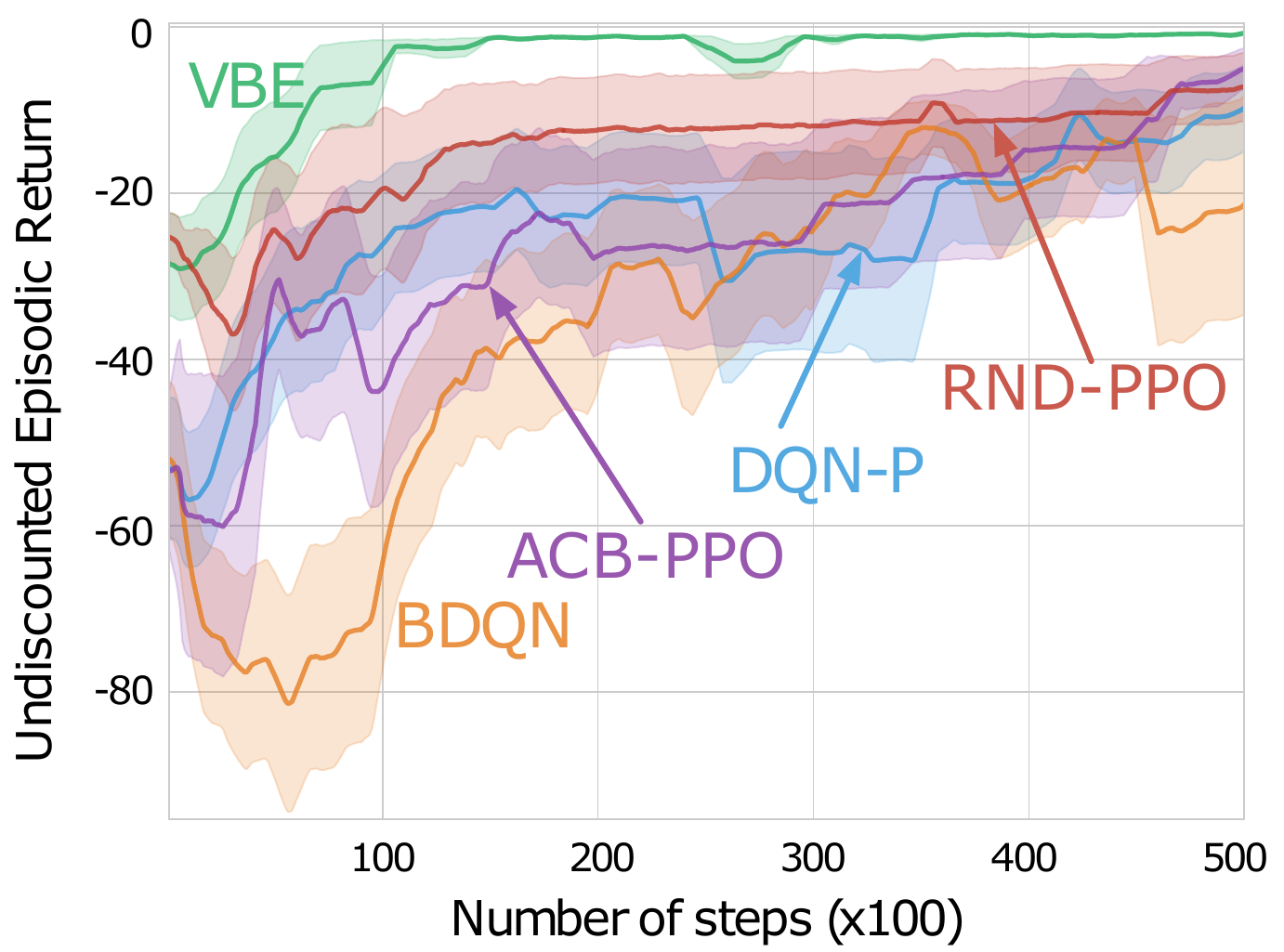}
        \caption{Puddle World}
        \label{fig:pw_rb_nn_best}
    \end{subfigure}
    \hfill%\hfil
    \centering
    \begin{subfigure}{0.24\textwidth}
        \includegraphics[width=\hsize]{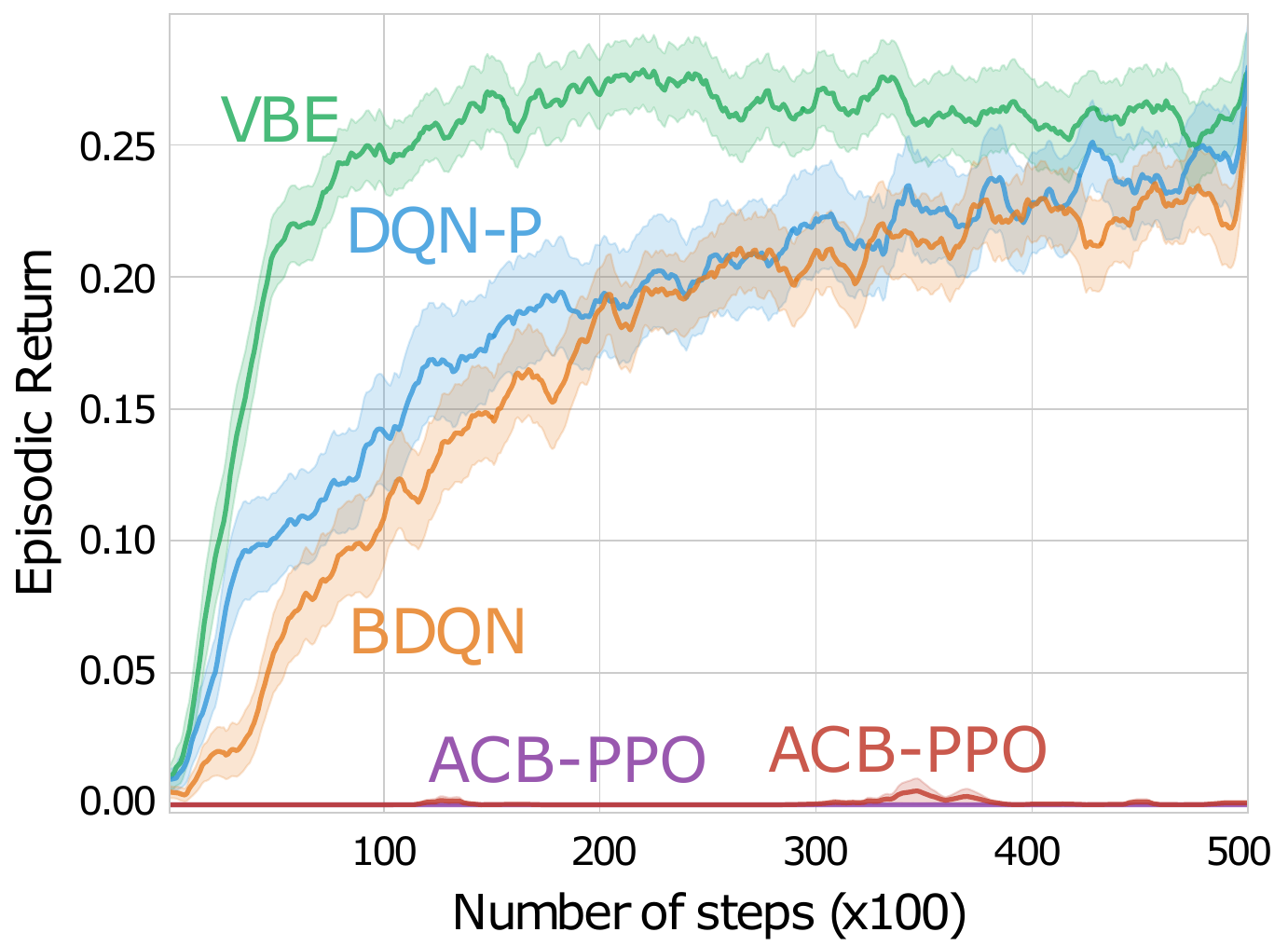}
        \caption{Mountain Car}
        \label{fig:mc_rb_nn_best}
    \end{subfigure}
    \hfill%\hfil
    \centering
    \begin{subfigure}{0.24\textwidth}
        \includegraphics[width=\hsize]{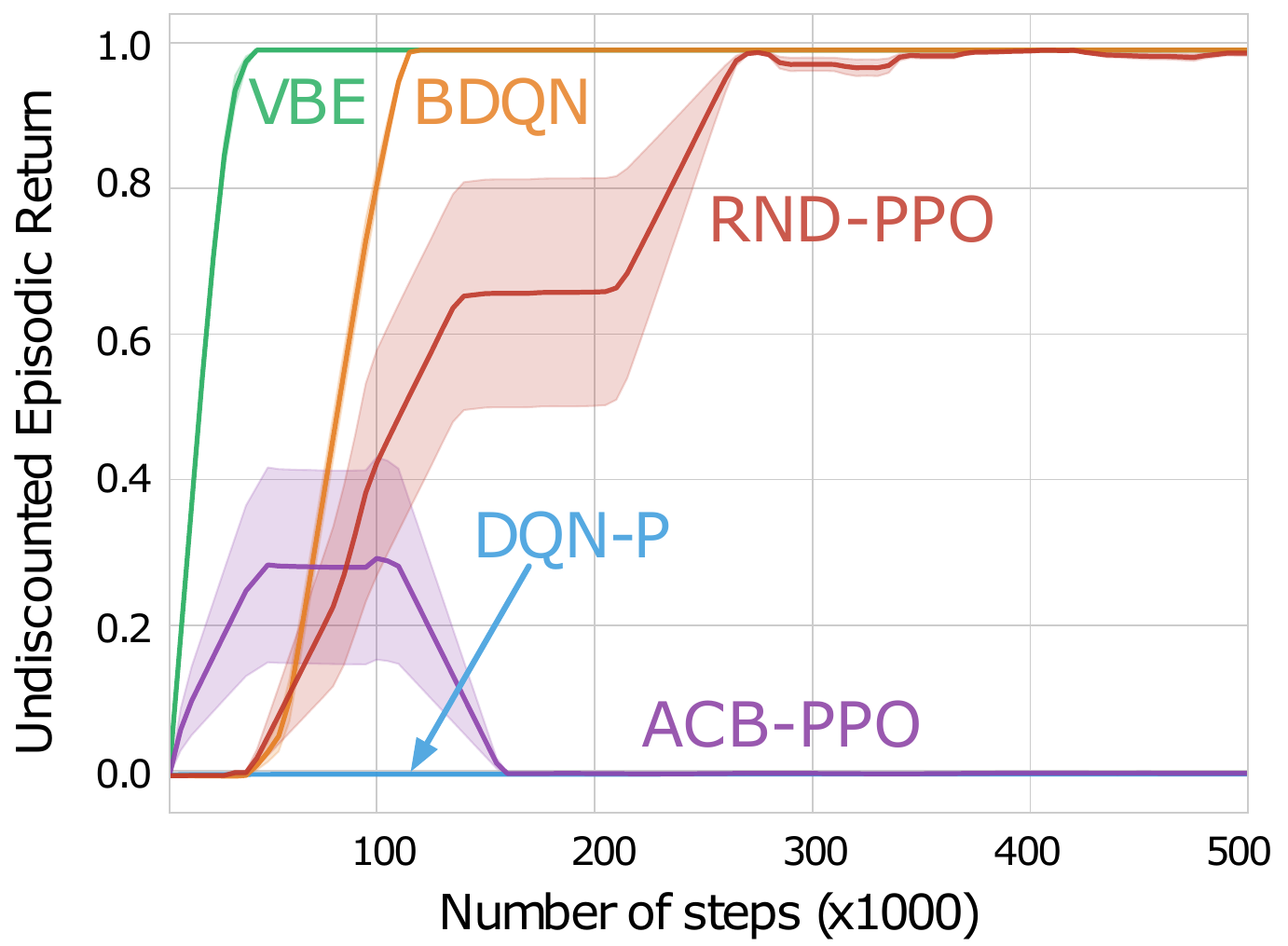}
        \caption{Deepsea}
        \label{fig:ds_rb_nn_best}
    \end{subfigure}
        \vspace{-0.1cm}
    \caption{Online performance of PPO-based variants of ACB and RND in River Swim, Puddle World, Mountain Car, and Deepsea. Higher on the y-axis is better. The x-axis denotes the number of interaction steps with the environment. The shaded region corresponds to standard errors.
    %The statistic visualized reflects online performance, where higher is better on the y-axis, and the x-axis denotes the number of interactions with the environment.
    }
    \label{fig:rb_nns_best}
\end{figure*}

\begin{table}[ht]
    \centering
    \begin{tabular}{ |c|c|c|c|c| }
        \hline
         & River Swim & Puddle World & Mountain Car & Deepsea\\
        \hline
        VBE   & $k=20, c=1.0$ & $k=1, c=10.0$ & $k=2, c=1.0$ &$k=20, c=1.0$\\
        ACB ($k=128$)   & $c=1.0$ & $c=3.0$ & $c=1.0$ &$c=10.0$\\
        RND   & $c=1.0$ & $c=10.0$ & $c=10.0$ &$c=1.0$\\
        \hline
    \end{tabular}
    \caption{Ensemble size $k$ and bonus scale $c$ for agents in Figure~\ref{fig:rb_nns_best}.}
    \label{tab:rb_nn_best_params}
\end{table}

\section{An alternate variant of VBE for Atari}
\label{sec:apdx_atari_vbe_vbe_cnn}
\begin{figure*}[!ht]
    \centering
    \begin{subfigure}{0.3\textwidth}
        \includegraphics[width=\hsize]
        {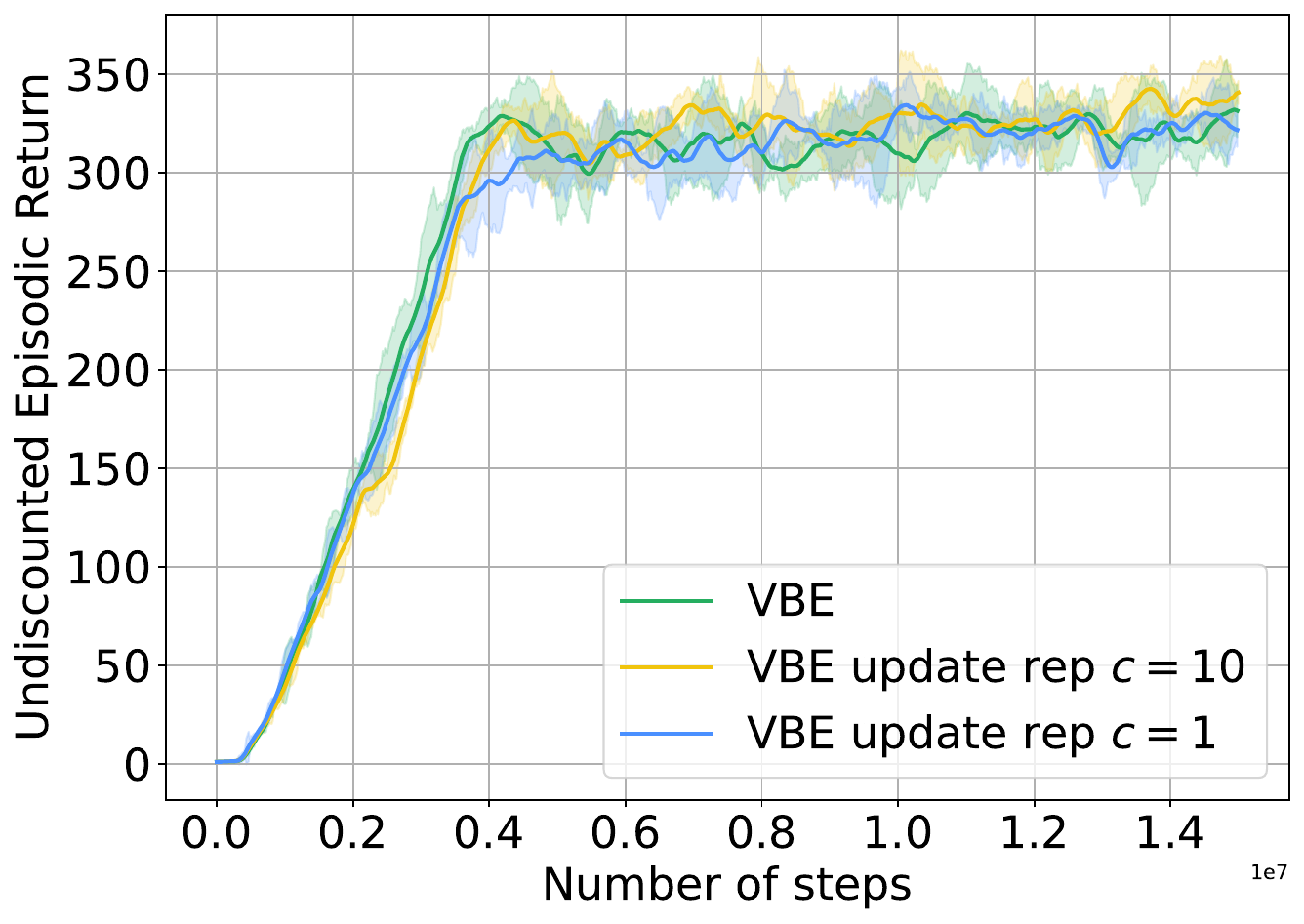}
        \caption{Breakout}
        \label{fig:vbe_cnn_breakout}
    \end{subfigure}
    \begin{subfigure}{0.3\textwidth}
        \includegraphics[width=\hsize]{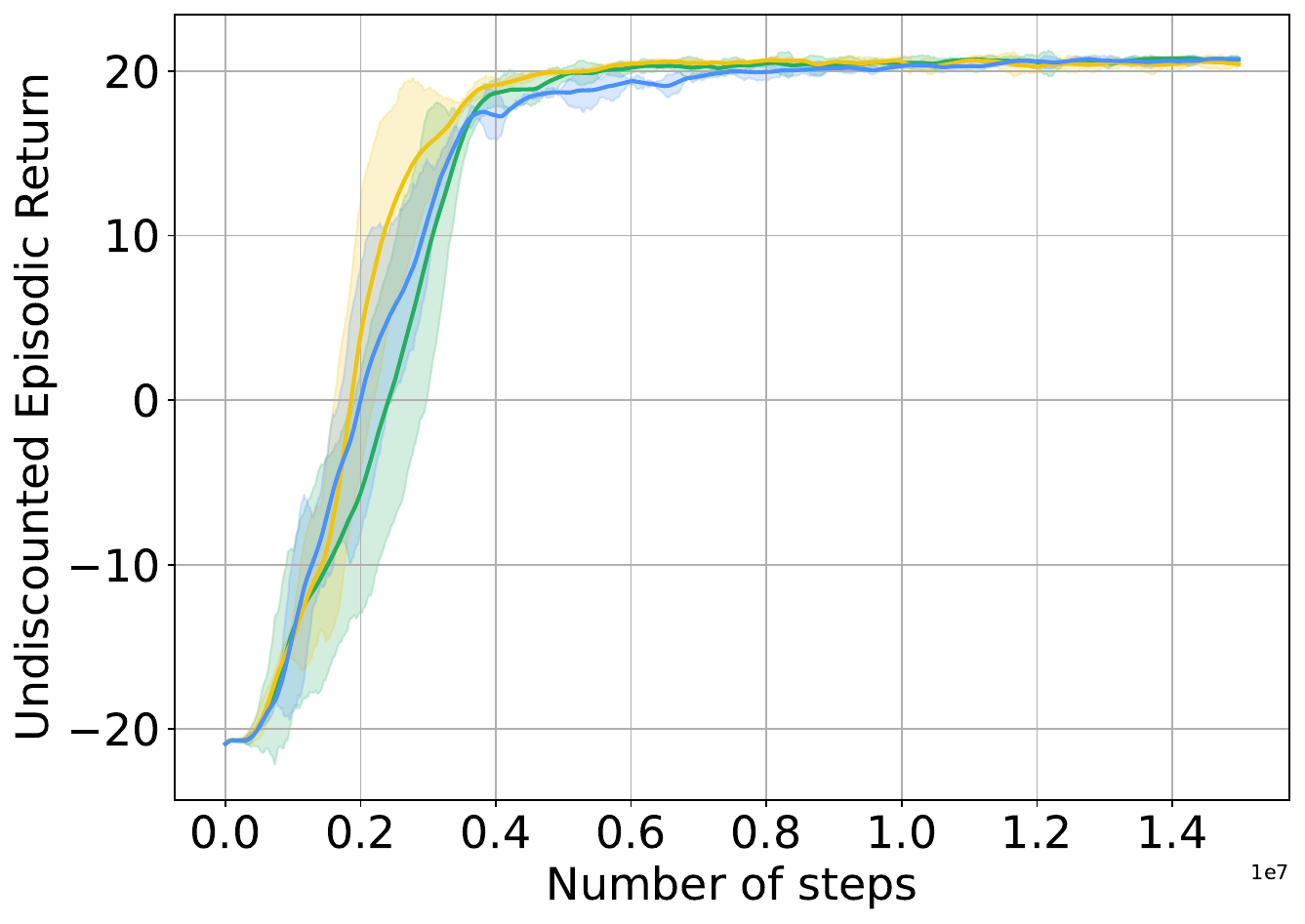}
        \caption{Pong}
        \label{fig:vbe_cnn_pong}
    \end{subfigure}
    \begin{subfigure}{0.3\textwidth}
        \includegraphics[width=\hsize]{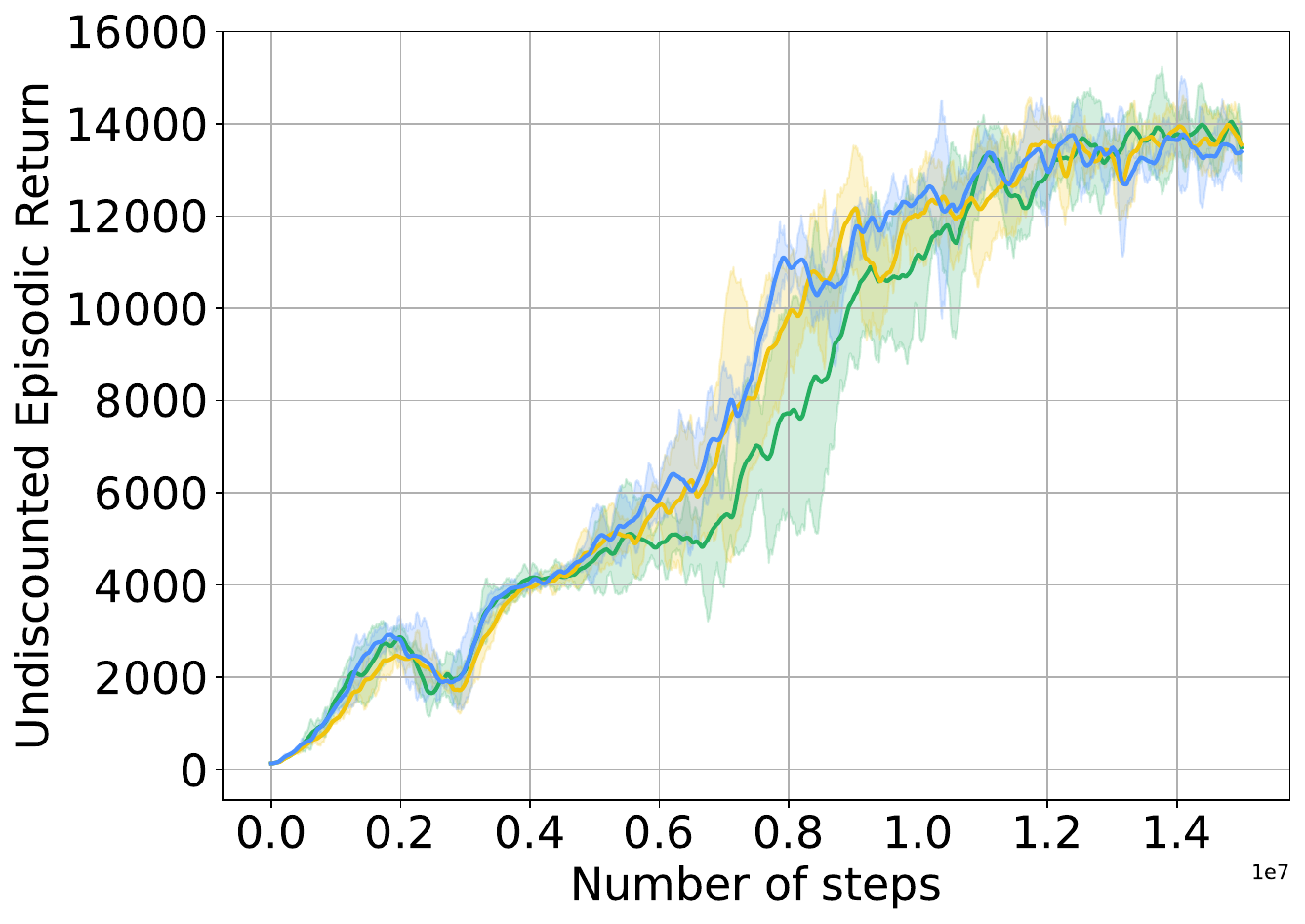}
        \caption{Q$^*$bert}
        \label{fig:vbe_cnn_Qbert}
    \end{subfigure}
    \begin{subfigure}{0.3\textwidth}
        \includegraphics[width=\hsize]{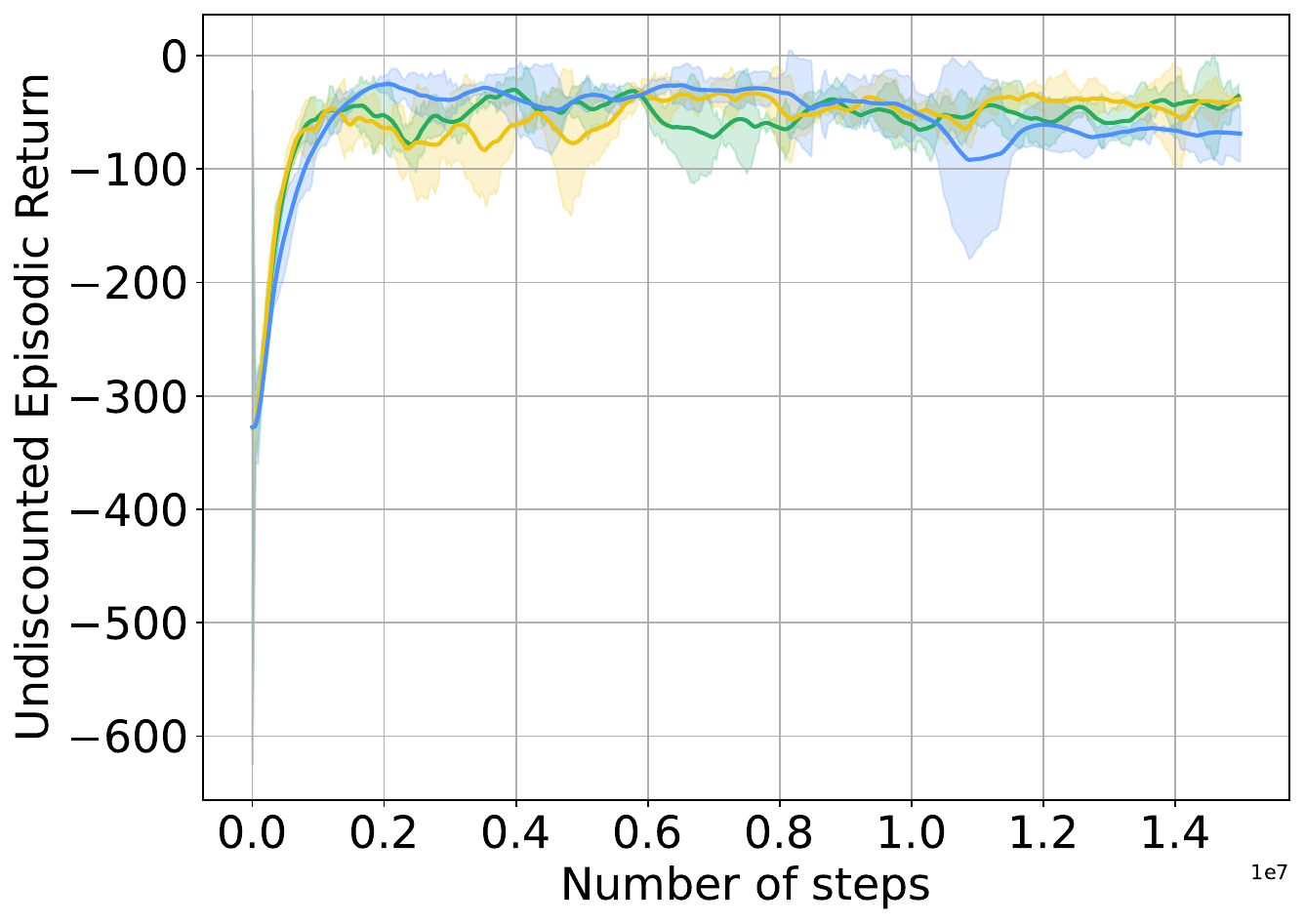}
        \caption{Pitfall}
        \label{fig:vbe_cnn_pitfall}
    \end{subfigure}
    \begin{subfigure}{0.3\textwidth}
        \includegraphics[width=\hsize]{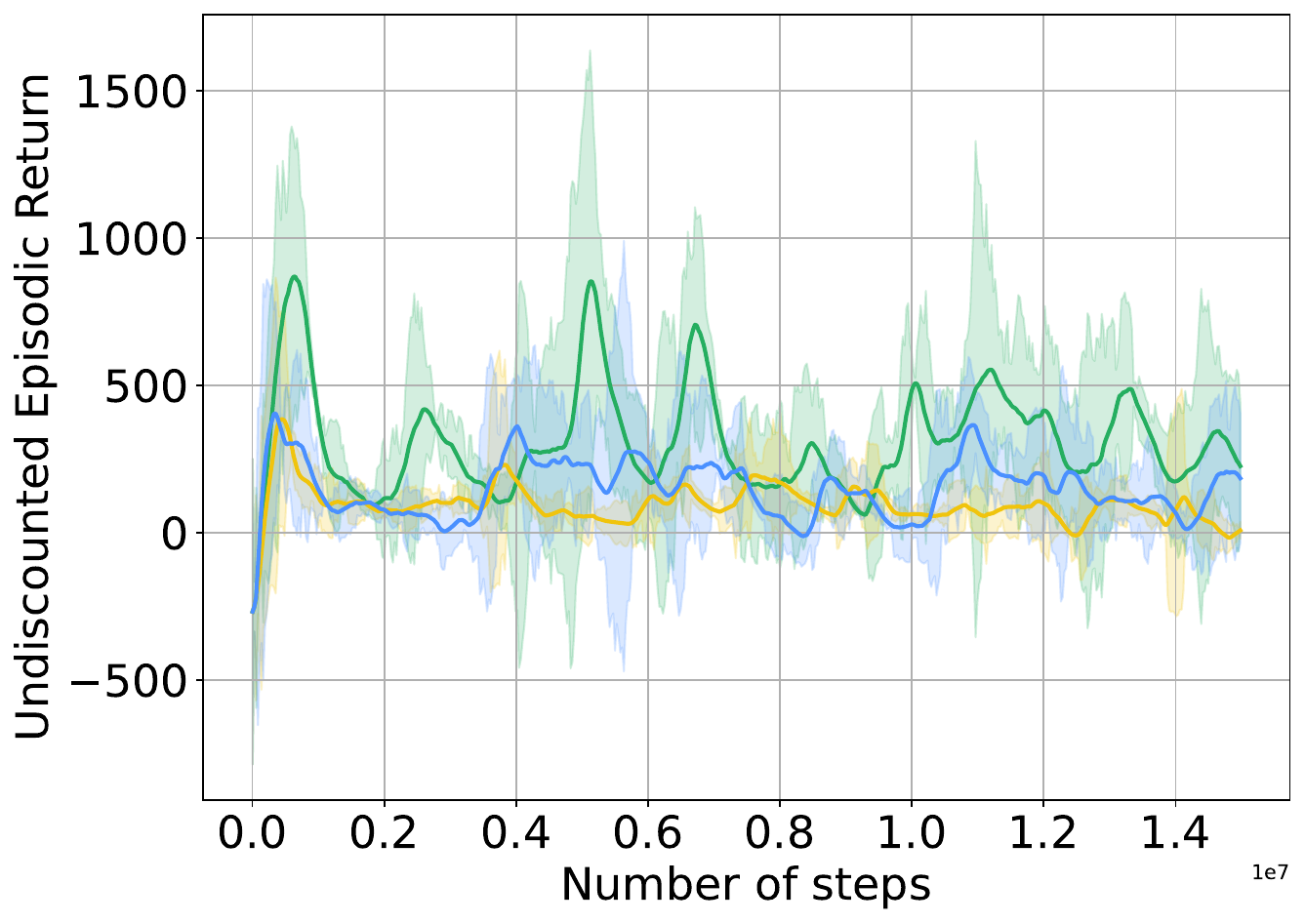}
        \caption{Private-Eye}
        \label{fig:vbe_cnn_privateeye}
    \end{subfigure}
    \begin{subfigure}{0.3\textwidth}
        \includegraphics[width=\hsize]{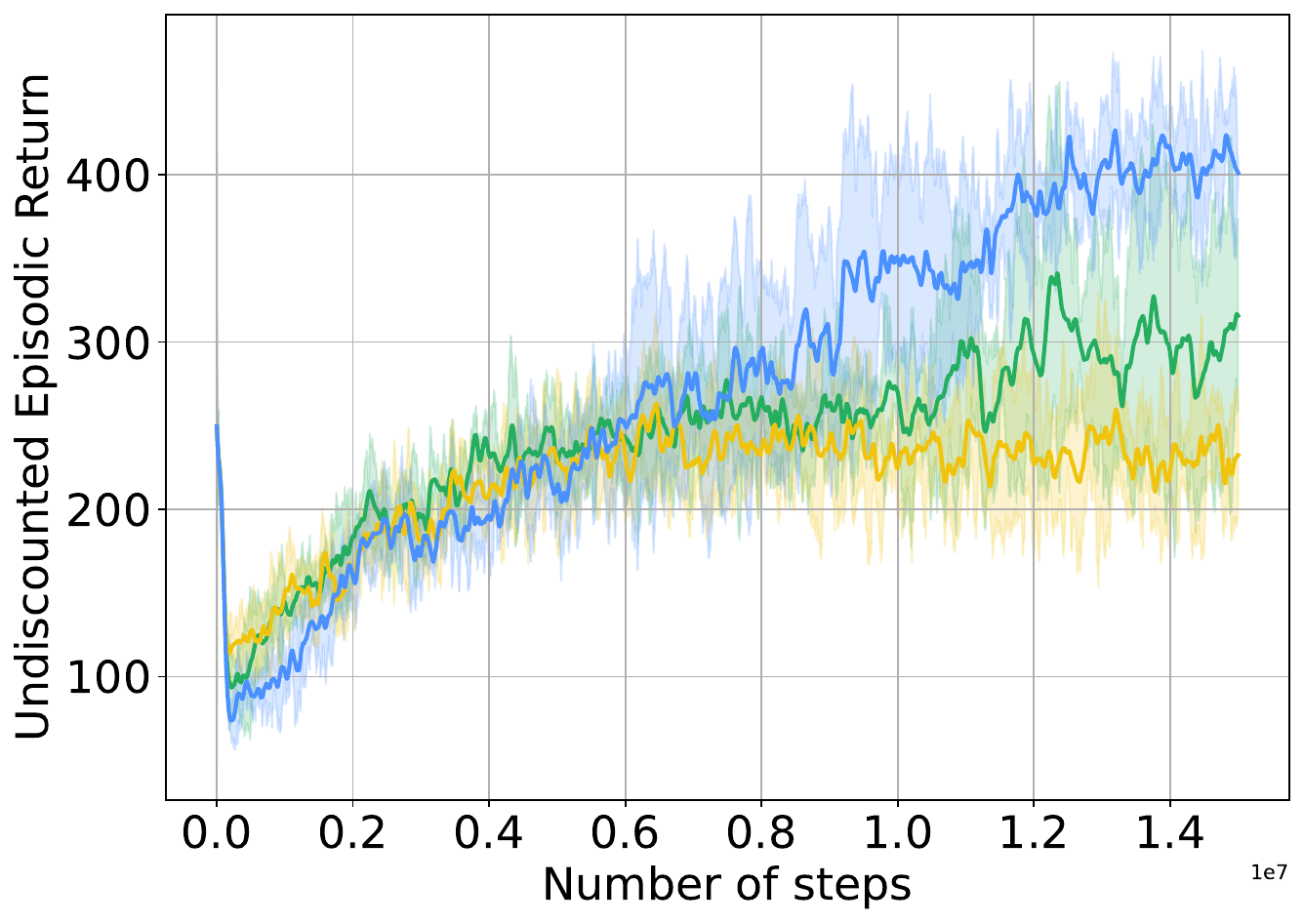}
        \caption{Gravitar}
        \label{fig:vbe_cnn_gravitar}
    \end{subfigure}    
    \caption{Comparing the online performance of VBE and two of its variants on six Atari environments. We do 3 runs for each agent for 15 Million steps.}
    \label{fig:vbe_cnn_best}
\end{figure*}

The VBE in Section \ref{sec:atari} uses a shared representation for RQFs and only updates the RQF head to reduce computational complexity. In this section, we test another variant of VBE that has a separate representation network for each RQF and updates the entire network. We also test this new variant with a smaller bonus-scale $\bscale$ = 1. The purpose of this experiment was to test whether reducing the representational ability of VBE by using a shared representation and not updating it can hinder its learning capabilities. However, we observe that the base variant of VBE can keep up with its more parameterized counterpart. In Breakout and Pong, the variant with learnable representation networks shows slight improvements compared to the base VBE. We also note that using a smaller $\bscale$ can affect the performance in hard exploration environments. In Pitfall, VBE with $\bscale$ = 1 performs well initially but then takes a dip after 10 million steps. Consistent with all other agents, VBE's variants are unable to learn on Private Eye, however, base VBE and the variant with $\bscale$ = 1 collects higher rewards more often. In Gravitar, VBE with $\bscale$ = 1 does exceptionally well compared to the other variants and matches the performance of RND in Figure \ref{fig:vbe_gravitar}. This experiment reveals that our design choices for using a shared representation for RQFs and only updating the RQF heads is valid and allows the agent to explore. It further reveals that more investigation is needed to understand the role of bonus-scale in these environments. With more compute resources, and search over the hyperparameters can potentially identify a variant of VBE that can perform even better on these environments. 
\section{Linear Function Approximation}
\label{sec:apdx_lfa_control}
In this section we test VBE and the baseline agents on the same four classic environments as in Section~\ref{sec:classic_control}, with tile-coded features and linear function approximation. We use the following tile-coding parameters --  River Swim :($tiles=4, tiling=32, features=128$), Puddle World: ($tiles=5, tiling=5, features=128$), and Mountain Car: ($tiles=4, tiling=16, features=512$). The results in Figure~\ref{fig:linear_best} are similar to their neural network counterpart results in Figure~\ref{fig:nns_best}, in that VBE does best or is competitive across all domains. RND and ACB perform similarly in all domains -- being competitive with VBE in Puddle World and Mountain Car, but failing in River Swim and Deepsea. BDQN outperforms VBE in Deepsea marginally. DQN-P is competitive with BDQN in all domains except Deepsea -- outperforming it in in Puddle World and Mountain Car. We also compare VBE, BDQN, and DQN-P to PPO-based variants of ACB and RND in Figure~\ref{fig:rb_l_best}. ACB-PPO's performance drops with respect to ACB in all domains. RND-PPO on the other hand improves in Deepsea to be competitive with BDQN and VBE, has a high rate of improvement later towards later learning in River Swim, while dropping in perfomance in both Puddle World and Mountain Car with respect to its value-based variant.

% In River Swim, RND does well and even surpasses BDQN in terms of performance. ACB, however, still fails on River Swim. In Puddle World RND is comparable towards the end of training, and ACB is much slower. DQN-P outperforms BDQN in Puddle World, and Mountain Car, whereas both ACB and RND fail in Mountain Car. In Deepsea we see that DQN-P and ACB fail to learn the optimal policy. RND learns relatively quickly but then fails to stick to the optimal policy and thus collects less reward per episode throughout. In Figure~\ref{fig:rb_l_best} we compare VBE with VB ACB and RND in the linear setting. In River Swim VB ACB does better than its PPO counter part in Figure~\ref{fig:rs_l_best}. VB ACB and RND perform comparable to VBE in Puddle World and Mountain Car. Both VB ACB and RND fail in Deepsea. 

\begin{figure*}[th]
    \centering
    \begin{subfigure}{0.24\textwidth}
        \includegraphics[width=\hsize]{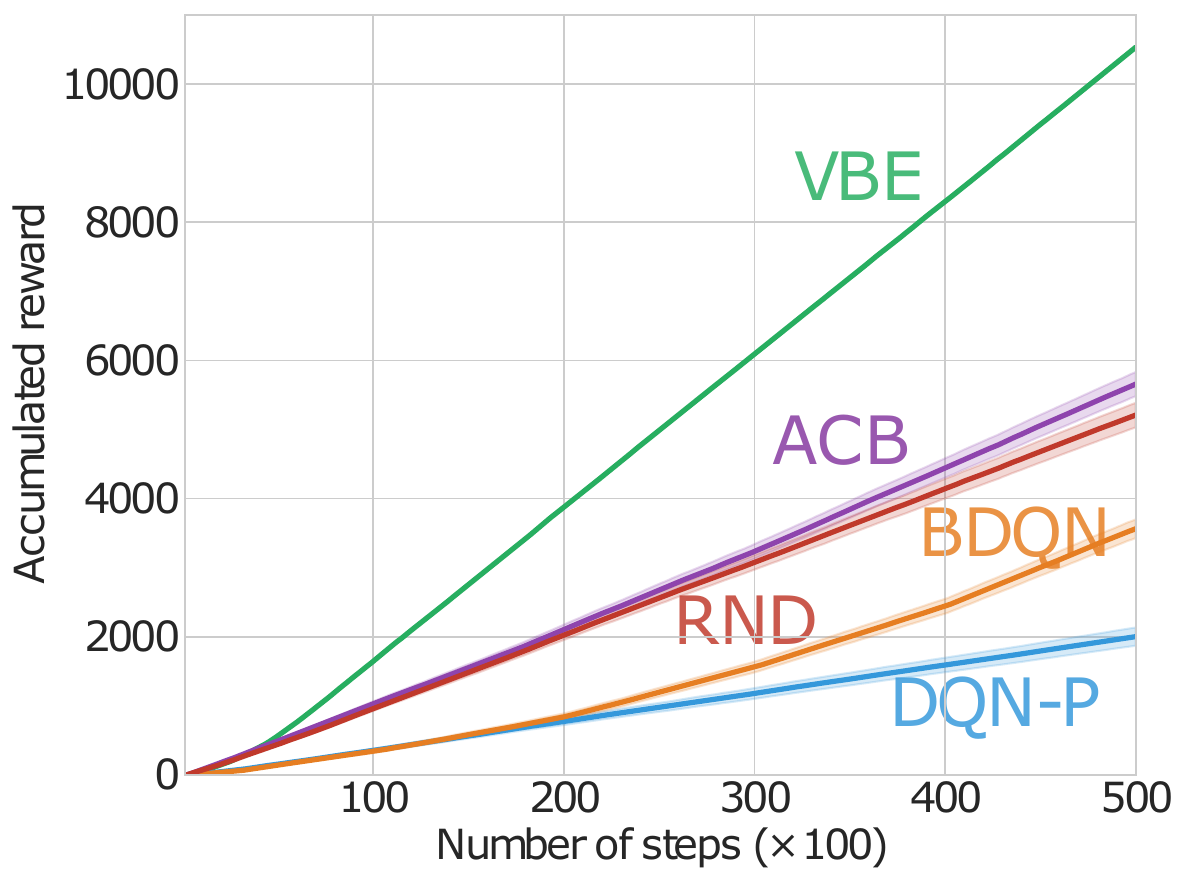}
        \caption{River Swim}
        \label{fig:rs_l_best}
    \end{subfigure}
    \hfill%\hfil
    \centering
    \begin{subfigure}{0.24\textwidth}
        \includegraphics[width=\hsize]{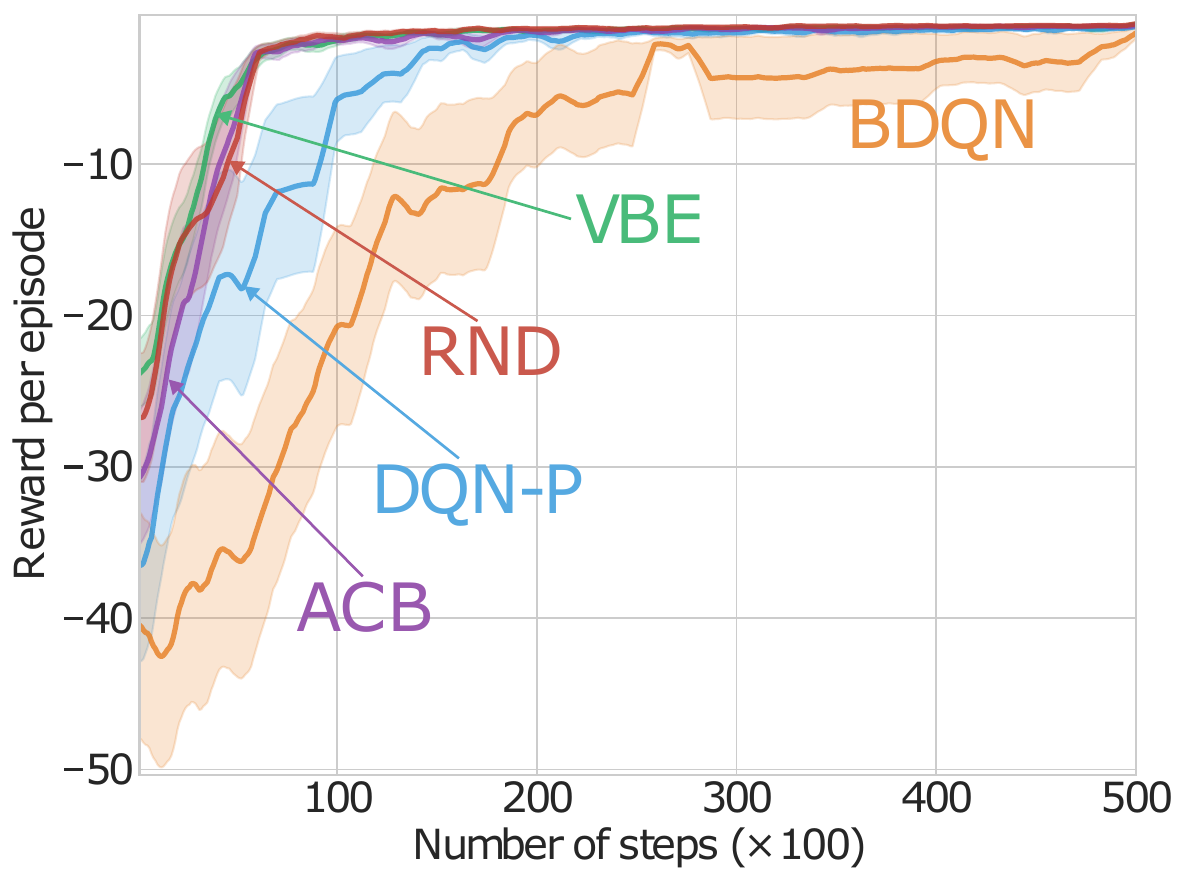}
        \caption{Puddle World}
        \label{fig:pw_l_best}
    \end{subfigure}
    \hfill%\hfil
    \centering
    \begin{subfigure}{0.24\textwidth}
        \includegraphics[width=\hsize]{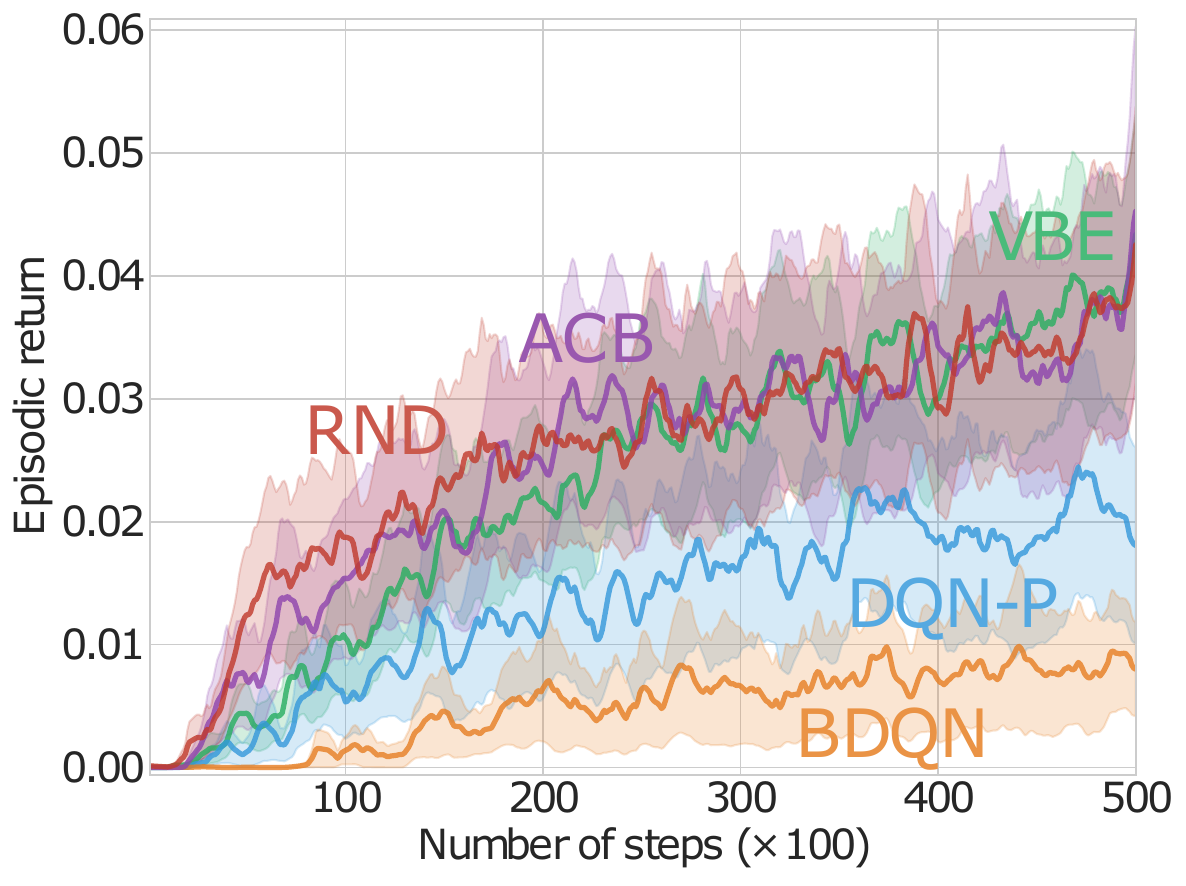}
        \caption{Mountain Car}
        \label{fig:mc_l_best}
    \end{subfigure}
    \hfill%\hfil
    \centering
    \begin{subfigure}{0.24\textwidth}
        \includegraphics[width=\hsize]{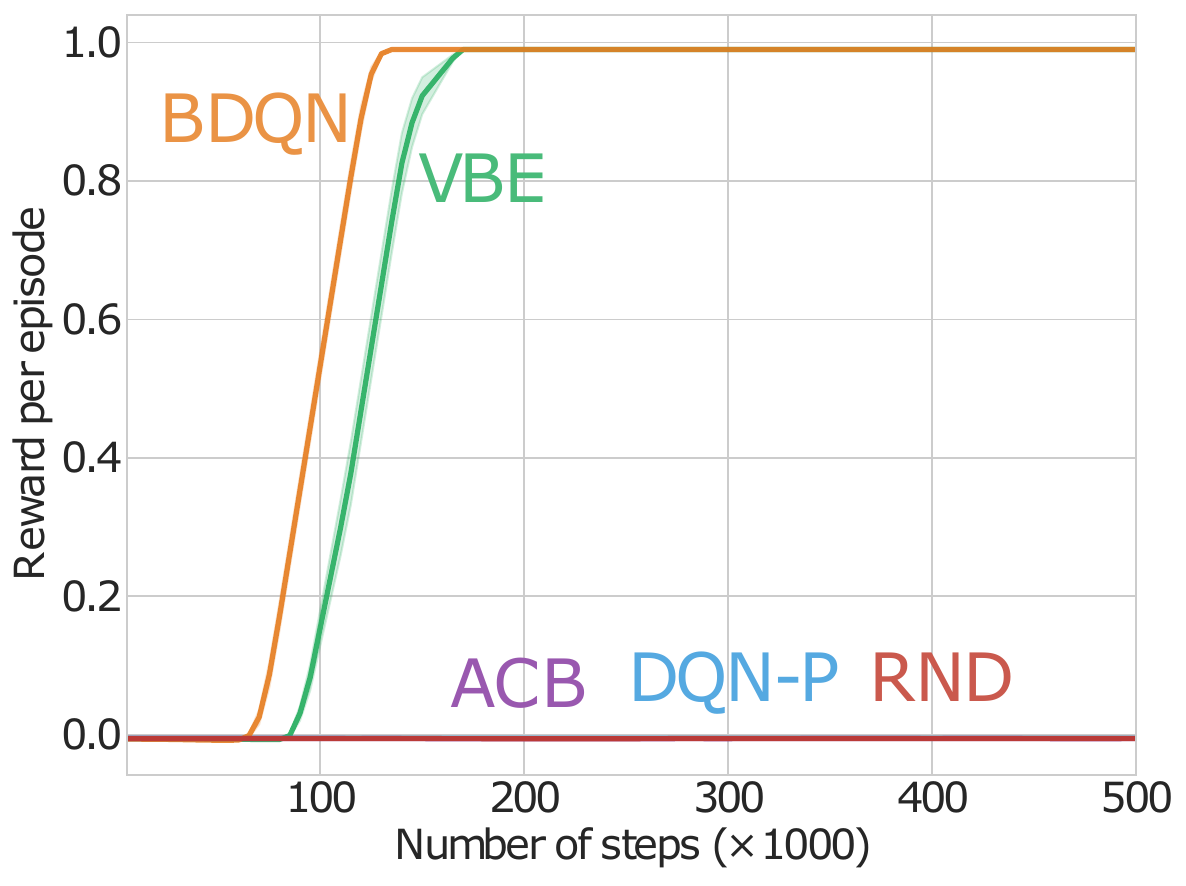}
        \caption{Deepsea}
        \label{fig:ds_l_best}
    \end{subfigure}
    \caption{Online performance in River Swim, Puddle World, Mountain Car, and Deepsea, with tile-coded features and linear function approximation. Higher on the y-axis is better. The x-axis denotes the number of interaction steps with the environment. The shaded region corresponds to standard errors.}
    \label{fig:linear_best}
\end{figure*}
\begin{figure*}[ht]
    \centering
    \begin{subfigure}{0.24\textwidth}
        \includegraphics[width=\hsize]{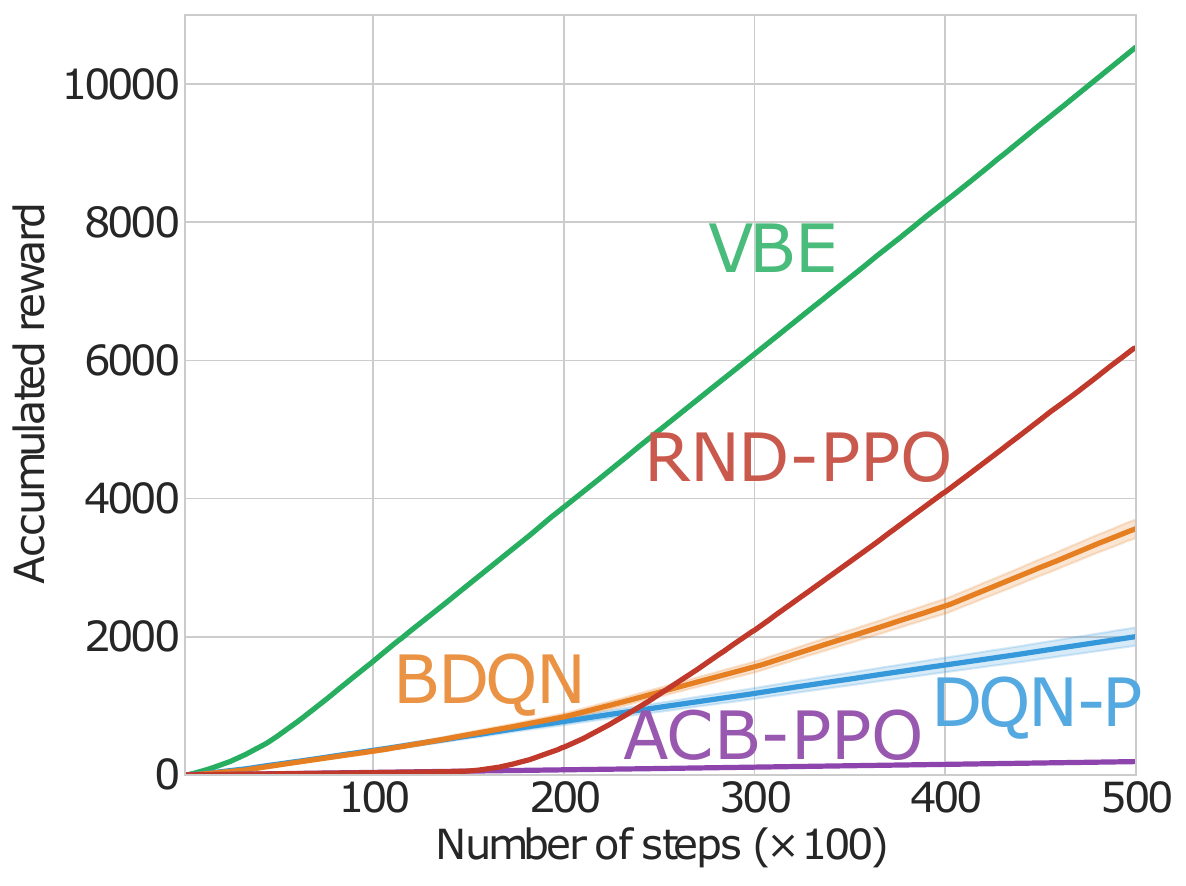}
        \caption{River Swim}
        \label{fig:rs_rb_l_best}
    \end{subfigure}
    \hfill%\hfil
    \centering
    \begin{subfigure}{0.24\textwidth}
        \includegraphics[width=\hsize]{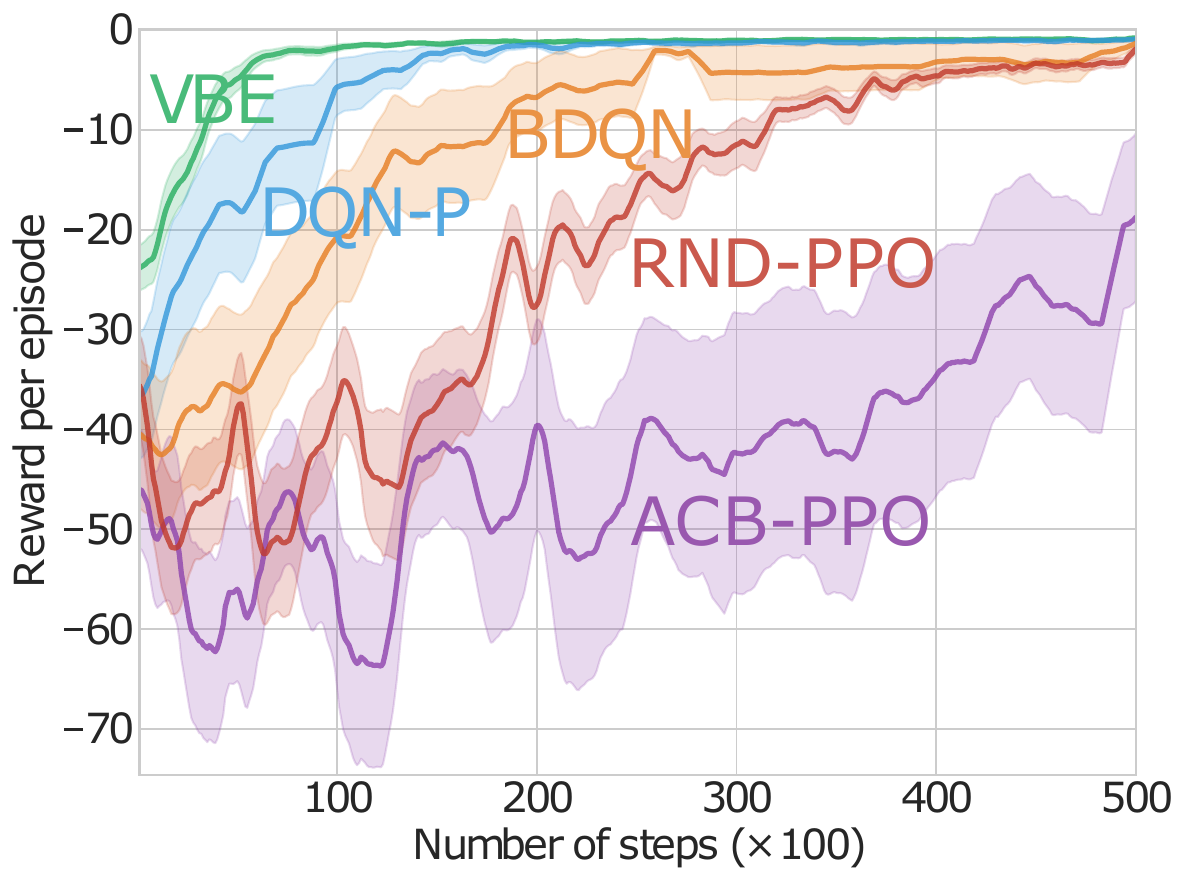}
        \caption{Puddle World}
        \label{fig:pw_rb_l_best}
    \end{subfigure}
    \hfill%\hfil
    \centering
    \begin{subfigure}{0.24\textwidth}
        \includegraphics[width=\hsize]{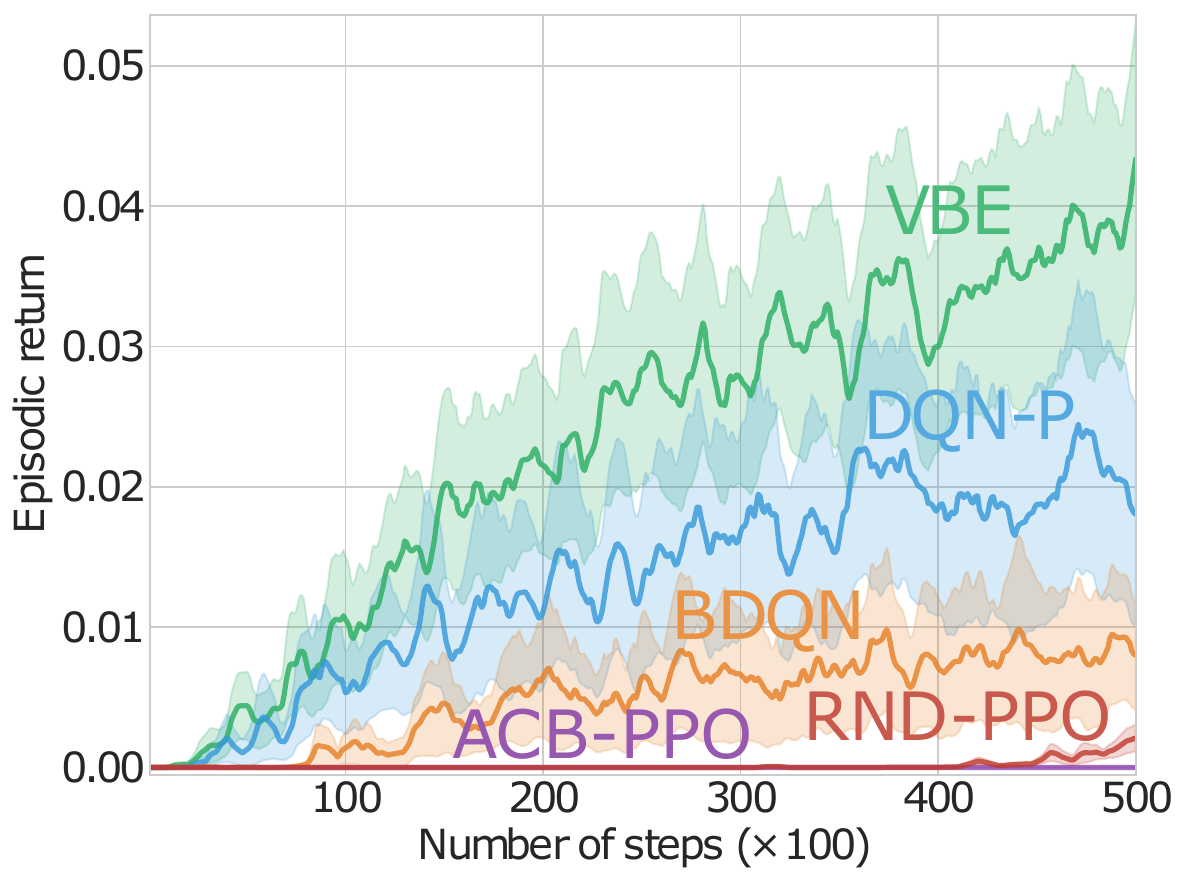}
        \caption{Mountain Car}
        \label{fig:mc_rb_l_best}
    \end{subfigure}
    \hfill%\hfil
    \centering
    \begin{subfigure}{0.24\textwidth}
        \includegraphics[width=\hsize]{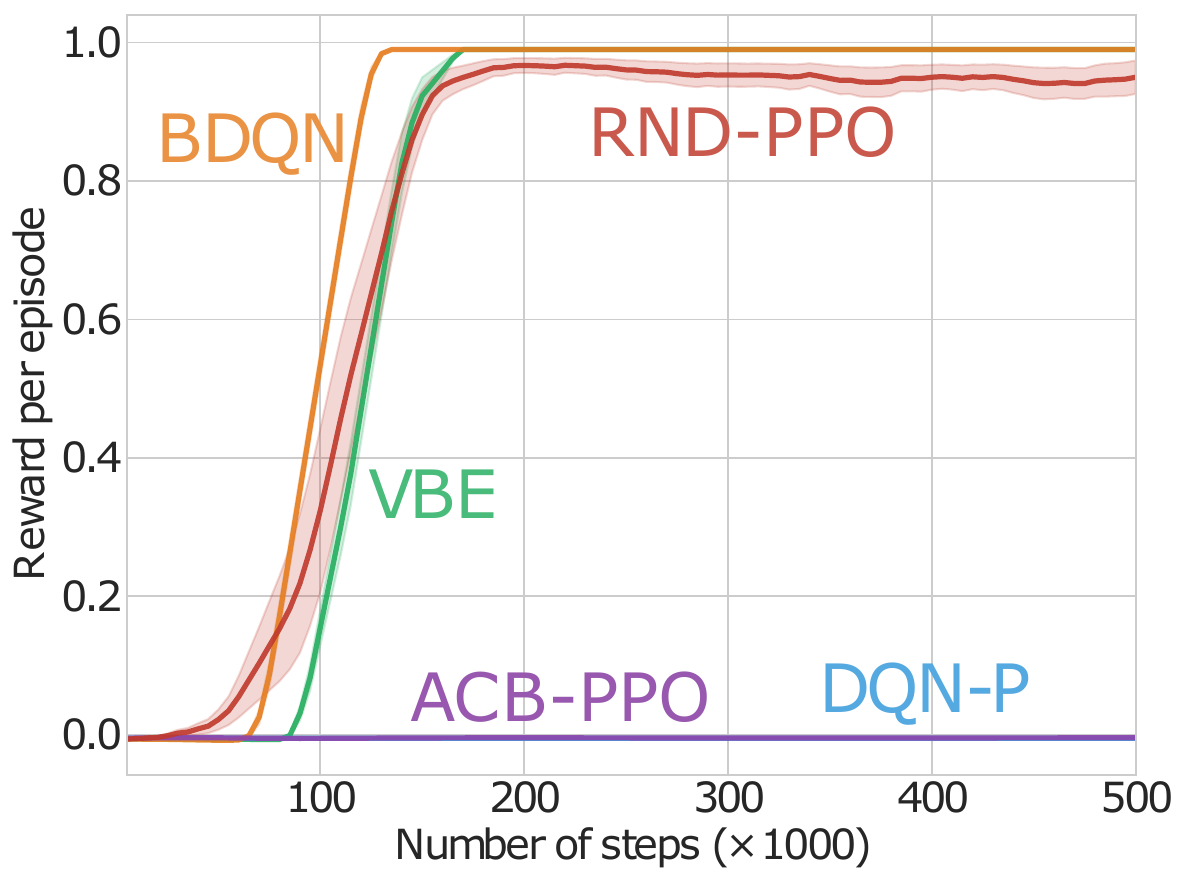}
        \caption{Deepsea}
        \label{fig:ds_rb_l_best}
    \end{subfigure}
    \caption{Online performance of PPO-based variants of ACB and RND in River Swim, Puddle World, Mountain Car, and Deepsea, with tile-coded features and linear function approximation. Higher on the y-axis is better. The x-axis denotes the number of interaction steps with the environment. The shaded region corresponds to standard errors.}
    \label{fig:rb_l_best}
\end{figure*}
% \begin{table}[ht]
%     \centering
%     \begin{tabular}{|c|c|c|c|}
%         \hline
%          & tiles & tiling & features\\
%          \hline
%         River Swim & 4 & 32 & 128\\
%         \hline
%         Puddle World & 5 & 5 & 128\\
%         \hline
%         Mountain Car & 4 & 16 & 512\\
%         \hline
%     \end{tabular}
%     \caption{Caption}
%     \label{tab:linear_best_tc_params}
% \end{table}

\section{Parameter Sensitivity of VBE: Ensemble Size $\times$ Bonus Scale}
\label{es_x_bs}
In this section we show the effect that bonus scales and ensemble sizes have on the performance of the VBE in each of the four classic control environments. In Figure~\ref{fig:bs_es_nn} we show the sensitivity of VBE used in Section~\ref{sec:classic_control} to its two parameters -- ensemble size and bonus scale --- in each environment. For River Swim we see that the performance improves as the bonus scale and the ensemble size is increased. This makes sense as River Swim is a hard exploration environment and requires more aggressive exploration. In Puddle World and Mountain Car, we observe that increasing the bonus scale and the ensemble size harms the performance, since they do not require too much exploration. For Deepsea we only test a bonus scale of 1 with different ensemble sizes on different grid sizes. We can see that only an ensemble size of 20 works well on all grid sizes. Figure~\ref{fig:bs_es_l} shows the parameter sensitivity results for the linear function approximation case. We observe a very similar pattern to its neural network counterpart across all the environments. 
\begin{figure*}[ht]
    \centering
    \begin{subfigure}{0.24\textwidth}
        \includegraphics[width=\hsize]{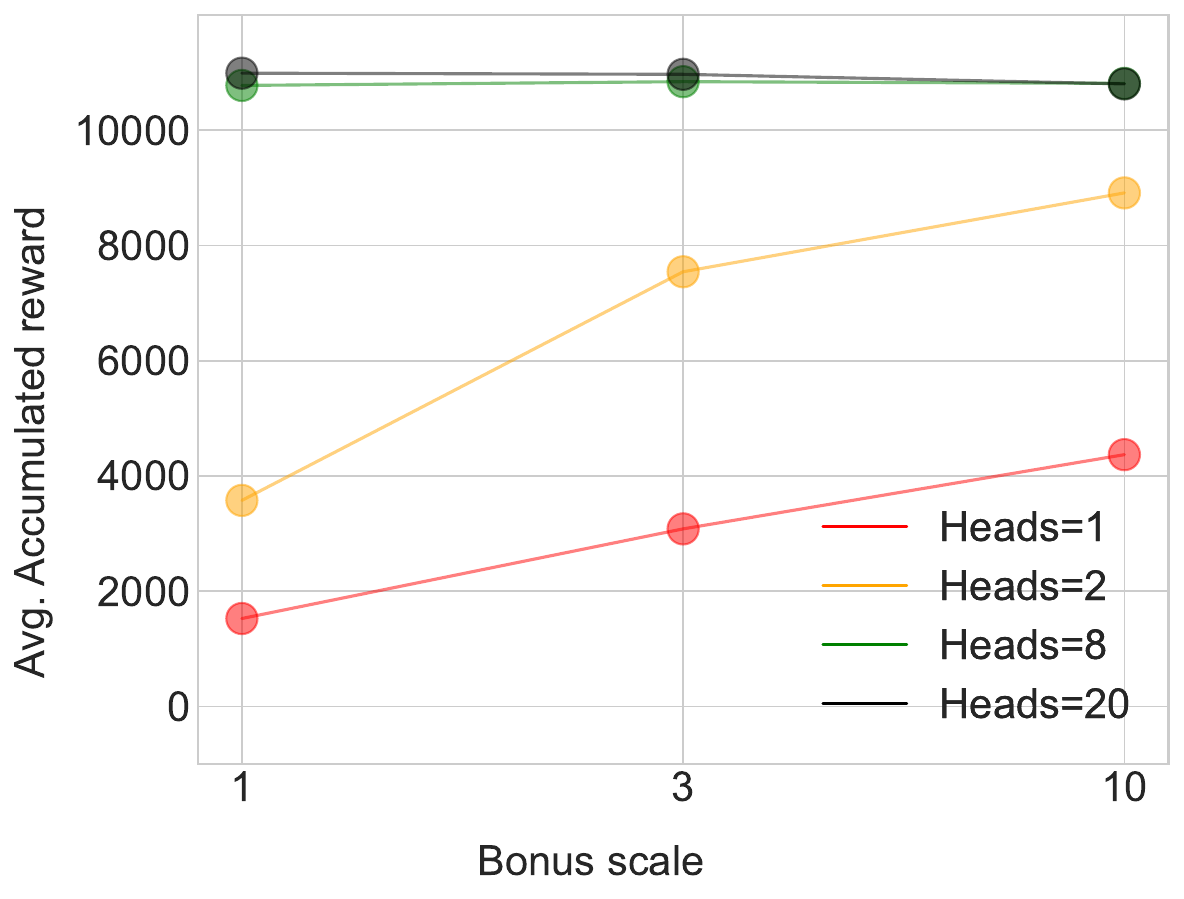}
        \caption{River Swim}
        \label{fig:rs_bs_es_nn}
    \end{subfigure}
    \hfill%\hfil
    \centering
    \begin{subfigure}{0.24\textwidth}
        \includegraphics[width=\hsize]{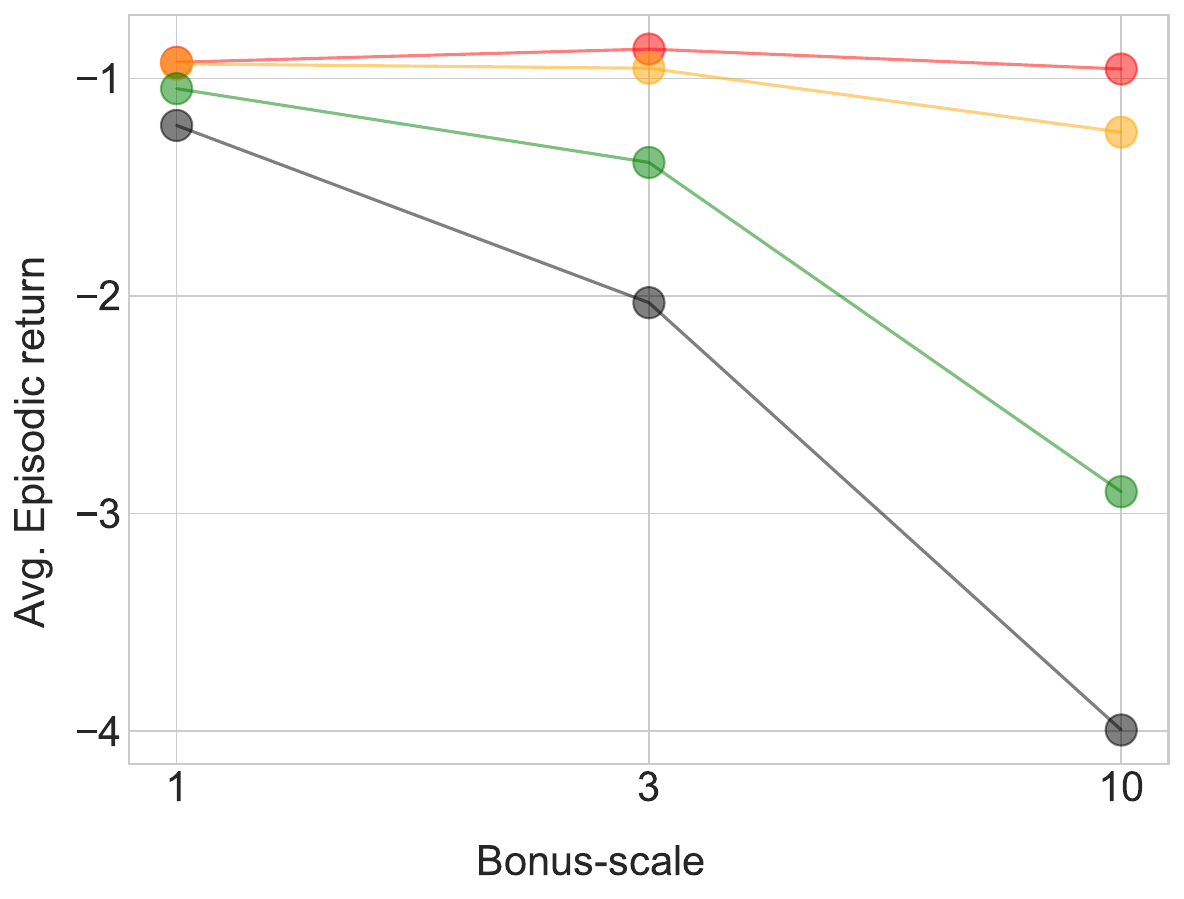}
        \caption{Puddle World}
        \label{fig:pw_bs_es_nn}
    \end{subfigure}
    \hfill%\hfil
    \centering
    \begin{subfigure}{0.24\textwidth}
        \includegraphics[width=\hsize]{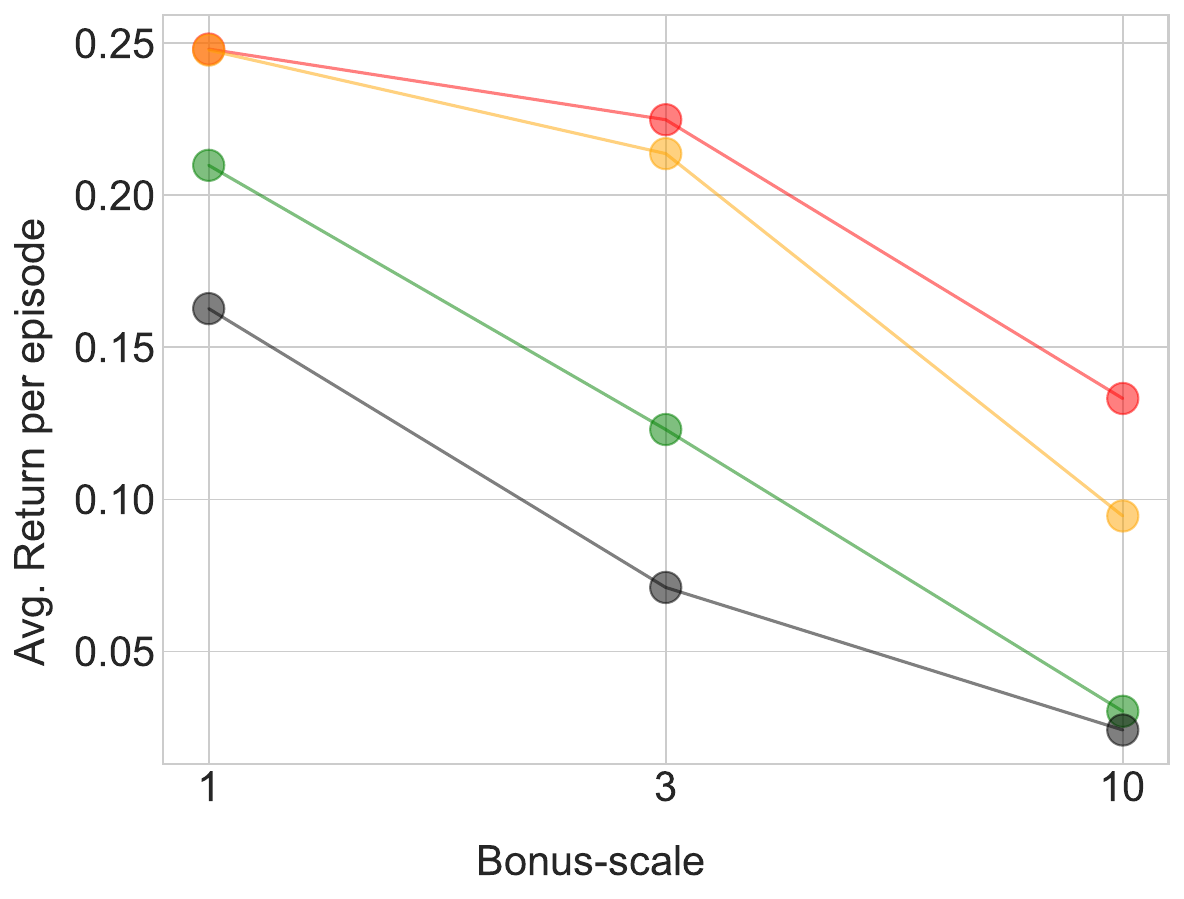}
        \caption{Mountain Car}
        \label{fig:mc_bs_es_nn}
    \end{subfigure}
    \hfill%\hfil
    \centering
    \begin{subfigure}{0.24\textwidth}
        \includegraphics[width=\hsize]{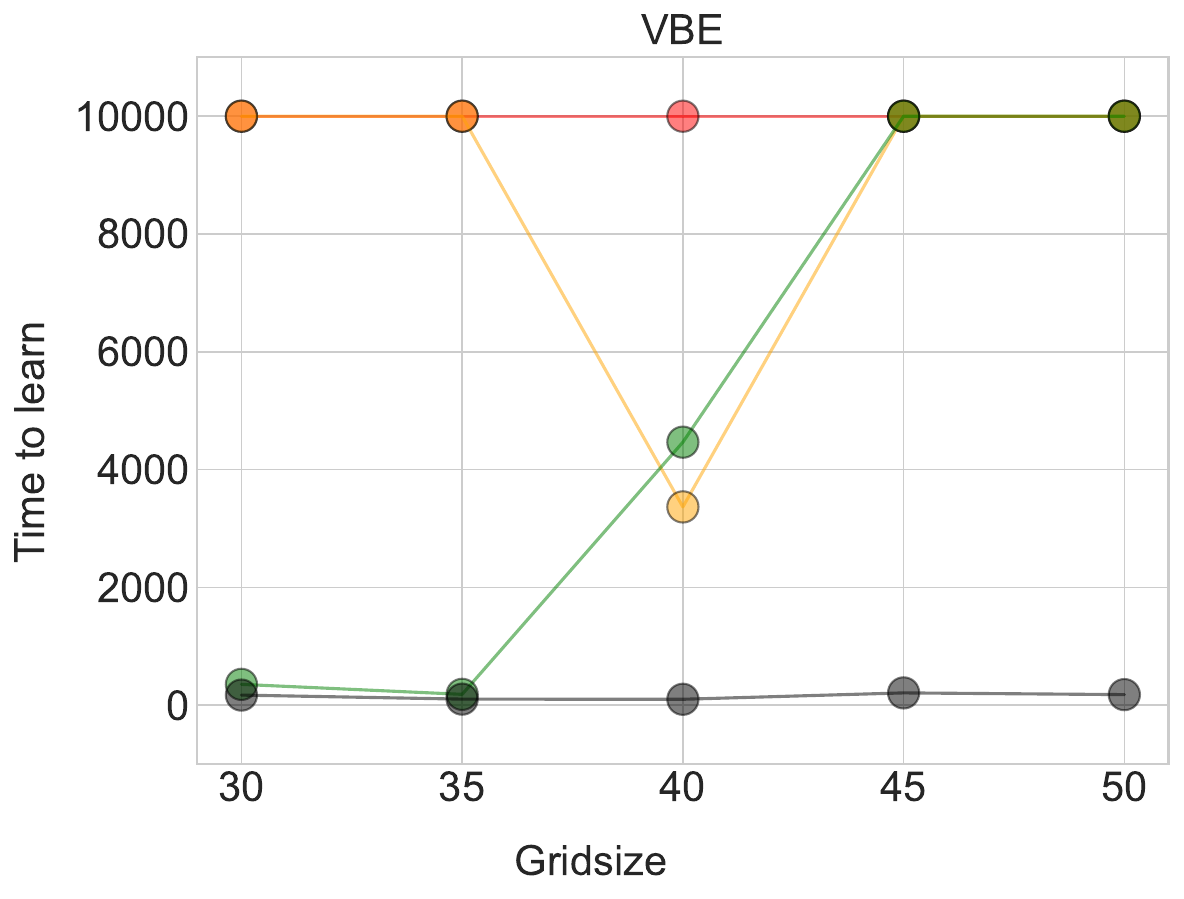}
        \caption{Deepsea}
        \label{fig:ds_bs_es_nn}
    \end{subfigure}
    \caption{Shows the effect of different bonus scales and ensemble sizes across the classic control environments. For Deepsea, we only use a bonus scale of 1 and test different ensemble sizes.}
    \label{fig:bs_es_nn}
\end{figure*}
\begin{figure*}[ht]
    \centering
    \begin{subfigure}{0.24\textwidth}
        \includegraphics[width=\hsize]{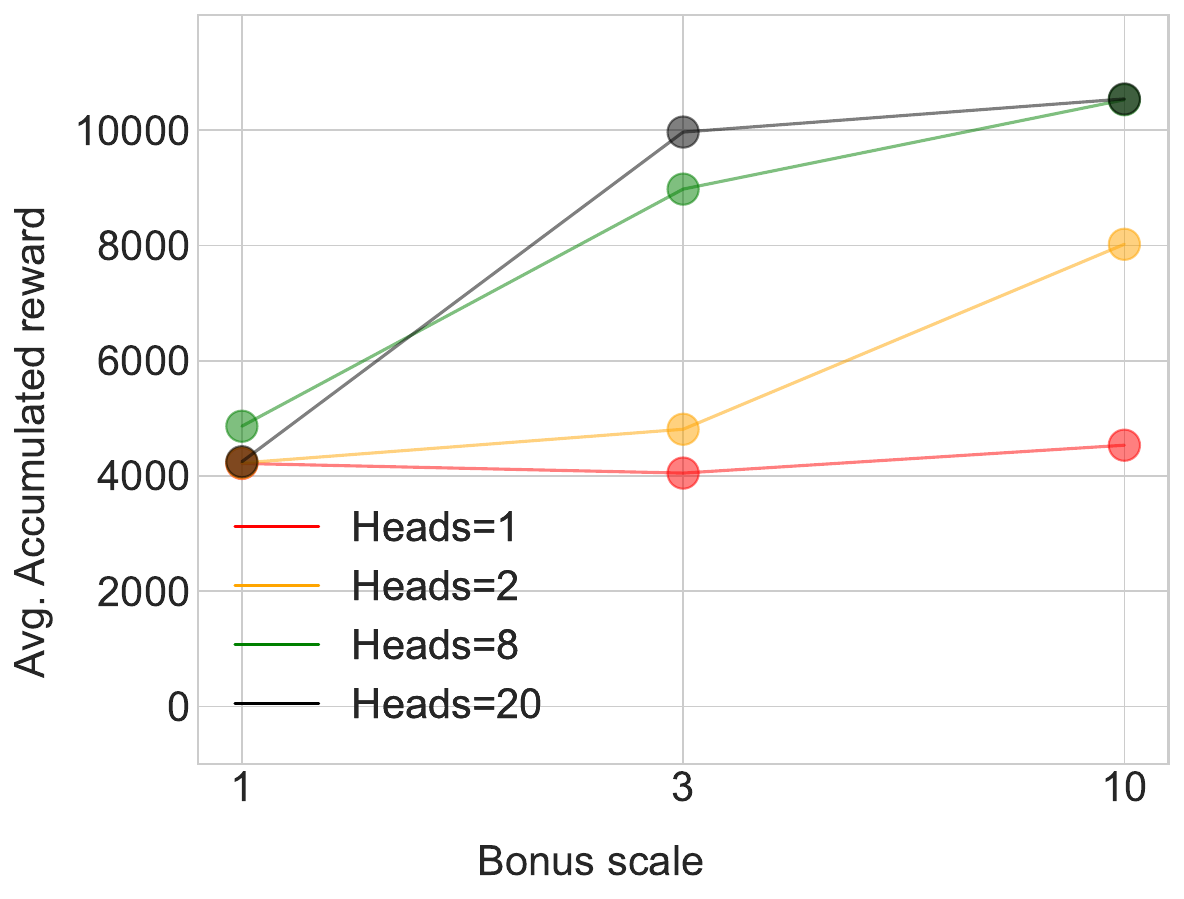}
        \caption{River Swim}
        \label{fig:rs_bs_es_l}
    \end{subfigure}
    \hfill%\hfil
    \centering
    \begin{subfigure}{0.24\textwidth}
        \includegraphics[width=\hsize]{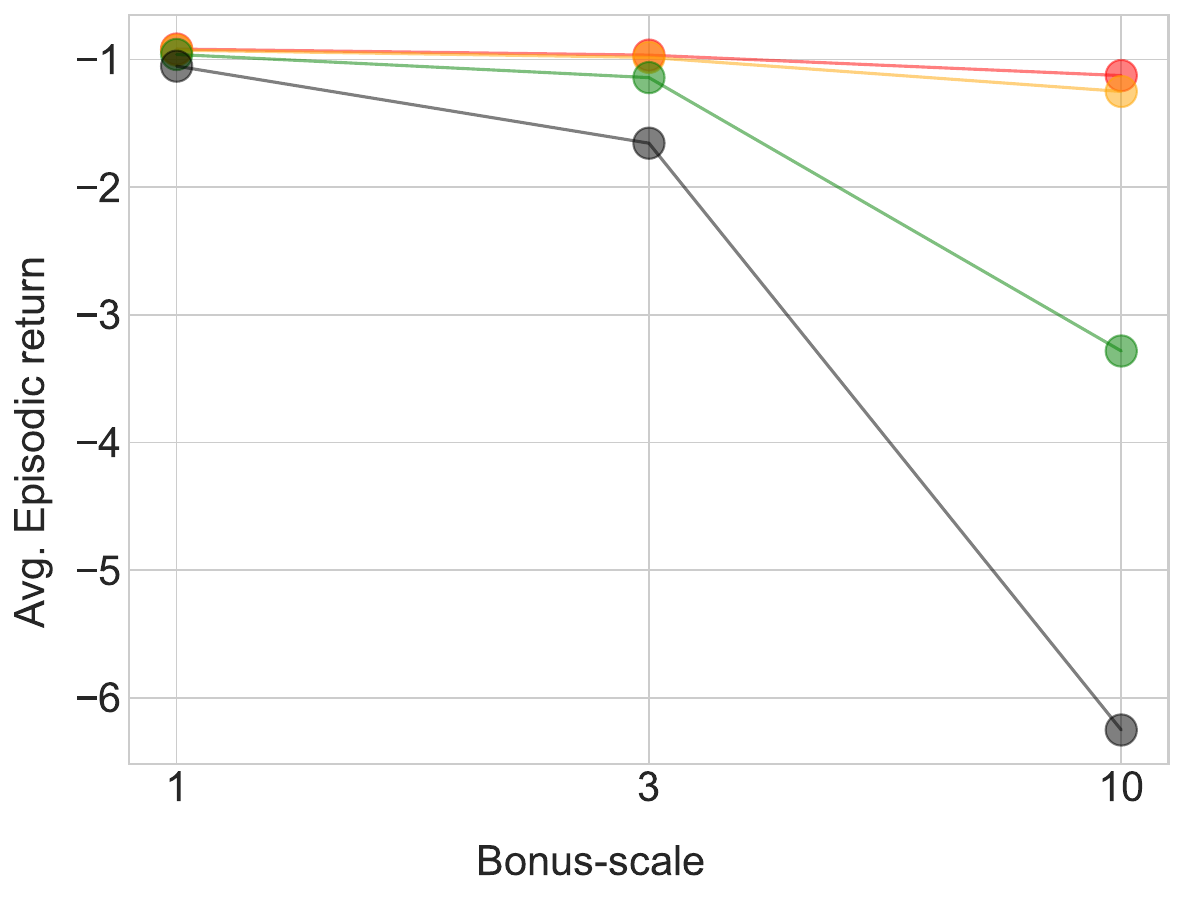}
        \caption{Puddle World}
        \label{fig:pw_bs_es_l}
    \end{subfigure}
    \hfill%\hfil
    \centering
    \begin{subfigure}{0.24\textwidth}
        \includegraphics[width=\hsize]{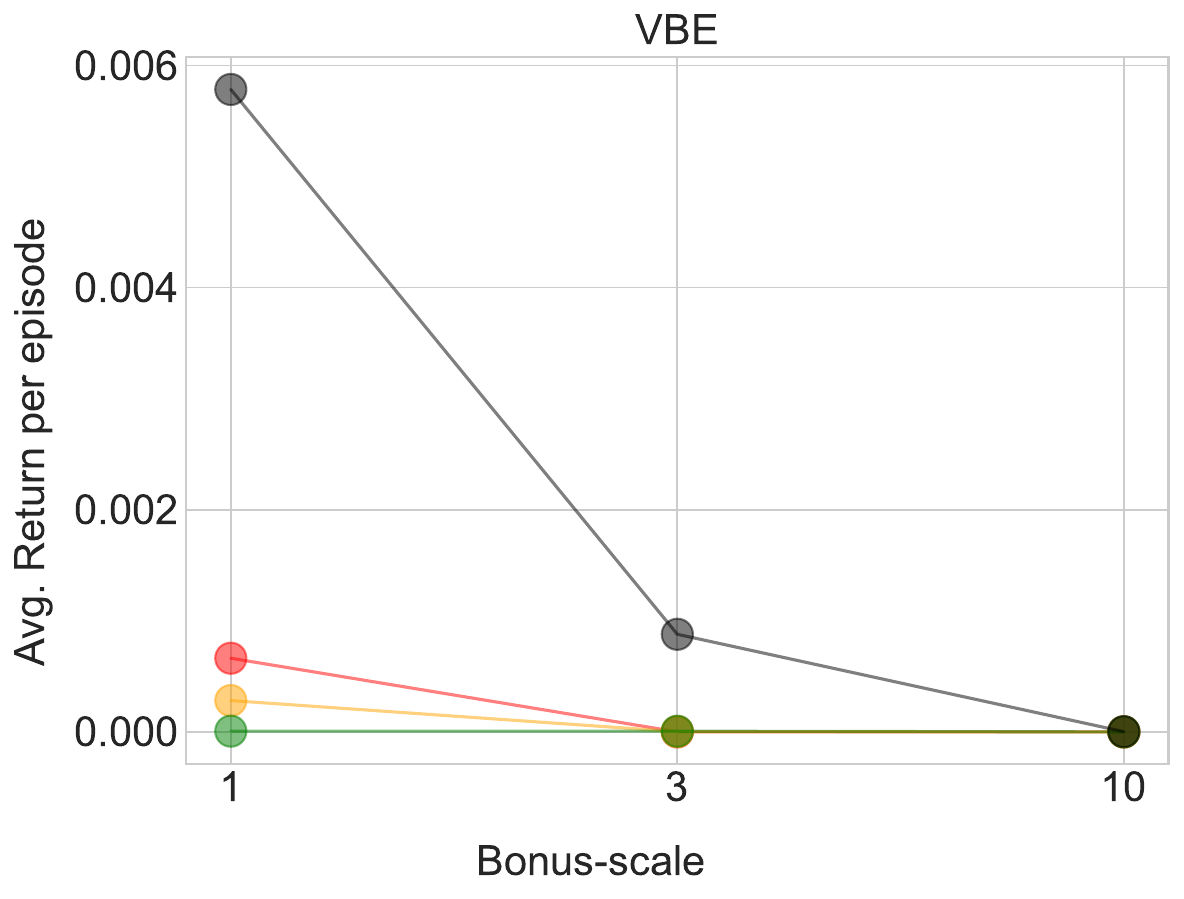}
        \caption{Mountain Car}
        \label{fig:mc_bs_es_l}
    \end{subfigure}
    \hfill%\hfil
    \centering
    \begin{subfigure}{0.24\textwidth}
        \includegraphics[width=\hsize]{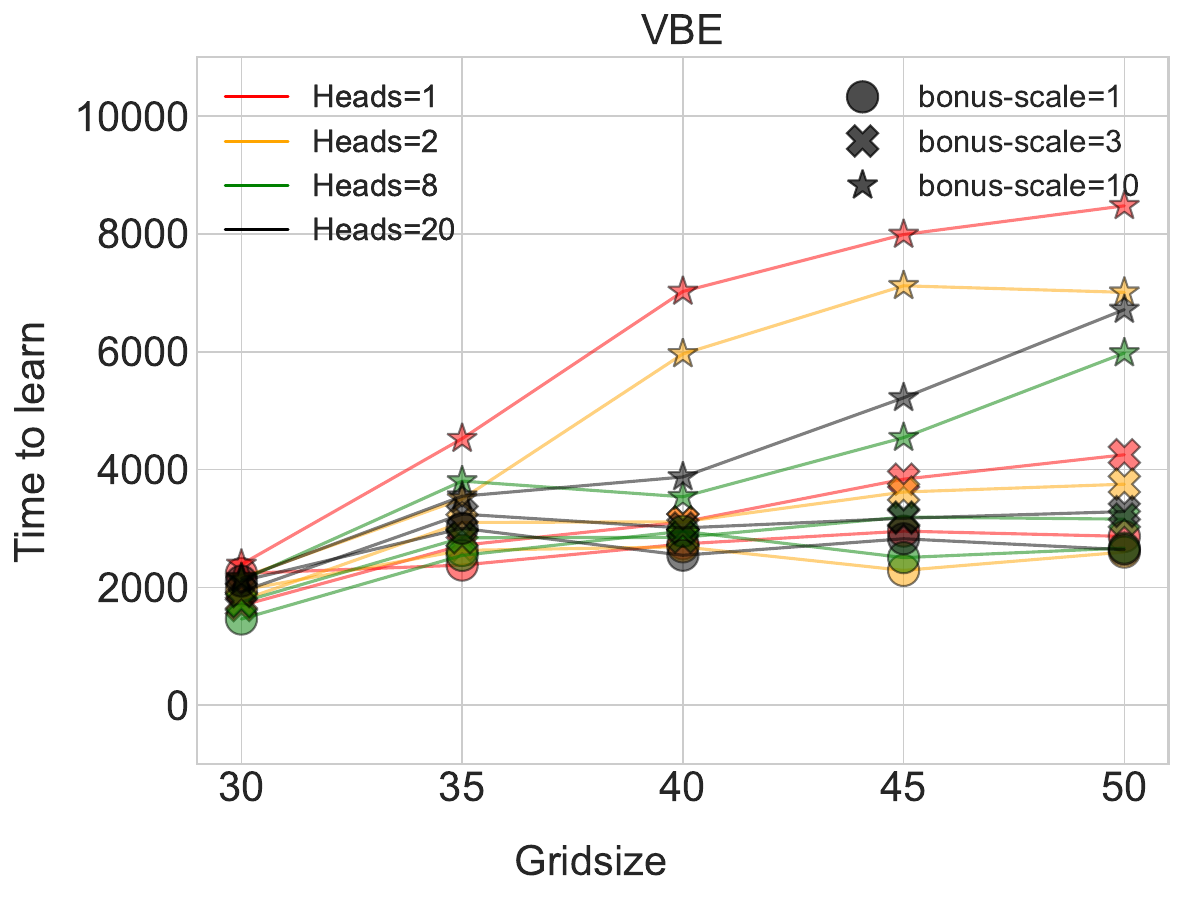}
        \caption{Deepsea}
        \label{fig:ds_bs_es_l}
    \end{subfigure}
    \caption{Shows the effect of different bonus scales and ensemble sizes across the classic control environments. These results correspond to the linear function approximation case.}
    \label{fig:bs_es_l}
\end{figure*}

\section{Target Policy Experiments}
\label{sp_tp_exp}

In this framework where we have access to two forms of value functions -- the action-values, and their corresponding bonuses -- there are two fundamental learning updates that we can employ -- on-policy, or off-policy. As we are in the control setting, we choose the off-policy q-learning update that tries to estimate $q^*$ directly in VBE. In this section we explore the impact of this choice by comparing the two updates: (1) on-policy updates where the target policy is the same as the behaviour policy, and (2) off-policy updates, which is the one employed by VBE, where the target policy is different from the behaviour policy. The former uses an optimistic target policy that maximizes over $q(s,.) + b(s,.)$, whereas the latter uses a greedy  target policy that maximizes over $q(s,.)$ alone.

\begin{figure*}[t]
    \includegraphics[width=\hsize]{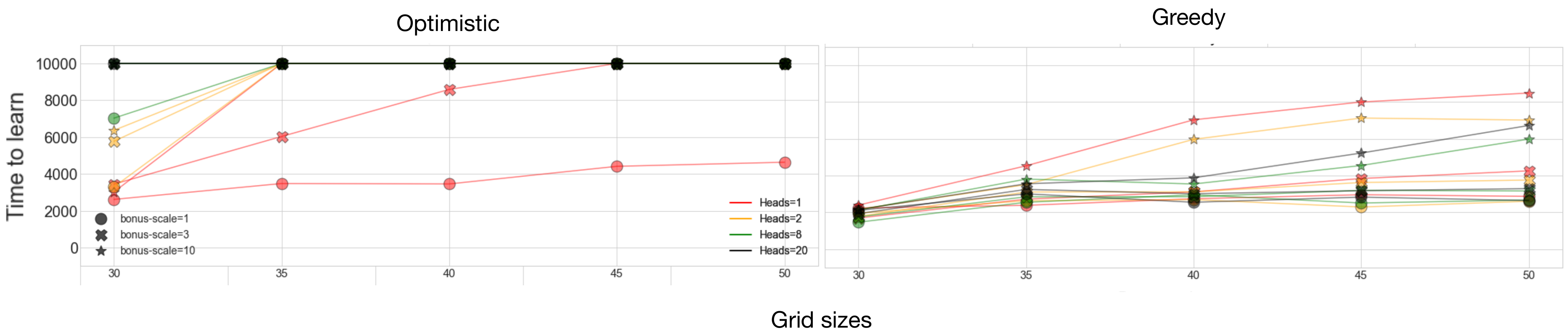}
    \caption{Comparing Optimistic target policy (On-policy) with Greedy target policy (Off-policy) on different grid sizes of Deepsea.}
    \label{fig:tp_best}
\end{figure*}

In Figure~\ref{fig:tp_best} we show the performance of the agents resulting from the two different updates, run on different grid sizes of Deepsea, with different ensemble size represented by the different colors, and different bonus scales represented by different shapes. The agents use linear function approximation for learning the action-value function and RQFs. We can see that the greedy agent generally performs better than the optimistic agent, which makes sense as the optimistic target policy can cause more exploration. The greedy agent performs well even with multiple RQFs, whereas optimistic agent fails to learn the optimal policy in this fixed budget of steps when the number of RQFs are increased. \begin{figure*}[t]
    \vspace{-0.3cm}
    \centering
    \begin{subfigure}{0.49\textwidth}
        \includegraphics[width=\hsize]{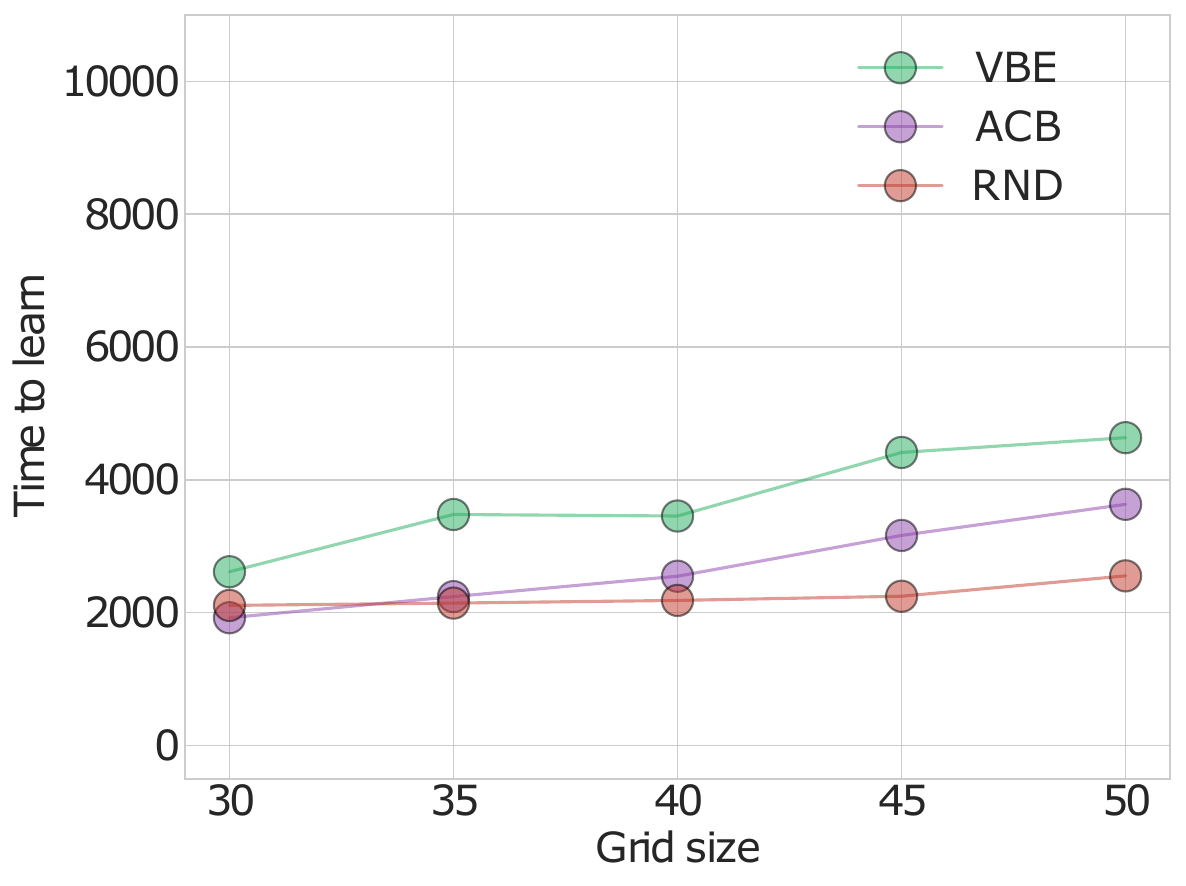}
        \caption{Optimistic target policy}
        \label{fig:ds_rb_l_optimistic_best}
    \end{subfigure}
    \hfill%\hfil
    \centering
    \begin{subfigure}{0.49\textwidth}
        \includegraphics[width=\hsize]{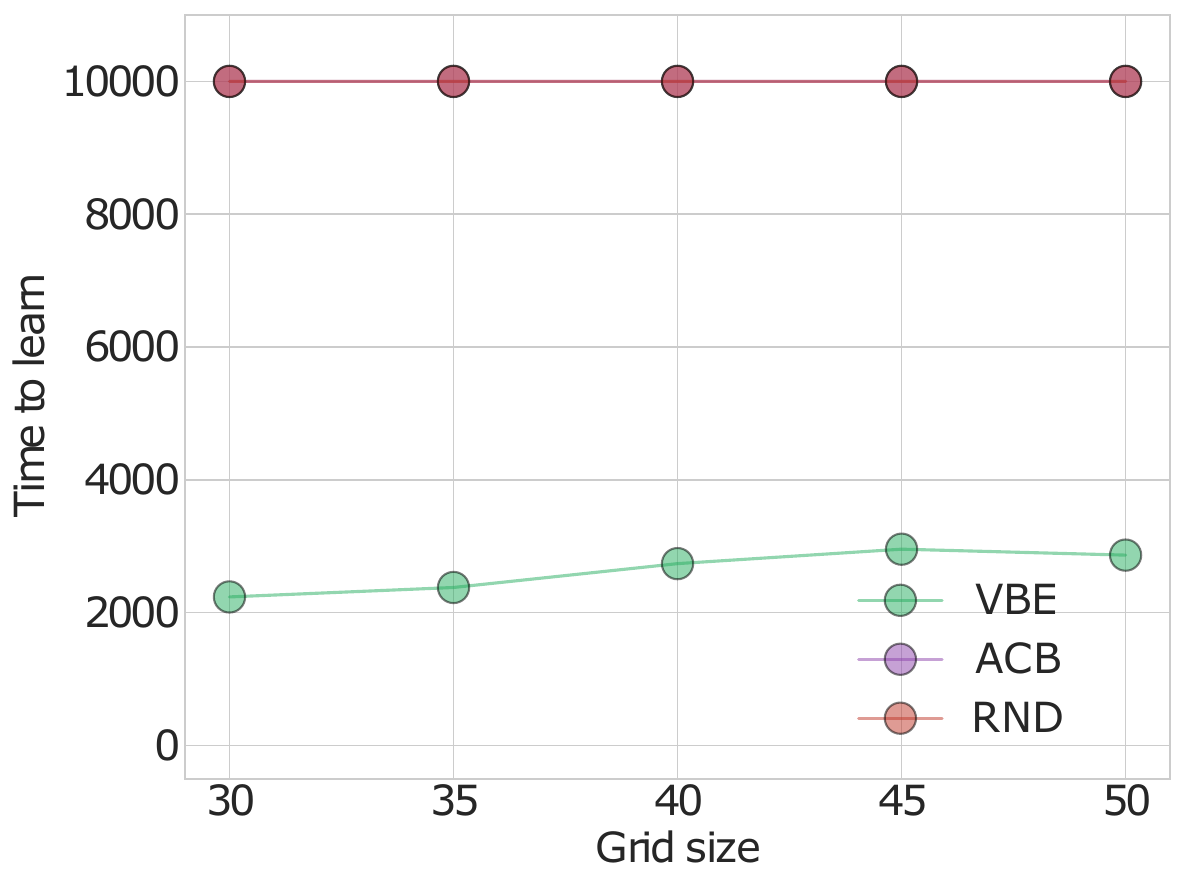}
        \caption{Greedy target policy}
        \label{fig:ds_rb_l_greedy_best}
    \end{subfigure}
    \caption{This figure compare optimistic target policy to Greedy target policy: (a) the agents use an on-policy (optimistic) updates, and (b) the agents use an off-policy (greedy) updates. ACB and RND fail with the greedy policy, whereas with optimistic target policy they outperform VBE. VBE however, does well with a greedy policy compared to the optimistic one. These agents use a single linear-layer with bias term.}
    \vspace{-0.3cm}
    \label{fig:ds_rb_l_greedy_optimistic}
\end{figure*}

Additionally, we investigated into the role of the target policy in ACB and RND. During these experiments with ACB and RND we noticed an interesting phenomenon: using an optimistic target policy allows ACB and RND to learn the optimal policy quickly on various grid-sizes of the Deepsea environment (Figure~\ref{fig:ds_rb_l_greedy_optimistic}), and using a greedy target policy for ACB and RND would cause the agents to fail to learn the optimal policy (Figure~\ref{fig:ds_l_best},~\ref{fig:ds_nn_best}). This is interesting as in either case we do not expect ACB and RND to be able to cover the entire state space based on random initialization because these reward bonus methods do not provide optimism for unseen state-action pairs -- that is, they do not provide first-visit optimism. Upon investigating the phenomenon, we found out that this happens because of the bias term in the linear layer, the momentum term in the optimizer and because the intrinsic value function in these methods are designed to be non-episodic. When the optimistic target policy is used along with momentum during optimization, the bias term  of the approximator consistently increases. Since the bias-term is a shared parameter, the increase in its value causes the intrinsic action values to start increasing, providing the optimism for unseen action-values as well -- allowing for the agent to cover the state space and thus learn the optimal policy. In case of the greedy target policy this phenomenon does not arise; the intrinsic values do not increase, and thus the agent fails to cover the state space. In Section~\ref{sec:pure-exploration}, we show that if we use tabular features with a linear-layer without any bias term then ACB and RND fail to cover the state space, reflecting that they do not provide first-visit optimism. 

\end{document}